\documentclass[11pt]{article}
\pdfoutput=1

\usepackage{mathrsfs}
\usepackage{amsmath,amssymb}
\usepackage{bm}
\usepackage{natbib}
\usepackage[usenames]{color}
\usepackage{amsthm}

\usepackage{multirow} 
\usepackage{enumitem}
\usepackage{wrapfig}

\renewcommand{\zeta}{\overline{\rho}}

\renewcommand{\omega}{\underline{\rho}}

\allowdisplaybreaks

\usepackage[colorlinks,
linkcolor=red,
anchorcolor=blue,
citecolor=blue
]{hyperref}

\DeclareMathOperator{\polylog}{\rm polylog}

\usepackage{setspace}
\usepackage[left=1in, right=1in, top=1in, bottom=1in]{geometry}

\usepackage{xcolor}

\ifdefined\final
\usepackage[disable]{todonotes}
\else
\usepackage[textsize=tiny]{todonotes}
\fi
\title{\huge Benign Overfitting in Two-layer Convolutional\\ Neural Networks}

\author
{
    Yuan Cao\thanks{Equal contribution} \thanks{Department of Statistics and Actuarial Science and Department of Mathematics, The University of Hong Kong, Hong Kong; e-mail:  {\tt yuancao@hku.hk}} 
    ~~~and~~~
	Zixiang Chen\footnotemark[1] \thanks{Department of Computer Science, University of California, Los Angeles, CA 90095, USA; e-mail: {\tt chenzx19@cs.ucla.edu}} 
	~~~and~~~
	Mikhail Belkin\thanks{Halicioğlu Data Science Institute, University of California, San Diego, La Jolla, CA 92093, USA; e-mail: {\tt mbelkin@ucsd.edu}} 
	~~~and~~~
	Quanquan Gu\thanks{Department of Computer Science, University of California, Los Angeles, CA 90095, USA; e-mail: {\tt qgu@cs.ucla.edu}}
}

\date{}

\usepackage{mylatexstyle}

\newcommand{\la}{\langle}
\newcommand{\ra}{\rangle}

\newtheorem{condition}[theorem]{Condition}

\def \poly {\mathrm{poly}}

\begin{document}

\maketitle

\begin{abstract}%
  Modern neural networks often have great expressive power and can be trained to overfit the training data, while still achieving a good test performance. This phenomenon is referred to as ``benign overfitting''. Recently, there emerges a line of works studying ``benign overfitting'' from the theoretical perspective. However, they are limited to linear models or kernel/random feature models, and there is still a lack of  theoretical understanding about when and how benign overfitting occurs in neural networks. In this paper, we study the benign overfitting phenomenon in training a two-layer convolutional neural network (CNN). We show that when the signal-to-noise ratio satisfies a certain condition, a two-layer CNN trained by gradient descent can achieve arbitrarily small training and test loss. On the other hand, when this condition does not hold, overfitting becomes harmful and the obtained CNN can only achieve constant level test loss. These together demonstrate a sharp phase transition between benign overfitting and harmful overfitting, driven by the signal-to-noise ratio. To the best of our knowledge, this is the first work that precisely characterizes the conditions under which benign overfitting can occur in training convolutional neural networks. 
\end{abstract}


\section{Introduction}
Modern deep learning models often consist of a huge number of model parameters, which is more than the number of training data points and therefore over-parameterized. These over-parameterized models can be trained to overfit the training data (achieving a close to $100\%$ training accuracy), while still making accurate prediction on the unseen test data. This phenomenon has been observed in a number of prior works \citep{zhang2016understanding,neyshabur2018role}, and
is often referred to as \textit{benign overfitting} \citep{bartlett2020benign}. 
It revolutionizes the  
the classical understanding about the bias-variance trade-off in statistical learning theory, and has drawn great attention from the community  \citep{belkin2018understand,belkin2019reconciling,belkin2019two,hastie2019surprises}. 




There exist a number of works towards understanding the benign overfitting phenomenon. While they offered important insights into the benign overfitting phenomenon, most of them are limited to the settings of linear models \citep{belkin2019two,bartlett2020benign,hastie2019surprises,wu2020optimal,chatterji2020finite,zou2021benign,cao2021risk} and kernel/random features models \citep{belkin2018understand,liang2020just,montanari2020interpolation}, and cannot be applied to neural network models that are of greater interest. The only notable exceptions are \citep{adlam2020neural,li2021towards}, which attempted to understand benign overfitting in neural network models. However, they are still limited to the ``neural tagent kernel regime'' \citep{jacot2018neural} where the neural network learning problem is essentially equivalent to kernel regression. Thus, it remains a largely open problem to show how and when benign overfitting can occur in neural networks. 


Clearly, understanding benign overfitting in neural networks is much more challenging than that in linear models, kernel methods or random feature models. The foremost challenge stems from nonconvexity: previous works on linear models and kernel methods/random features are all in the convex setting, while neural network training is a highly nonconvex optimization problem. 
Therefore, while most of the previous works can study the minimum norm interpolators/maximum margin classifiers according to the \textit{implicit bias} \citep{soudry2017implicit} results for the corresponding models, existing implicit bias results for neural networks (e.g., \citet{lyu2019gradient}) are not sufficient and a new analysis of the neural network learning process is in demand. 



In this work, we provide one such algorithmic analysis for learning two-layer convolutional neural networks (CNNs) with the second layer parameters being fixed as $+1$'s and $-1$'s and  polynomial ReLU activation function: $\sigma(z) = \max\{0,z\}^q$, where $q > 2$ is a hyperparameter. 
We consider a setting where the input data consist of \textit{label dependent signals} and \textit{label independent noises}, and utilize a \textit{signal-noise decomposition} of the CNN filters to precisely characterize the signal learning and noise memorization processes during neural network training. Our result not only demonstrates that benign overfitting can occur in learning two-layer neural networks, but also gives precise conditions under which the overfitted CNN trained by gradient descent can achieve small population loss. Our paper makes the following major contributions:
\begin{itemize}[leftmargin = *]
    \item We establish population loss bounds of overfitted CNN models trained by gradient descent, and theoretically demonstrate that benign overfitting can occur in learning over-parameterized neural networks. We show that under certain conditions on the signal-to-noise ratio, CNN models trained by gradient descent will prioritize learning the signal over memorizing the noise, and thus achieving both small training and test losses. To the best of our knowledge, this is the first result on the benign overfitting of neural networks that is beyond the neural tangent kernel regime. 
    \item We also establish a negative result showing that when the conditions on the signal-to-noise ratio do not hold, then the overfitted CNN model will achieve at least a constant population loss. This result, together with our upper bound result,  reveals an interesting phase transition between benign overfitting and harmful overfitting. 
    \item Our analysis is based on a new proof technique namely \textit{signal-noise decomposition}, which decomposes the convolutional filters into a linear combination of initial filters, the signal vectors and the noise vectors.  
    We convert the neural network learning into a discrete dynamical system of the coefficients from the decomposition, and perform a two-stage analysis that decouples the complicated relation among the coefficients. This enables us to analyze the non-convex  optimization problem, and bound the population loss of the CNN trained by gradient descent.
    We believe our proof technique is of independent interest and can potentially be applied to deep neural networks.
\end{itemize}


We note that a concurrent work \citep{frei2022benign} studies learning log-Concave mixture data with label flip noise using fully-connected two-layer neural networks with smoothed leaky ReLU activation. Notably, their risk bound matches the risk bound for linear models given in \citet{cao2021risk} when the label flip noise is zero. However, their analysis only focuses on upper bounding the risk, and cannot demonstrate the phase transition between benign and harmful overfitting. Compared with \citep{frei2022benign}, we focus on CNNs, and consider a different data model to better capture the nature of image classification problems. Moreover, we present both positive and negative results under different SNR regimes, and demonstrate a sharp phase transition between benign and harmful overfitting.
\paragraph{Notation.} 

Given two sequences $\{x_n\}$ and $\{y_n\}$, we denote $x_n = O(y_n)$ if there exist some absolute constant $C_1 > 0$ and $N > 0$ such that $|x_n|\le C_1 |y_n|$ for all $n \geq N$. Similarly, we denote $x_n = \Omega(y_n)$ if there exist $C_2 >0$ and $N > 0$ such that $|x_n|\ge C_2 |y_n|$ for all $n > N$. We say $x_n = \Theta(y_n)$ if $x_n = O(y_n)$ and $x_n = \Omega(y_n)$ both holds. We use $\tilde O(\cdot)$, $\tilde \Omega(\cdot)$, and $\tilde \Theta(\cdot)$ to hide logarithmic factors in these notations respectively. Moreover, we denote $x_n=\poly(y_n)$ if $x_n=O( y_n^{D})$ for some positive constant $D$, and $x_n = \polylog(y_n)$ if $x_n= \poly( \log (y_n))$. Finally, for two scalars $a$ and $b$, we denote $a \vee b = \max\{a, b\}$.

\section{Related Work} 
A line of recent works have attempted to understand why overfitted predictors can still achieve a good test performance. 
\citet{belkin2019reconciling} first empirically demonstrated that in many machine learning models such as random Fourier features, decision trees and ensemble methods , the population risk curve has a \textit{double descent} shape with respect to the number of model parameters. 
\citet{belkin2019two} further studied two specific data models, namely the Gaussian model and Fourier series model, and theoretically demonstrated the double descent risk curve in linear regression.
\citet{bartlett2020benign} studied over-parameterized linear regression to fit data produced by a linear model with additive noises, and established matching upper and lower bounds of the risk achieved by the minimum norm interpolator on the training dataset. It is shown that under certain conditions on the spectrum of the data covariance matrix, the population risk of the interpolator can be asymptotically optimal. \citet{hastie2019surprises,wu2020optimal} studied linear regression in the setting where both the dimension and sample size grow together with a fixed ratio, and showed double descent of the risk with respect to this ratio.
\citet{chatterji2020finite} studied the population risk bounds of over-parameterized linear logistic regression on sub-Gaussian mixture models with label flipping noises, and showed how gradient descent can train over-parameterized linear models to achieve nearly optimal population risk. \citet{cao2021risk} tightened the upper bound given by  \citet{chatterji2020finite} in the case without the label flipping noises, and established a matching lower bound of the risk achieved by over-parameterized maximum margin interpolators. \citet{shamir2022implicit} proposed a generic data model for benign overfitting of linear predictors, and studied different problem settings under which benign overfitting can or cannot occur.

Besides the studies on linear models, several recent works also studied the benign overfitting and double descent phenomena in kernel methods or random feature models. \citet{zhang2016understanding} first pointed out that overfitting kernel predictors can sometimes still achieve good population risk. \citet{liang2020just} studied how interpolating kernel regression with radial basis function (RBF) kernels (and variants) can generalize and how the spectrum of the data covariance matrix affects the population risk of the interpolating kernel predictor. \citet{li2021towards} studied the benign overfitting phenomenon of random feature models defined as two-layer neural networks whose first layer parameters are fixed at random initialization.  \citet{mei2019generalization,liao2020random} demonstrated the double descent phenomenon for the population risk of interpolating random feature predictors with respect to the ratio between the dimensions of the random feature and the data input. \citet{adlam2020neural} shows that neural tangent kernel \citep{jacot2018neural} based kernel regression has a triple descent risk curve with respect to the total number of trainable parameters.  \citet{montanari2020interpolation} further pointed out an interesting phase transition of the generalization error achieved by neural networks trained in the neural tangent kernel regime.

\section{Problem Setup}
In this section, we introduce the data generation model and the convolutional neural network we consider in this paper. We focus on binary classification, and present our data distribution $\cD$ in the following definition.

\begin{definition}\label{def:data}
Let $\bmu\in \RR^d$ be a fixed vector representing the signal contained in each data point. Then each data point $(\xb,y)$ with $\xb = [\xb_1^\top, \xb_2^\top]^\top\in\RR^{2d}$ and $y\in\{-1,1\}$ is generated from the following distribution $\cD$: 
\begin{enumerate}[leftmargin = *]
    \item The label $y$ is generated as a Rademacher random variable.
    \item A noise vector $\bxi$ is generated from the Gaussian distribution $N(\mathbf{0}, \sigma_p^2 \cdot (\Ib - \bmu \bmu^\top \cdot \| \bmu \|_2^{-2}))$.
    \item One of $\xb_1, \xb_2$ is given as $y\cdot \bmu$, which represents the signal, the other is given by $\bxi$, which represents noises.
\end{enumerate}

\end{definition}

Our data generation model is inspired by image data, where the inputs consist of different patches, and only some of the patches are related to the class label of the image. In detail, the patch assigned as $y\cdot \bmu$ is the signal patch that is correlated to the label of the data, and the patch assigned as $\bxi$ is the noise patch that is independent of the label of the data and therefore is irrelevant for prediction. We assume that the noise patch is generated from the Gaussian distribution $N(\mathbf{0}, \sigma_p^2 \cdot (\Ib - \bmu \bmu^\top \cdot \| \bmu \|_2^{-2}))$ to ensure that the noise vector is orthogonal to the signal vector $\bmu$ for simplicity. Note that when the dimension $d$ is large, $ \| \bxi \|_2 \approx \sigma_p \sqrt{d}$ by standard concentration bounds. Therefore, we can treat $\| \bmu\|_2 / (\sigma_p \sqrt{d}) \approx \| \bmu\|_2 / \| \bxi \|_2$ as the signal-to-noise ratio. For the ease of discussion, we denote  $\mathrm{SNR} = \| \bmu\|_2 / (\sigma_p \sqrt{d})$. Note that the Bayes risk for learning our model is zero. We can also add label flip noise similar to \citet{chatterji2020finite,frei2022benign} to make the Bayes risk equal to the label flip noise and therefore nonzero, but this will not change the key message of our paper.  


Intuitively, if a classifier learns the signal $\bmu$ and utilizes the signal patch of the data to make prediction, it can  perfectly fit a given training data set $\{ (\xb_i, y_i) : i\in [n] \}$ and at the same time have a good performance on the test data. However, when the dimension $d$ is large ($d > n$), a classifier that is a function of the noises $\bxi_i$, $i \in [n]$ can also perfectly fit the training data set, while the prediction will be totally random on the new test data. Therefore, the data generation model given in Definition~\ref{def:data} is a useful model to study the population loss of overfitted classifiers. Similar models have been studied in some recent works by \citet{li2019towards,allen2020feature,allen2020towards,zou2021understanding}.


\noindent\textbf{Two-layer CNNs.} We consider a two-layer convolutional neural network whose filters are applied to the two patches $\xb_1$ and $\xb_2$ separately, 
and the second layer parameters of the network are fixed as $+1/m$ and $-1/m$ respectively. Then
the network can be written as $f(\Wb, \xb) = F_{+1}(\Wb_{+1},\xb) - F_{-1}(\Wb_{-1},\xb)$, 
where $F_{+1}(\Wb_{+1},\xb)$,  $F_{-1}(\Wb_{-1},\xb)$ are defined as: 
\begin{align*}
F_j(\Wb_j,\xb)  &= \frac{1}{m}{\sum_{r=1}^m} \big[\sigma(\la\wb_{j,r},\xb_1\ra) + \sigma(\la\wb_{j,r}, \xb_2\ra)\big]= \frac{1}{m} {\sum_{r=1}^m} \big[\sigma(\la\wb_{j,r},y\cdot\bmu\ra) + \sigma(\la\wb_{j,r}, \bxi\ra)\big],
\end{align*}
for $j\in \{+ 1, -1\}$, 
$m$ is the number of convolutional filters in $F_{+1}$ and $F_{-1}$. Here, $\sigma(z) = (\max\{0,z\})^q$ is the $\mathrm{ReLU}^q$ activation function where $q> 2$, $\wb_{j,r}\in\RR^{d}$ denotes the weight for the $r$-th filter (i.e., neuron), and $\Wb_{j}$ is the collection of model weights associated with $F_j$. We also use $\Wb$ to denote the collection of all model weights. We note that our CNN model can also be viewed as a CNN with average global pooling \citep{lin2013network}.
We train the above CNN model by minimizing the empirical cross-entropy loss function
\begin{align*}
    L_S(\bW) 
    &= \frac{1}{n} {\sum_{i=1}^n} \ell[ y_i \cdot f(\Wb,\xb_i) ],
\end{align*}
where $\ell(z) = \log(1 + \exp(-z))$, and $S = \{(\xb_i,y_i)\}_{i=1}^n$ is the training data set. We further define the true loss (test loss) $L_{\cD}(\Wb) := \EE_{(\xb,y )\sim \cD} \ell[ y \cdot f(\Wb,\xb) ]$.

We consider gradient descent starting from Gaussian initialization, where each entry of $\Wb_{+1}$ and $\Wb_{-1}$ is sampled from a Gaussian distribution $N(0 , \sigma_0^2)$, and $\sigma_0^2$ is the variance. The gradient descent update of the filters in the CNN can be written as
\begin{align}
    \wb_{j,r}^{(t+1)} &= \wb_{j,r}^{(t)} - \eta \cdot \nabla_{\wb_{j,r}} L_S(\bW^{(t)}) \nonumber\\
    &= \wb_{j,r}^{(t)} - \frac{\eta}{nm} \sum_{i=1}^n \ell_i'^{(t)} \cdot  \sigma'(\la\wb_{j,r}^{(t)}, \bxi_{i}\ra)\cdot j y_{i}\bxi_{i} - \frac{\eta}{nm} \sum_{i=1}^n \ell_i'^{(t)} \cdot \sigma'(\la\wb_{j,r}^{(t)}, y_{i} \bmu\ra)\cdot j\bmu \label{eq:gdupdate}
\end{align}
for $j \in \{\pm 1\}$ and $r \in [m]$, 
where we introduce a shorthand notation $\ell_i'^{(t)} = \ell'[ y_i \cdot f(\Wb^{(t)},\xb_i) ] $. 

\section{Main Results}
In this section, we present our main theoretical results. At the core of our analyses and results is a \textit{signal-noise decomposition} of the filters in the CNN trained by gradient descent. By the gradient descent update rule \eqref{eq:gdupdate}, it is clear that the gradient descent iterate $\wb_{j,r}^{(t)}$ is a linear combination of its random initialization $\wb_{j,r}^{(0)}$, the signal vector $\bmu$ and the noise vectors in the training data $\bxi_i$, $i\in [n]$. Motivated by this observation, we introduce the following definition. 
\begin{definition}\label{def:w_decomposition}
Let $\wb_{j,r}^{(t)}$ for $j\in \{\pm 1\}$, $r \in [m]$ be the convolution filters of the CNN at the $t$-th iteration of gradient descent. Then there exist unique coefficients $\gamma_{j,r}^{(t)} \geq 0$ and $\rho_{j,r,i}^{(t)}$ such that 
\begin{align*}
    \wb_{j,r}^{(t)} = \wb_{j,r}^{(0)} + j \cdot \gamma_{j,r}^{(t)} \cdot \| \bmu \|_2^{-2} \cdot \bmu + \sum_{ i = 1}^n \rho_{j,r,i}^{(t) }\cdot \| \bxi_i \|_2^{-2} \cdot \bxi_{i}.
\end{align*}
We further denote $\zeta_{j,r,i}^{(t)} := \rho_{j,r,i}^{(t)}\ind(\rho_{j,r,i}^{(t)} \geq 0)$, $\omega_{j,r,i}^{(t)} := \rho_{j,r,i}^{(t)}\ind(\rho_{j,r,i}^{(t)} \leq 0)$. Then we have that 
\begin{align}\label{eq:w_decomposition}
    \wb_{j,r}^{(t)} = \wb_{j,r}^{(0)} + j \cdot \gamma_{j,r}^{(t)} \cdot \| \bmu \|_2^{-2} \cdot \bmu + \sum_{ i = 1}^n \zeta_{j,r,i}^{(t) }\cdot \| \bxi_i \|_2^{-2} \cdot \bxi_{i} + \sum_{ i = 1}^n \omega_{j,r,i}^{(t) }\cdot \| \bxi_i \|_2^{-2} \cdot \bxi_{i} .
\end{align}
\end{definition}


We refer to \eqref{eq:w_decomposition} as the \textit{signal-noise decomposition} of $\wb_{j,r}^{(t)}$. We add normalization factors $\|\bmu\|_{2}^{-2}, \|\bxi_{i}\|_{2}^{-2}$ in the definition so that  $\gamma_{j,r}^{(t)} \approx \la \wb_{j,r}^{(t)}, \bmu \ra, \rho_{j,r,i}^{(t)}\approx \la \wb_{j,r}^{(t)}, \bxi_{i} \ra$. 
In this decomposition, $\gamma_{j,r}^{(t)}$ characterizes the progress of  learning the signal vector $\bmu$, and $\rho_{j,r,i}^{(t)}$ characterizes the degree of noise memorization by the filter. Evidently, based on this decomposition, for some iteration $t$, 
(i) If some of $\gamma_{j,r}^{(t)}$'s are large enough while $|\rho_{j,r,i}^{(t)}|$ are relatively small, then the CNN will have small training and test losses; (ii) If some $\zeta_{j,r,i}^{(t) }$'s are large and all $\gamma_{j,r}^{(t)}$'s are small, then the CNN will achieve a small training loss, but a large test loss. 
Thus, Definition~\ref{def:w_decomposition} provides a handle for us to study the convergence of the training loss as well as the the population loss of the CNN trained by gradient descent. 


Our results are based on the following conditions on the dimension $d$, sample size $n$, neural network width $m$, learning rate $\eta$, initialization scale $\sigma_0$.
\begin{condition}\label{condition:d_sigma0_eta}
Suppose that 
\begin{enumerate}[leftmargin = *]
    \item Dimension $d$ is sufficiently large:  $d = \tilde{\Omega}(m^{2\vee [4/(q-2)]}n^{4 \vee [(2q-2)/(q-2)]})$.
    \item Training sample size $n$ and neural network width $m$ satisfy $n,m = \Omega(\polylog(d))$.
    \item The learning rate $\eta$ satisfies $\eta \leq  \tilde{O}(\min\{\|\bmu\|_{2}^{-2}, \sigma_{p}^{-2}d^{-1}\})$.
    \item The standard deviation of Gaussian initialization $\sigma_0$ is appropriately chosen such that $\tilde{O}(nd^{-1/2}) \cdot \min\{(\sigma_{p}\sqrt{d})^{-1}, \|\bmu\|_{2}^{-1}\} \leq \sigma_{0} \leq \tilde{O}(m^{-2/(q-2)}n^{-[1/(q-2)]\vee 1})\cdot\min\{(\sigma_{p}\sqrt{d})^{-1}, \|\bmu\|_{2}^{-1}\}$.
\end{enumerate}
\end{condition}
A few remarks on Condition~\ref{condition:d_sigma0_eta} are in order. The condition on $d$ is to ensure that the learning is in a sufficiently over-parameterized setting, and similar conditions have been made in the study of learning over-parameterized linear models \citep{chatterji2020finite,cao2021risk}. For example, if we choose $q=3$, then the condition on $d$ becomes $d = \tilde \Omega(m^4n^4)$. Furthermore, we require the sample size and neural network width to be at least polylogarithmic in the dimension $d$ to ensure some statistical properties of the training data and weight initialization to hold with probability at least $1 - d^{-1}$, which is a mild condition. 
Finally, the conditions on $\sigma_0$ and $\eta$ are to ensure that gradient descent can effectively minimize the training loss, and they depend on the scale of the training data points. When $\sigma_{p} = O(d^{-1/2})$ and $\|\bmu\|_{2} = O(1)$, the step size $\eta$ can be chosen as large as $\tilde{O}(1)$ and the initialization $\sigma_{0}$ can be as large as $\tilde{O}(m^{-2/(q-2)}n^{-[1/(q-2)]\vee 1})$. In our paper, we only require $m,n = \Omega(\text{polylog}(d))$, so our initialization and step-size can be chosen as an almost constant order. 
Based on these conditions, we give our main result on signal learning in the following theorem.

\begin{theorem}\label{thm:signal_learning_main}
For any $\epsilon > 0$, let $T = \tilde{\Theta}( \eta^{-1} m\sigma_0 ^{-(q-2)} \| \bmu \|_2^{-q} +  \eta^{-1}\epsilon^{-1} m^{3}\| \bmu \|_2^{-2})$. Under Condition~\ref{condition:d_sigma0_eta}, if $ n \cdot \mathrm{SNR}^q = \tilde\Omega( 1 )$\footnote{Here the $\tilde\Omega(\cdot)$ hides an $\polylog(\epsilon^{-1}) $ factor. This applies to Theorem~\ref{thm:noise_memorization_main} as well.}, then with probability at least $ 1 - d^{-1}$, there exists $0 \leq t \leq T$ such that:
\begin{enumerate}[leftmargin = *]
    \item The CNN learns the signal: $\max_r\gamma_{j,r}^{(t)} = \Omega(1)$ for $j\in \{\pm 1\}$.
    \item The CNN does not memorize the noises in the training data: 
    $\max_{j,r,i} |\rho_{j,r,i}^{(T)}| = \tilde O( \sigma_0 \sigma_p \sqrt{d} )$.
    \item The training loss converges to $\epsilon$, i.e., $L_S(\Wb^{(t)}) \leq \epsilon$.
    \item The trained CNN achieves a small test loss: $L_{\cD}(\Wb^{(t)})\leq 6\epsilon + \exp(- n^{2})$
\end{enumerate}
\end{theorem}

Theorem~\ref{thm:signal_learning_main} characterizes the case of signal learning. It shows that, if $ n \cdot \mathrm{SNR}^q = \tilde\Omega( 1  )$, then at least one CNN filter can learn the signal by achieving $\gamma_{j,r_j^*}^{(t)} \geq  \Omega(1)$, and as a result, the learned neural network can achieve small training and test losses. To demonstrate the sharpness of this condition, we also present the following theorem for the noise memorization by the CNN.

\begin{theorem}\label{thm:noise_memorization_main}
For any $\epsilon > 0$, let  $T = \tilde{\Theta}( \eta^{-1}m \cdot n (\sigma_{p}\sqrt{d})^{-q} \cdot \sigma_0^{-(q-2)}  + \eta^{-1}\epsilon^{-1}nm^{3} d^{-1}\sigma_{p}^{-2})$. Under Condition~\ref{condition:d_sigma0_eta},
if $ n^{-1}\cdot \mathrm{SNR}^{-q} = \tilde \Omega( 1)$, then with probability at least $ 1 - d^{-1}$, there exists $0 \leq t \leq T$ such that:
\begin{enumerate}[leftmargin = *]
    \item The CNN memorizes noises in the training data: $\max_r \zeta_{y_i,r,i}^{(t)} = \Omega(1)$.
    \item The CNN does not sufficiently learn the signal: $\max_{j,r}\gamma_{j,r}^{(t)} \leq \tilde{O}(\sigma_{0}\|\bmu\|_{2})$.
    \item The training loss converges to $\epsilon$, i.e., $L_S(\Wb^{(t)}) \leq \epsilon$.
    \item The trained CNN has a constant order test loss: $L_{\cD}({\Wb}^{(t)}) = \Theta(1)$.
\end{enumerate}
\end{theorem}
\begin{wrapfigure}{R}{0.45\textwidth}
\vskip -0.2in
     \centering
     \includegraphics[width=0.48\columnwidth]{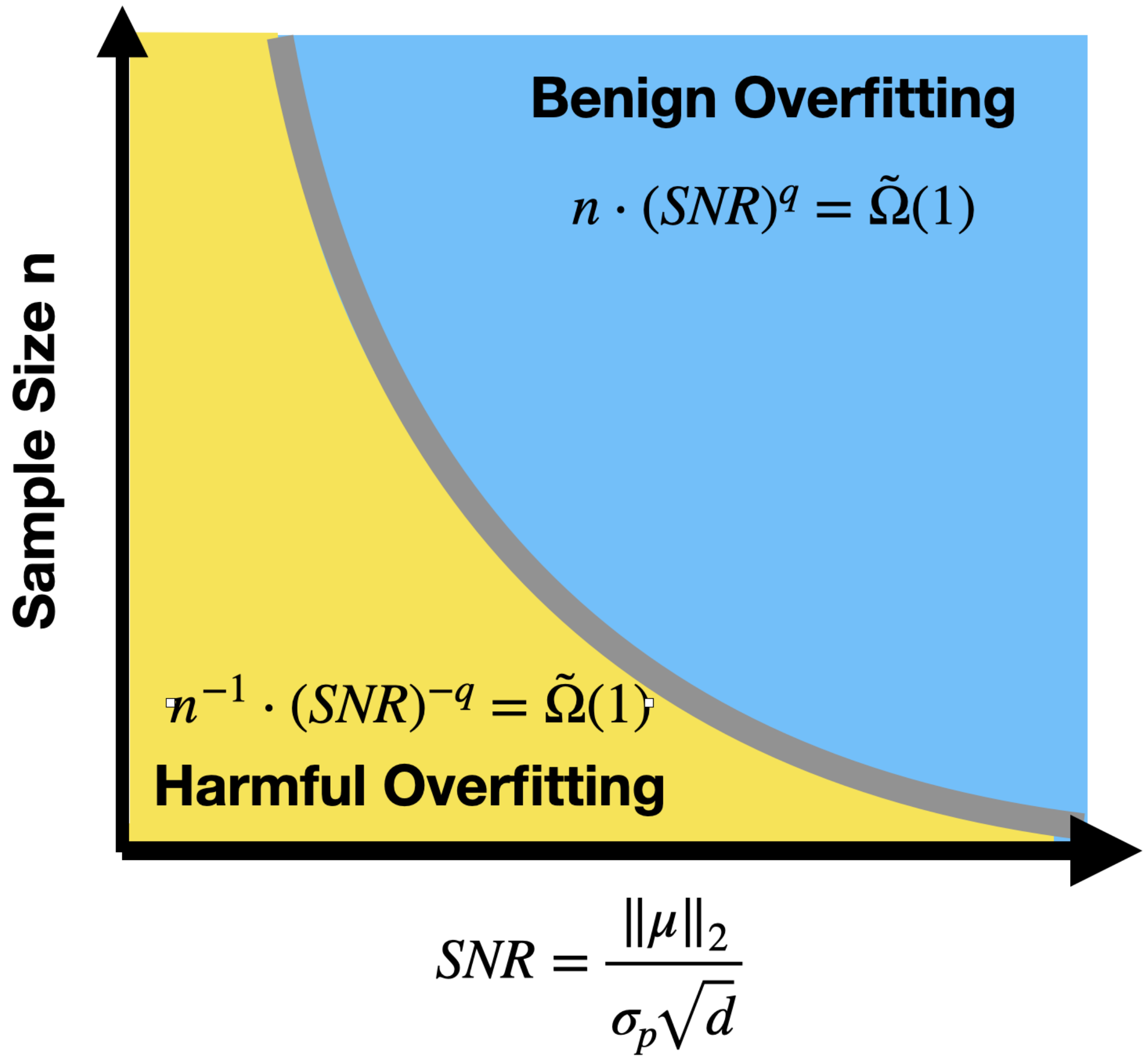}
      \vskip -0.2in
    \caption{Illustration of the phase transition between benign and harmful overfitting. The blue region represents the setting under which the overfitted CNN trained by gradient descent is guaranteed to have small population loss, and the yellow region represents the setting under which the population loss is guaranteed to be of constant order. The slim gray band region is the setting where the population loss is not well characterized.}
    \label{fig1}
    \vskip -0.5in
\end{wrapfigure}
Theorem~\ref{thm:noise_memorization_main} holds under the condition that $n^{-1}\cdot \mathrm{SNR}^{-q} = \tilde \Omega( 1)$. Clearly, this is the opposite regime (up to some logarithmic factors) compared with Theorem~\ref{thm:signal_learning_main}. In this case, the CNN trained by gradient descent mainly memorizes noises in the training data and does not learn enough signal. This, together with the results in Theorem~\ref{thm:signal_learning_main}, reveals a clear phase transition between signal learning and noise memorization in CNN training:
\begin{itemize}[leftmargin = *]
    \item If $ n\cdot \mathrm{SNR}^{q} = \tilde \Omega( 1)$, then the CNN learns the signal and achieves a $O(\epsilon + \exp(- n^{2}))$ test loss. This is the regime of benign overfitting.
    \item If $ n^{-1}\cdot \mathrm{SNR}^{-q} = \tilde \Omega( 1)$ then the CNN can only memorize noises and will have a $\Theta(1)$ test loss. This is the regime of harmful overfitting.
\end{itemize}
The phase transition is illustrated in Figure~\ref{fig1}.
Clearly, when learning a two-layer CNN on the data generated from Definition~\ref{def:data},$ n\cdot \mathrm{SNR}^{q} = \tilde \Omega( 1)$ is the precise condition under which benign overfitting occurs. Remarkably, in this case the population loss decreases \textit{exponentially} with the sample size $n$. Under our condition that $n = \Omega(\polylog(d))$, this term can also be upper bounded by $1 / \mathrm{poly}(d)$, which is small in the high-dimensional setting. Note that when $\| \bmu \|_2 = \Theta(1)$ and $\sigma_p = \Theta(d^{-1/2})$, applying standard uniform convergence based bounds \citep{bartlett2017spectrally,neyshabur2017pac} or stability based bounds \citep{hardt2016train, mou2017generalization, chen2018stability} typically gives a $\tilde O(n^{-1/2})$ bound on the generalization gap, which is vacuous when $n = O(\polylog(d))$. Our bound under the same setting is $O(1/\poly(d))$, which is non-vacuous. This is attributed to our precise analysis of signal learning and noise memorization in Theorems~\ref{thm:signal_learning_main} and \ref{thm:noise_memorization_main}.




\noindent\textbf{Comparison with neural tangent kernel (NTK) results.} 
We want to emphasize that our analysis is beyond the so-called neural tangent kernel regime. In the NTK regime, it has been shown that gradient descent can train an over-parameterized neural network to achieve good training and test accuracies \citep{jacot2018neural,du2018gradient,du2018gradientdeep,allen2018convergence,zou2019gradient,arora2019fine,cao2019generalizationsgd,chen2019much}. However, it is widely believed in literature that the NTK analyses cannot fully explain the success of deep learning, as the neural networks in the NTK regime are almost ``linearized'' \citep{lee2019wide,cao2019generalizationsgd}. Our analysis and results are not in the NTK regime: In the NTK regime, the network parameters stay close to their initialization throughout training, i.e., $\|\mathbf{W}^{(t)} - \mathbf{W}^{(0)}\|_{F} = O(1)$, so that the NN model can be approximated by its linearization \citep{allen2018convergence,cao2019generalizationsgd,chen2019much}. In comparison, our analysis does not rely on linearizing the neural network function, and $\|\mathbf{W}^{(t)} - \mathbf{W}^{(0)}\|_{F}$ can be as large as $O(\text{poly}(m))$. 




\section{Overview of Proof Technique}\label{section:tech_overview}
In this section, we discuss the main challenges in the study of CNN training under our setting, and explain some key techniques we implement in our proofs to overcome these challenges. The complete proofs of all the results are given in the appendix.

\noindent\textbf{Main challenges.} Studying benign overfitting under our setting is a challenging task. The first challenge 
is the nonconvexity of the training objective function $L_S(\Wb)$. Nonconvexity has introduced new challenges in the study of benign overfitting particularly because our goal is not only to show the convergence of the training loss, but also to study the population loss in the over-parameterized setting, which requires a precise algorithmic analysis of the learning problem.

\subsection{Iterative Analysis of the Signal-Noise Decomposition}
In order to study the learning process based on the nonconvex optimization problem, we propose a key technique which enables the iterative analysis of the coefficients in the signal-noise decomposition in Definition~\ref{def:w_decomposition}. This technique is given in the following lemma. 
\begin{lemma}\label{lemma:coefficient_iterative}
The coefficients $\gamma_{j,r}^{(t)},\zeta_{j,r,i}^{(t)},\omega_{j,r,i}^{(t)}$ in Definition~\ref{def:w_decomposition} satisfy the following  equations:
\begin{align}
    &\gamma_{j,r}^{(0)},\zeta_{j,r,i}^{(0)},\omega_{j,r,i}^{(0)} = 0,\label{eq:update_initial}\\
    &\gamma_{j,r}^{(t+1)} = \gamma_{j,r}^{(t)} - \frac{\eta}{nm} \cdot \sum_{i=1}^n \ell_i'^{(t)} \cdot \sigma'(\la\wb_{j,r}^{(t)}, y_{i} \cdot \bmu\ra) \cdot \| \bmu \|_2^2, \label{eq:update_gamma1}\\
    &\zeta_{j,r,i}^{(t+1)} = \zeta_{j,r,i}^{(t)} - \frac{\eta}{nm} \cdot \ell_i'^{(t)}\cdot \sigma'(\la\wb_{j,r}^{(t)}, \bxi_{i}\ra) \cdot \| \bxi_i \|_2^2 \cdot \ind(y_{i} = j), \label{eq:update_zeta1}\\
    &\omega_{j,r,i}^{(t+1)} = \omega_{j,r,i}^{(t)} + \frac{\eta}{nm} \cdot \ell_i'^{(t)}\cdot \sigma'(\la\wb_{j,r}^{(t)}, \bxi_{i}\ra) \cdot \| \bxi_i \|_2^2 \cdot \ind(y_{i} = -j).\label{eq:update_omega1}
\end{align}
\end{lemma}
\begin{remark}
With the decomposition \eqref{eq:w_decomposition}, the signal learning and noise memorization processes of a CNN can be formally studied by analyzing the dynamics of $\gamma_{j,r}^{(t)},\zeta_{j,r,i}^{(t)}, \omega_{j,r,i}^{(t)}$ based on the dynamical system \eqref{eq:update_gamma1}-\eqref{eq:update_omega1}. Note that prior to our work, several existing results have utilized the inner products $\la \wb_{j,r}^{(t)}, \bmu \ra$ during the neural network training process in order to establish generalization bounds \citep{brutzkus2017sgd,chatterji2020finite,frei2021provable}. Similar inner product based arguments are also implemented in \citet{allen2020feature,allen2020towards,zou2021understanding}, which study different topics related to learning neural networks. Compared with the inner product based argument, our method has two major advantages: (i) Based on the definition \eqref{eq:update_gamma1}-\eqref{eq:update_omega1} and the fact that $ \ell_i'^{(t)} < 0$, it is clear that $\gamma_{j,r}^{(t)},\zeta_{j,r,i}^{(t)}$ are monotonically increasing, while $\omega_{j,r,i}^{(t)}$ is monotonically decreasing throughout the whole training process. In comparison, monotonicity does not hold in the inner product based argument, especially for $\la \wb_{j,r}^{(t)}, \bxi_i \ra$. (ii) Our signal-noise decomposition also enables a clean homogeneity-based proof for the convergence of the training loss to an arbitrarily small error rate $\epsilon > 0$, which will be presented in Subsection~\ref{subsection:twostage}.
\end{remark}

With Lemma~\ref{lemma:coefficient_iterative}, we can reduce the study of the CNN learning process to the analysis of the discrete dynamical system given by \eqref{eq:update_initial}-\eqref{eq:update_omega1}. 
Our proof then focuses on a careful assessment of the values of the coefficients $\gamma_{j,r}^{(t)},\zeta_{j,r,i}^{(t) }, \omega_{j,r,i}^{(t)}$ throughout training.
To prepare for more detailed analyses, we first present the following bounds of the coefficients, which hold throughout training.

\begin{proposition}\label{Prop:main1}
Under Condition~\ref{condition:d_sigma0_eta}, for any $T^{*} = \eta^{-1}\text{poly}(\epsilon^{-1}, \|\bmu\|_{2}^{-1}, d^{-1}\sigma_{p}^{-2},\sigma_{0}^{-1}, n, m, d)$, the following bounds hold for $t \in [0, T^*]$:
\begin{itemize}[leftmargin = *]
\item $0\leq \gamma_{j,r}^{(t)}, \zeta_{j,r,i}^{(t)} \leq 4\log(T^{*})$ for all $j\in \{\pm 1\}$, $r\in [m]$ and $i\in [n]$.
\item $0\geq \omega_{j,r,i}^{(t)} \geq -2\max_{i,j,r}\{|\la \wb_{j,r}^{(0)}, \bmu\ra|,|\la \wb_{j,r}^{(0)}, \bxi_{i}\ra|\} - 16n\sqrt{\frac{\log(4n^{2}/\delta)}{d}}\cdot 4\log(T^{*})$  for all $j\in \{\pm 1\}$, $r\in [m]$ and $i\in [n]$.
\end{itemize}
\end{proposition}
We can then prove the following lemma, which demonstrates that the training objective function $L_S(\Wb)$ can dominate the gradient norm $\|\nabla L_{S}(\Wb^{(t)})\|_{F}$ along the gradient descent path. 
\begin{lemma}\label{lm: gradient upbound sketch}
Under Condition~\ref{condition:d_sigma0_eta}, for any $T^{*} = \eta^{-1}\text{poly}(\epsilon^{-1}, \|\bmu\|_{2}^{-1}, d^{-1}\sigma_{p}^{-2},\sigma_{0}^{-1}, n, m, d)$, the following result holds for $t\in [0,T^*]$: 
\begin{align*}
\|\nabla L_{S}(\Wb^{(t)})\|_{F}^{2} = O\big(\max\{\|\bmu\|_{2}^{2}, \sigma_{p}^{2}d\}\big)\cdot L_{S}(\Wb^{(t)}).
\end{align*}
\end{lemma}
Lemma~\ref{lm: gradient upbound sketch} plays a key role in the convergence proof of training loss function. However, note that our study of benign overfitting requires carefully monitoring the changes of the coefficients in the signal-noise decomposition, which cannot be directly done by Lemma~\ref{lm: gradient upbound sketch}. 
This is quite a challenging task, due to the complicated interactions among $\gamma_{j,r}^{(t)}$, $\zeta_{j,r,i}^{(t)}$ and $\omega_{j,r,i}^{(t)}$. Note that even $\gamma_{j,r}^{(t)}$, which has the simplest formula \eqref{eq:update_gamma1}, depends on \textit{all} the quantities $\gamma_{j',r'}^{(t)}$, $\zeta_{j',r',i}^{(t)}$ and $\omega_{j',r',i}^{(t)}$ for $j'\in \{\pm 1\}$, $r'\in [m]$ and $i\in[n]$. This is because the cross-entropy loss derivative term $\ell_i'^{(t)} = \ell'[ y_i \cdot f(\Wb^{(t)},\xb_i)]$ depends on all the neurons of the network. 
To overcome this challenge, we introduce in the next subsection a decoupling technique based on a two-stage analysis. 




\subsection{Decoupling with a Two-Stage Analysis.}\label{subsection:twostage}
We utilize a two-stage analysis to decouple the complicated relation among the coefficients $\gamma_{j,r}^{(t)}$, $\zeta_{j,r,i}^{(t)}$ and $\omega_{j,r,i}^{(t)}$. Intuitively, the initial neural network weights are small enough so that the neural network at initialization has constant level cross-entropy loss derivatives on all the training data: $\ell_i'^{(0)} = \ell'[ y_i \cdot f(\Wb^{(0)},\xb_i)] = \Theta(1)$ for all $i\in [n]$. This is guaranteed under Condition~\ref{condition:d_sigma0_eta} and matches neural network training in practice. Motivated by this, we can consider the first stage of the training process where $ \ell_i'^{(t)} = \Theta(1) $, in which case we can show significant scale differences among $\gamma_{j,r}^{(t)}$, $\zeta_{j,r,i}^{(t)}$ and $\omega_{j,r,i}^{(t)}$. Based on the result in the first stage, we then proceed to the second stage of the training process where the loss derivatives are no longer at a constant level and show that the training loss can be optimized to be arbitrarily small and meanwhile, the scale differences shown in the first learning stage remain the same throughout the training process. In the following, we focus on explaining the key proof steps for Theorem~\ref{thm:signal_learning_main}. The proof idea for Theorem~\ref{thm:noise_memorization_main} is similar, so we defer the details to the appendix.

\noindent\textbf{\textit{Stage 1.}}
It can be shown that, until some of the coefficients $\gamma_{j,r}^{(t)}$, $\rho_{j,r,i}^{(t)}$ reach $\Theta(1)$, we have $\ell_i'^{(t)} = \ell'[ y_i \cdot f(\Wb^{(t)},\xb_i)] = \Theta(1)$ for all $i\in [n]$. Therefore, we first focus on this first stage of the training process, where the dynamics of the coefficients in  \eqref{eq:update_gamma1} - \eqref{eq:update_omega1} can be greatly simplified by replacing the $\ell_i'^{(t)}$ factors by their constant upper and lower bounds. The following lemma summarizes our main conclusion at stage 1 for signal learning: 
\begin{lemma}\label{lemma:phase1_main_sketch}
Under the same conditions as Theorem~\ref{thm:signal_learning_main}, there exists $T_1 = \tilde O(\eta^{-1}m\sigma_{0}^{2-q}\|\bmu\|_{2}^{-q})$ such that 
\begin{itemize}[leftmargin = *]
\item $\max_{ r}\gamma_{j, r}^{(T_{1})} = \Omega(1)$ for $j\in \{\pm 1\}$.
\item $|\rho_{j,r,i}^{(t)}| = O(\sigma_0 \sigma_p \sqrt{d} )$ for all $j\in \{\pm 1\}$, $r\in[m]$, $i \in [n]$ and $0\leq t\leq T_1$. 
\end{itemize}
\end{lemma}
Lemmas~\ref{lemma:phase1_main_sketch}  takes advantage of the training period when the loss function derivatives remain a constant order to show that the CNN can capture the signal. 
At the end of stage 1 in signal learning, $\max_{r}\gamma_{j,r}$ reaches $\Theta(1)$, and is significantly larger than $\rho_{j,r,i}^{(t)}$. After this, it is no longer guaranteed that the loss derivatives $\ell_i'^{(t)} $ will remain constant order, and thus starts the training stage 2. 



\noindent\textbf{\textit{Stage 2.}}
In this stage, we take into full consideration the exact definition $\ell_i'^{(t)} = \ell'[ y_i \cdot f(\Wb^{(t)},\xb_i)]$ and show that the training loss function will converge to $L_S(\Wb^{(t)}) < \epsilon$.
Thanks to the analysis in stage 1, we know that some $\gamma_{j,r}^{(t)}$ is significantly larger than all $\rho_{j,r,i}^{(t)}$'s at the end of stage 1. 
This scale difference is the key to our analysis in stage 2.
Based on this scale difference and the monotonicity of $\gamma_{j,r}^{(t)}$, $\zeta_{j,r,i}^{(t)}$,  $\omega_{j,r,i}^{(t)}$ in the signal-noise decomposition, it can be shown that there exists $\Wb^*$ such that $y_i\cdot \la \nabla f(\Wb^{(t)},\xb_{i}), \Wb^{*} \ra \geq q\log(2q/\epsilon)$ throughout stage 2.
Moreover, since the neural network $f(\Wb, \xb)$ is $q$-homogeneous in $\Wb$, we have $\la \nabla f(\Wb^{(t)}, \xb) , \Wb^{(t)}\ra = q\cdot f(\Wb^{(t)}, \xb)$. Therefore, 
\begin{align*}
    \la \nabla L_{S}(\Wb^{(t)}), \Wb^{(t)} - \Wb^{*}\ra &= \frac{1}{n} \sum_{i=1}^n \ell_i'^{(t)} \cdot y_i \cdot \la \nabla f(\Wb^{(t)}, \xb_{i}) , \Wb^{(t)} - \Wb^{*}\ra \\
    & = \frac{1}{n} \sum_{i=1}^n \ell_i'^{(t)}  \cdot [ y_i \cdot  q\cdot f(\Wb^{(t)}, \xb_{i}) - y_i \cdot \la \nabla f(\Wb^{(t)}, \xb_{i}) , \Wb^{*}\ra  ]\\
    &\geq \frac{1}{n}\sum_{i=1}^{n}\ell'[y_{i}\cdot  f(\Wb^{(t)}, \xb_{i}) ] \cdot [y_{i}\cdot q\cdot f(\Wb^{(t)}, \xb_{i}) - q\log(2q/\epsilon)]\\
    &\geq q\cdot \frac{1}{n}\sum_{i=1}^{n} [ \ell( f(\Wb^{(t)}, \xb_i)) - \ell( \log(2q/\epsilon) ) ]\\
    &\geq q\cdot L_S(\Wb^{(t)}) - \epsilon / 2,
\end{align*}
where the second inequality follows by the convexity of the cross-entropy loss function. With the above key technique, we can prove the following lemma.
\begin{lemma}\label{lemma:signal_proof_sketch}
Let $T,T_1$ be defined in Theorem~\ref{thm:signal_learning_main} and Lemma~\ref{lemma:phase1_main_sketch} respectively. Then under the same conditions as Theorem~\ref{thm:signal_learning_main}, for any $t\in [T_1,  T]$,  it holds that $|\rho_{j,r,i}^{(t)}| \leq \sigma_{0}\sigma_{p}\sqrt{d}$
for all $j\in\{\pm 1\}$, $r\in [m]$ and $i\in[n]$. 
Moreover, let $\Wb^*$ be the collection of CNN parameters with convolution filters $\wb^{*}_{j,r} = \wb_{j,r}^{(0)} + 2qm\log(2q/\epsilon) \cdot j \cdot \|\bmu\|_{2}^{-2} \cdot \bmu$. Then the following bound holds
\begin{align*}
\frac{1}{t - T_{1} + 1}\sum_{s=T_{1}}^{t}L_{S}(\Wb^{(s)}) \leq  \frac{\|\Wb^{(T_{1})} - \Wb^{*}\|_{F}^{2}}{(2q-1) \eta(t - T_{1} + 1)} + \frac{\epsilon}{(2q-1)}
\end{align*}
for all $t\in [T_1,  T]$, where we denote $ \| \Wb \|_F = \sqrt{ \| \Wb_{+1} \|_F^2 + \| \Wb_{-1} \|_F^2  } $.
\end{lemma}
Lemma~\ref{lemma:signal_proof_sketch} states two main results on signal learning. First of all, during this training period, it is guaranteed that the coefficients of noise vectors $\rho_{j,r,i}^{(t)}$ in the signal-noise decomposition remain sufficiently small. Moreover, it also gives an optimization type result that the best iterate in $[T_1,T]$ is small as long as $T$ is large enough. 

Clearly, the convergence of the training loss stated in Theorems~\ref{thm:signal_learning_main} directly follows by choosing $T$ to be sufficiently large in Lemmas~\ref{lemma:signal_proof_sketch}. The lemma below further gives an upper bound on the test loss.
\begin{lemma}\label{lemma:signal_polulation_loss_main}
Let $T$ be defined in Theorem~\ref{thm:signal_learning_main}.  Under the same conditions as Theorem~\ref{thm:signal_learning_main}, for any $t \leq T$ with $L_{S}(\Wb^{(t)}) \leq 1$, it holds that $L_{\cD}(\Wb^{(t)}) \leq 6 \cdot L_{S}(\Wb^{(t)}) + \exp(- n^{2})$.
\end{lemma}
Below we finalize the proof of Theorem~\ref{thm:signal_learning_main}. The proofs of other results are in the appendix.
\begin{proof}[Proof of Theorem~\ref{thm:signal_learning_main}] 
The first part of Theorem~\ref{thm:signal_learning_main} follows by Lemma~\ref{lemma:phase1_main_sketch} and the monotonicity of $\gamma_{j, r}^{(t)}$. 
The second part of  Theorem~\ref{thm:signal_learning_main} follows by Lemma~\ref{lemma:signal_proof_sketch}. 
For the third part, let $\Wb^*$ be defined in Lemma~\ref{lemma:signal_proof_sketch}. 
Then by the definition of $\Wb^*$, we have
\begin{align*}
\|\Wb^{(T_{1})} - \Wb^{*}\|_{F} &\leq \|\Wb^{(T_{1})} - \Wb^{(0)}\|_{F} + \|\Wb^{(0)} - \Wb^{*}\|_{F} \\
&\leq \sum_{j,r}\gamma_{j,r}^{(T_{1})}\|\bmu\|_{2}^{-1} + \sum_{j,r,i}\frac{\zeta_{j,r,i}^{(T_{1})}}{\|\bxi_{i}\|_{2}} + \sum_{j,r,i}\frac{\omega_{j,r,i}^{(T_{1})}}{\|\bxi_{i}\|_{2}} + \Theta(m^{3/2}\log(1/\epsilon))\|\bmu\|_{2}^{-1}\\
&= \tilde{O}(m^{3/2}\|\bmu\|_{2}^{-1}),
\end{align*}
where the first inequality is by triangle inequality, the second inequality is by the signal-noise decomposition of $\Wb^{(T_{1})}$ and the definition of $ \Wb^{*}$, and the last equality is by Proposition~\ref{Prop:main1} and  Lemma~\ref{lemma:phase1_main_sketch}.
Therefore, choosing $T = \tilde{\Theta}( \eta^{-1} T_1 +  \eta^{-1}\epsilon^{-1} m^{3}\| \bmu \|_2^{-2}) =\tilde{\Theta}( \eta^{-1} \sigma_0 ^{-(q-2)} \| \bmu \|_2^{-q} +  \eta^{-1}\epsilon^{-1} m^{3}\| \bmu \|_2^{-2})$ in Lemma~\ref{lemma:signal_proof_sketch} ensures that
\begin{align*}
\frac{1}{T - T_{1} + 1}\sum_{t=T_{1}}^{T }L_{S}(\Wb^{(t)}) \leq \frac{\|\Wb^{(T_{1})} - \Wb^{*}\|_{F}^{2}}{(2q-1) \eta(T - T_{1} + 1)} + \frac{\epsilon}{2q-1} \leq \frac{\tilde{O}(m^{3}\|\bmu\|_{2}^{-2}) }{(2q-1) \eta(T - T_{1} + 1)} + \frac{\epsilon}{2q-1} \leq \epsilon, 
\end{align*}
and there exists $t \in [T_1 , T]$ such that $L_S(\Wb^{(t)}) \leq \epsilon$. This completes the proof of the third part of Theorem~\ref{thm:signal_learning_main}. Finally, combining this bound with Lemma~\ref{lemma:signal_polulation_loss_main} gives
\begin{align*}
    L_{\cD}(\Wb^{(t)}) \leq 6 \cdot L_{S}(\Wb^{(t)}) + \exp(- n^{2}) \leq  6 \epsilon + \exp(- n^{2}),
\end{align*}
which proves the last part of Theorem~\ref{thm:signal_learning_main}.
\end{proof}

\section{Conclusion and Future Work}
This paper utilizes a signal-noise decomposition to study the signal learning and noise memorization process in the training of a two-layer CNN. We precisely give the conditions under which the CNN will mainly focus on learning signals or memorizing noises, and reveals a phase transition of the population loss with respect to the sample size, signal strength, noise level, and dimension. Our result theoretically demonstrates that benign overfitting can happen in neural network training. An important future work direction is to study the benign overfitting phenomenon of neural networks in learning other data models. Moreover, it is also important to generalize our analysis to deep convolutional neural networks.

\section*{Acknowledgements}
We would like to thank Spencer Frei for valuable comment and discussion on the earlier version of this paper, and pointing out a related work.

\appendix



\section{Additional Related Work}
There has also been a large number of works studying the optimization and generalization of neural networks. A series of work \citep{ li2017convergence,soltanolkotabi2017learning,du2017convolutional, du2017gradient,zhong2017recovery,zhang2018learning,cao2019tight} studied the parameter recovery problem in two-layer neural networks, where the data are given by a teacher network and the task is to recover the parameters in the teacher network. These works either focus on the noiseless setting, or requires the number of training data points to be larger than the number of parameters in the network, and therefore does not cover the setting where the neural network can overfit the training data. Another line of works \citep{neyshabur2015norm,bartlett2017spectrally,neyshabur2017pac,golowich2017size,arora2018stronger} have studied the generalization gap between the training and test losses of neural networks with uniform convergence based arguments. However, these results are not algorithm-dependent and cannot explain benign overfitting. Some recent works studied the generalization gap based on stability based arguments  \citep{bousquet2002stability,hardt2016train, mou2017generalization, chen2018stability}. A more recent line of works studied the convergence \citep{jacot2018neural,li2018learning,du2018gradient,allen2018convergence,du2018gradientdeep,zou2019gradient} and test error bounds \citep{allen2018learning,arora2019fine,arora2019exact,cao2019generalizationsgd,ji2019polylogarithmic,chen2019much} of over-parameterized networks in the neural tangent kernel regime. However, these works depend on the equivalence between neural network training and kernel methods, which cannot fully explain the success of deep learning. Compared with the works mentioned above, our work has a different focus which is to study the conditions for benign and harmful overfitting.

\section{Preliminary Lemmas}\label{sec: initial}

In this section, we present some pivotal lemmas that give some important properties of the data and the neural network parameters at their random initialization.

\begin{lemma}\label{lemma:numberofdata}
Suppose that $\delta > 0$ and $n \geq 8 \log(4/\delta)$. Then with probability at least $1 - \delta$, 
\begin{align*}
    |\{i \in [n]: y_i = 1\}|,~ |\{i \in [n]:y_i = -1\}| \geq n/4.
\end{align*}
\end{lemma}
\begin{proof}[Proof of Lemma~\ref{lemma:numberofdata}] By Hoeffding's inequality, with probability at least $1 - \delta / 2$, we have
\begin{align*}
    \Bigg| \frac{1}{n}\sum_{i=1}^n \ind\{y_i = 1\}  - \frac{1}{2} \Bigg| \leq \sqrt{\frac{\log(4/\delta)}{2n}}.
\end{align*}
Therefore, as long as $n \geq 8 \log(4/\delta)$, we have
\begin{align*}
    |\{i \in [n]: y_i = 1\}| = \sum_{i=1}^n \ind\{y_i = 1\} \geq \frac{n}{2} - n\cdot  \sqrt{\frac{ \log(4/\delta)}{2n}} \geq \frac{n}{4}.
\end{align*}
This proves the result for $|\{i \in [n]: y_i = 1\}|$. The proof for $|\{i \in [n]: y_i = -1\}|$ is exactly the same, and we can conclude the proof by applying a union bound. 
\end{proof}

The following lemma estimates the norms of the noise vectors $\bxi_i$, $i\in [n]$, and gives an upper bound of their inner products with each other.

\begin{lemma}\label{lemma:data_innerproducts}
Suppose that $\delta > 0$ and $d = \Omega( \log(4n / \delta) ) $. Then with probability at least $1 - \delta$, 
\begin{align*}
    &\sigma_p^2 d / 2\leq \| \bxi_i \|_2^2 \leq 3\sigma_p^2 d / 2,\\
    & |\la \bxi_i, \bxi_{i'} \ra| \leq 2\sigma_p^2 \cdot \sqrt{d \log(4n^2 / \delta)}
\end{align*}
for all $i,i'\in [n]$.
\end{lemma}
\begin{proof}[Proof of Lemma~\ref{lemma:data_innerproducts}] By Bernstein's inequality, with probability at least $1 - \delta / (2n)$ we have
\begin{align*}
    \big| \| \bxi_i \|_2^2 - \sigma_p^2 d \big| = O(\sigma_p^2 \cdot \sqrt{d \log(4n / \delta)}).
\end{align*}
Therefore, as long as $d = \Omega( \log(4n / \delta) )$, we have 
\begin{align*}
     \sigma_p^2 d /2  \leq \| \bxi_i \|_2^2 \leq 3\sigma_p^2 d / 2.
\end{align*}
Moreover, clearly $\la \bxi_i, \bxi_{i'} \ra$ has mean zero. 
For any $i,i'$ with $i\neq i'$, by Bernstein's inequality, with probability at least $1 - \delta / (2n^2)$ we have
\begin{align*}
    | \la \bxi_i, \bxi_{i'} \ra| \leq 2\sigma_p^2 \cdot \sqrt{d \log(4n^2 / \delta)}.
\end{align*}
Applying a union bound completes the proof.
\end{proof}

The following lemma studies the inner product between a randomly initialized CNN convolutional filter $\wb_{j,r}^{(0)}$, $j\in \{+1, -1\}$ and $r\in [m]$ and the signal/noise vectors in the training data. The calculations characterize how the neural network at initialization randomly captures signal and noise information.

\begin{lemma}\label{lemma:initialization_norms} Suppose that $d \geq \Omega(\log(mn/\delta))$, $ m = \Omega(\log(1 / \delta))$. Then with probability at least $1 - \delta$, 
\begin{align*}
    &|\la \wb_{j,r}^{(0)}, \bmu \ra | \leq \sqrt{2 \log(8m/\delta)} \cdot \sigma_0 \| \bmu \|_2,\\
    &| \la \wb_{j,r}^{(0)}, \bxi_i \ra | \leq 2\sqrt{ \log(8mn/\delta)}\cdot \sigma_0 \sigma_p \sqrt{d} 
\end{align*}
for all $r\in [m]$,  $j\in \{\pm 1\}$ and $i\in [n]$. Moreover, 
\begin{align*}
    &\sigma_0 \| \bmu \|_2 / 2 \leq \max_{r\in[m]} j\cdot \la \wb_{j,r}^{(0)}, \bmu \ra \leq \sqrt{2 \log(8m/\delta)} \cdot \sigma_0 \| \bmu \|_2,\\
    &\sigma_0 \sigma_p \sqrt{d} / 4 \leq \max_{r\in[m]} j\cdot \la \wb_{j,r}^{(0)}, \bxi_i \ra \leq 2\sqrt{ \log(8mn/\delta)} \cdot \sigma_0 \sigma_p \sqrt{d}
\end{align*}
for all $j\in \{\pm 1\}$ and $i\in [n]$.
\end{lemma}
\begin{proof}[Proof of Lemma~\ref{lemma:initialization_norms}]
It is clear that for each $r\in [m]$, $j\cdot \la \wb_{j,r}^{(0)}, \bmu \ra$ is a Gaussian random variable with mean zero and variance $\sigma_0^2 \| \bmu \|_2^2$. Therefore, by Gaussian tail bound and union bound, with probability at least $1 - \delta/4$, 
\begin{align*}
    j\cdot \la \wb_{j,r}^{(0)}, \bmu \ra \leq |\la \wb_{j,r}^{(0)}, \bmu \ra| \leq \sqrt{2\log(8m/\delta)} \cdot \sigma_0 \| \bmu \|_2.
\end{align*}
Moreover, $\PP( \sigma_0 \| \bmu \|_2 / 2 > j\cdot \la \wb_{j,r}^{(0)}, \bmu \ra )$ is an absolute constant, and therefore by the condition on $m$, we have
\begin{align*}
    \PP\big( \sigma_0 \| \bmu \|_2 / 2 \leq \max_{r\in[m]} j\cdot \la \wb_{j,r}^{(0)}, \bmu \ra ) &= 1 - \PP( \sigma_0 \| \bmu \|_2 / 2 > \max_{r\in[m]} j\cdot \la \wb_{j,r}^{(0)}, \bmu \ra \big) \\
    &= 1 - \PP\big( \sigma_0 \| \bmu \|_2 / 2 > j\cdot \la \wb_{j,r}^{(0)}, \bmu \ra \big)^{2m} \\
    &\geq 1 - \delta / 4.
\end{align*}
By Lemma~\ref{lemma:data_innerproducts}, with probability at least $1 - \delta / 4$, $ \sigma_p \sqrt{d} / \sqrt{2} \leq \| \bxi_i \|_2 \leq \sqrt{3/2}\cdot \sigma_p \sqrt{d}$ for all $i\in [n]$. Therefore, the result for $\la \wb_{j,r}^{(0)}, \bxi_i\ra$ follows the same proof as $j\cdot \la \wb_{j,r}^{(0)}, \bmu \ra$. 
\end{proof}

\section{Signal-noise Decomposition Analysis}\label{section:decompositionproof}
In this section, we establish a series of results on the signal-noise decomposition. These results are based on the conclusions in Section~\ref{sec: initial}, which hold with high probability. Denote by  $\cE_{\mathrm{prelim}}$ the event that all the results in Section~\ref{sec: initial} hold. Then for simplicity and clarity, we state all the results in this and the following sections conditional on  $\cE_{\mathrm{prelim}}$. 
\begin{lemma}[Restatement of Lemma~\ref{lemma:coefficient_iterative}]\label{lemma:coefficient_iterative_proof}
The coefficients $\gamma_{j,r}^{(t)},\zeta_{j,r,i}^{(t)},\omega_{j,r,i}^{(t)}$ defined in Definition~\ref{def:w_decomposition} satisfy the following iterative equations:
\begin{align*}
    &\gamma_{j,r}^{(0)},\zeta_{j,r,i}^{(0)},\omega_{j,r,i}^{(0)} = 0,\\
    &\gamma_{j,r}^{(t+1)} = \gamma_{j,r}^{(t)} - \frac{\eta}{nm} \cdot \sum_{i=1}^n \ell_i'^{(t)} \cdot \sigma'(\la\wb_{j,r}^{(t)}, y_{i} \cdot \bmu\ra) \cdot \| \bmu \|_2^2, \\
    &\zeta_{j,r,i}^{(t+1)} = \zeta_{j,r,i}^{(t)} - \frac{\eta}{nm} \cdot \ell_i'^{(t)}\cdot \sigma'(\la\wb_{j,r}^{(t)}, \bxi_{i}\ra) \cdot \| \bxi_i \|_2^2 \cdot \ind(y_{i} = j), \\
    &\omega_{j,r,i}^{(t+1)} = \omega_{j,r,i}^{(t)} + \frac{\eta}{nm} \cdot \ell_i'^{(t)}\cdot \sigma'(\la\wb_{j,r}^{(t)}, \bxi_{i}\ra) \cdot \| \bxi_i \|_2^2 \cdot \ind(y_{i} = -j)
\end{align*}
for all $r\in [m]$,  $j\in \{\pm 1\}$ and $i\in [n]$.
\end{lemma}
\begin{proof}[Proof of Lemma~\ref{lemma:coefficient_iterative_proof}]
By our data model in Definition~\ref{def:data} and Gaussian initialization of the CNN weights, it is clear that with probability $1$, the vectors are linearly independent. Therefore, the decomposition \eqref{eq:w_decomposition} is unique. Now consider $\tilde\gamma_{j,r}^{(0)},\tilde\rho_{j,r,i}^{(0)} = 0$ and 
\begin{align*}
    &\tilde\gamma_{j,r}^{(t+1)} = \tilde\gamma_{j,r}^{(t)} - \frac{\eta}{nm} \cdot \sum_{i=1}^n \ell_i'^{(t)} \cdot \sigma'(\la\wb_{j,r}^{(t)}, y_{i} \cdot \bmu\ra) \cdot \| \bmu \|_2^2,\\
    &\tilde\rho_{j,r,i}^{(t+1)} = \tilde\rho_{j,r,i}^{(t)} - \frac{\eta}{nm} \cdot \ell_i'^{(t)}\cdot \sigma'(\la\wb_{j,r}^{(t)}, \bxi_{i}\ra) \cdot \| \bxi_i \|_2^2 \cdot jy_{i},\\
\end{align*}
It is then easy to check by \eqref{eq:gdupdate} that 
\begin{align*}
    \wb_{j,r}^{(t)} = \wb_{j,r}^{(0)} + j \cdot \tilde\gamma_{j,r}^{(t)} \cdot \| \bmu \|_2^{-2} \cdot \bmu + \sum_{ i = 1}^n \tilde\rho_{j,r,i}^{(t)} \| \bxi_i \|_2^{-2} \cdot \bxi_{i}.
\end{align*}
Hence by the uniqueness of the decomposition we have $\gamma_{j,r}^{(t)} = \tilde\gamma_{j,r}^{(t)}$ and $\rho_{j,r,i}^{(t) } = \tilde\rho_{j,r,i}^{(t)}$. Then we have that 
\begin{align*}
\rho_{j,r,i}^{(t)} = - \sum_{s=0}^{t-1}\frac{\eta}{nm} \cdot \ell_i'^{(s)}\cdot \sigma'(\la\wb_{j,r}^{(s)}, \bxi_{i}\ra) \cdot \| \bxi_i \|_2^2 \cdot jy_{i}    
\end{align*}
Moreover, note that $\ell_i'^{(t)} < 0$ by the definition of the cross-entropy loss. Therefore,
\begin{align}
&\zeta_{j,r,i}^{(t)} = - \sum_{s=0}^{t-1}\frac{\eta}{nm} \cdot \ell_i'^{(s)}\cdot \sigma'(\la\wb_{j,r}^{(s)}, \bxi_{i}\ra) \cdot \| \bxi_i \|_2^2 \cdot \ind(y_{i} = j),\label{eq:iterate1}\\
&\omega_{j,r,i}^{(t)} = - \sum_{s=0}^{t-1}\frac{\eta}{nm} \cdot \ell_i'^{(s)}\cdot \sigma'(\la\wb_{j,r}^{(s)}, \bxi_{i}\ra) \cdot \| \bxi_i \|_2^2 \cdot \ind(y_{i} = -j).\label{eq:iterate2}
\end{align}
Writing out the iterative versions of \eqref{eq:iterate1} and \eqref{eq:iterate2} completes the proof.
\end{proof}

We can futher plug the signal-noise decomposition \eqref{eq:w_decomposition} into the iterative formulas in Lemma~\ref{lemma:coefficient_iterative_proof}. By the second equation in Lemma~\ref{lemma:coefficient_iterative_proof}, we have
\begin{align}
    \gamma_{j,r}^{(t+1)} &= \gamma_{j,r}^{(t)} - \frac{\eta}{nm} \cdot \sum_{i=1}^n \ell_i'^{(t)} \cdot \sigma'( y_{i} \cdot  \la\wb_{j,r}^{(0)}, \bmu \ra +  y_{i} \cdot j \cdot \gamma_{j,r}^{(t)} )\cdot \|\bmu\|_{2}^{2}, \label{eq:update_gamma2}
\end{align}
Moreover, by the third equation in Lemma~\ref{lemma:coefficient_iterative_proof}, we have
\begin{align}
    \zeta_{j,r,i}^{(t+1)} &= \zeta_{j,r,i}^{(t)} - \frac{\eta}{nm} \cdot \ell_i'^{(t)} \sigma'\Bigg(\la\wb_{j,r}^{(0)}, \bxi_{i}\ra + \sum_{ i'= 1 }^n \zeta_{j,r,i'}^{(t)} \frac{\la \bxi_{i'}, \bxi_i \ra}{\| \bxi_{i'} \|_2^2} + \sum_{ i' = 1}^n \omega_{j,r,i'}^{(t)} \frac{\la \bxi_{i'}, \bxi_i \ra}{\| \bxi_{i'} \|_2^2} \Bigg)\cdot \| \bxi_{i} \|_2^2  \label{eq:update_zeta2}
\end{align}
if $j = y_i$, and $\zeta_{j,r,i}^{(t)} = 0$ for all $t\geq 0$ if  $j = - y_i$. Similarly, by the last equation in Lemma~\ref{lemma:coefficient_iterative_proof}, we have 
\begin{align}
    \omega_{j,r,i}^{(t+1)} &= \omega_{j,r,i}^{(t)} + \frac{\eta}{nm} \cdot \ell_i'^{(t)} \sigma'\Bigg(\la\wb_{j,r}^{(0)}, \bxi_{i}\ra + \sum_{ i'= 1 }^n \zeta_{j,r,i'}^{(t)} \frac{\la \bxi_{i'}, \bxi_i \ra}{\| \bxi_{i'} \|_2^2} + \sum_{ i' = 1}^n \omega_{j,r,i'}^{(t)} \frac{\la \bxi_{i'}, \bxi_i \ra}{\| \bxi_{i'} \|_2^2} \Bigg) \cdot \| \bxi_{i} \|_2^2 \label{eq:update_omega2}
\end{align}
if $j = - y_i$, and $\omega_{j,r,i}^{(t)} = 0$ for all $t\geq 0$ if  $j = y_i$.

We will now show that the parameter of the signal-noise decomposition will stay a reasonable scale during a long time of training. Let us consider the learning period $0 \leq t \leq T^{*}$, where $T^{*} = \eta^{-1}\text{poly}(\epsilon^{-1}, \|\bmu\|_{2}^{-1}, d^{-1}\sigma_{p}^{-2},\sigma_{0}^{-1}, n, m, d)$ is the maximum admissible iterations. Note that we can consider any polynomial training time $T^*$.
Denote $\alpha = 4\log(T^{*})$. Here we list the exact conditions for $\eta, \sigma_{0}, d$ required by the proofs in this section, which are part of Condition~\ref{condition:d_sigma0_eta}:
\begin{align}
&\eta = O\Big(\min\{nm/(q\sigma_{p}^{2}d), nm/(q2^{q+2}\alpha^{q-2} \sigma_{p}^{2}d),  nm/(q2^{q+2}\alpha^{q-2}\|\bmu\|_{2}^{2})\}\Big), \label{eq: verify}\\
&\sigma_{0} \leq [16\sqrt{ \log(8mn/\delta)}]^{-1}\min\{\|\bmu\|_{2}^{-1}, (\sigma_{p}\sqrt{d})^{-1}\}, \label{eq: verifyy}\\
&d \geq 1024\log(4n^{2}/\delta)\alpha^{2}n^{2}. \label{eq: verifyyy}
\end{align}
Denote $\beta = 2 \max_{i,j,r}\{|\la \wb_{j,r}^{(0)}, \bmu\ra|,|\la \wb_{j,r}^{(0)}, \bxi_{i}\ra|\}$. By Lemma~\ref{lemma:initialization_norms}, with probability at least $1- \delta$, we can upper bound $\beta$ by $4\sqrt{ \log(8mn/\delta)} \cdot \sigma_0 \cdot \max\{\|\bmu\|_{2}, \sigma_p \sqrt{d} \}$. Then, by \eqref{eq: verifyy} and \eqref{eq: verifyyy}, it is straightforward to verify the following inequality: 
\begin{align}
4\max\bigg\{\beta, 8n\sqrt{\frac{\log(4n^{2}/\delta)}{d}}\alpha\bigg\} \leq 1. \label{eq:verify0}
\end{align}
Suppose the conditions listed in \eqref{eq: verify}, \eqref{eq: verifyy} and \eqref{eq: verifyyy} hold, we claim that for $0 \leq t \leq T^*$ the following property holds.

\begin{proposition}[Restatement of Proposition~\ref{Prop:main1}]\label{Prop:noise} 
Under Condition~\ref{condition:d_sigma0_eta}, for $0 \leq t\leq T^*$, we have that
\begin{align}
 &0\leq \gamma_{j,r}^{(t)}, \zeta_{j,r,i}^{(t)} \leq \alpha,\label{eq:gu0001}\\
&0\geq \omega_{j,r,i}^{(t)} \geq -\beta - 16n\sqrt{\frac{\log(4n^{2}/\delta)}{d}}\alpha \geq -\alpha\label{eq:gu0002}.
\end{align}
for all $r\in [m]$,  $j\in \{\pm 1\}$ and $i\in [n]$.
\end{proposition}
We will use induction to prove Proposition~\ref{Prop:noise}. 
We first introduce several technical lemmas that will be used for the proof of Proposition~\ref{Prop:noise}.
\begin{lemma}\label{lm: mubound}
For any $t\geq 0$, it holds that $\la \wb_{j,r}^{(t)} - \wb_{j,r}^{(0)}, \bmu \ra =  j\cdot\gamma_{j,r}^{(t)}$ for all $r\in [m]$,  $j\in \{\pm 1\}$.
\end{lemma}
\begin{proof}[Proof of Lemma~\ref{lm: mubound}]
For any time $t\geq 0$, we have that
\begin{align*}
 \la \wb_{j,r}^{(t)} - \wb_{j,r}^{(0)}, \bmu \ra &= j\cdot\gamma_{j,r}^{(t)}  + \sum_{ i'= 1 }^n \zeta_{j,r,i'}^{(t)}  \| \bxi_{i'} \|_2^{-2} \cdot \la \bxi_{i'}, \bmu \ra + \sum_{ i' = 1}^n \omega_{j,r,i'}^{(t)}  \| \bxi_{i'} \|_2^{-2} \cdot \la \bxi_{i'}, \bmu \ra \\
 &= j\cdot\gamma_{j,r}^{(t)},
\end{align*}
where the equation is by our orthogonal assumption.  
\end{proof}
\begin{lemma}\label{lm:oppositebound}
Under Condition~\ref{condition:d_sigma0_eta},  suppose \eqref{eq:gu0001} and \eqref{eq:gu0002} hold at iteration $t$. Then
\begin{align*}
\omega_{j,r,i}^{(t)} - 8n\sqrt{\frac{\log(4n^{2}/\delta)}{d}}\alpha\leq \la \wb_{j,r}^{(t)} -  \wb_{j,r}^{(0)},\bxi_{i} \ra &\leq \omega_{j,r,i}^{(t)} + 8n\sqrt{\frac{\log(4n^{2}/\delta)}{d}}\alpha, ~j\not= y_{i},  \\  
\zeta_{j,r,i}^{(t)} - 8n\sqrt{\frac{\log(4n^{2}/\delta)}{d}}\alpha \leq \la\wb_{j,r}^{(t)} -  \wb_{j,r}^{(0)} ,\bxi_{i} \ra &\leq \zeta_{j,r,i}^{(t)} +  8n\sqrt{\frac{\log(4n^{2}/\delta)}{d}}\alpha, ~j = y_{i} 
\end{align*}
for all $r\in [m]$,  $j\in \{\pm 1\}$ and $i\in [n]$.
\end{lemma}
\begin{proof}[Proof of Lemma~\ref{lm:oppositebound}]
For $j \not= y_{i}$, we have that $\zeta_{j,r,i}^{(t)} = 0$ and
\begin{align*}
\la \wb_{j,r}^{(t)} - \wb_{j,r}^{(0)} ,\bxi_{i} \ra &= \sum_{ i'= 1 }^n \zeta_{j,r,i'}^{(t)}  \| \bxi_{i'} \|_2^{-2} \cdot \la \bxi_{i'}, \bxi_i \ra + \sum_{ i' = 1}^n \omega_{j,r,i'}^{(t)}  \| \bxi_{i'} \|_2^{-2} \cdot \la \bxi_{i'}, \bxi_i \ra \\
&\leq  4\sqrt{\frac{\log(4n^{2}/\delta)}{d}}\sum_{i'\not= i}|\zeta_{j,r,i'}^{(t)}| +4\sqrt{\frac{\log(4n^{2}/\delta)}{d}}\sum_{i' \not = i}|\omega_{j,r,i'}^{(t)}| + \omega_{j,r,i}^{(t)} \\
&\leq \omega_{j,r,i}^{(t)} + 8n\sqrt{\frac{\log(4n^{2}/\delta)}{d}}\alpha,
\end{align*}
where the second inequality is by Lemma~\ref{lemma:data_innerproducts} and the last inequality is by $|\zeta^{(t)}_{j,r,i'}|, |\omega_{j,r,i'}^{(t)}| \leq \alpha$ in \eqref{eq:gu0001} Similarly, for $y_{i} = j$, we have that $\omega_{j,r,i}^{(t)} = 0$ and 
\begin{align*}
\la \wb_{j,r}^{(t)} - \wb_{j,r}^{(0)} ,\bxi_{i} \ra &= \sum_{ i'= 1 }^n \zeta_{j,r,i'}^{(t)}  \| \bxi_{i'} \|_2^{-2} \cdot \la \bxi_{i'}, \bxi_i \ra + \sum_{ i' = 1}^n \omega_{j,r,i'}^{(t)}  \| \bxi_{i'} \|_2^{-2} \cdot \la \bxi_{i'}, \bxi_i \ra \\
&\leq \zeta_{j,r,i}^{(t)} +  4\sqrt{\frac{\log(4n^{2}/\delta)}{d}}\sum_{ i'\not= i } |\zeta_{j,r,i'}^{(t)}| + 4\sqrt{\frac{\log(4n^{2}/\delta)}{d}}\sum_{ i' \not= i} |\omega_{j,r,i'}^{(t)}|\\
&\leq \zeta^{(t)}_{j,r,i} + 8n\sqrt{\frac{\log(4n^{2}/\delta)}{d}}\alpha,
\end{align*}
where the first inequality is by Lemma~\ref{lemma:numberofdata} and the second inequality is by $|\zeta^{(t)}_{j,r,i'}|, |\omega_{j,r,i'}^{(t)}| \leq \alpha$ in \eqref{eq:gu0001}. Similarly, we can show that $\la \wb_{j,r}^{(t)} -  \wb_{j,r}^{(0)},\bxi_{i} \ra \geq \omega_{j,r,i}^{(t)} - 8n\sqrt{\log(4n^{2}/\delta)/d}\cdot\alpha$ and $\la\wb_{j,r}^{(t)} -  \wb_{j,r}^{(0)} ,\bxi_{i} \ra \geq \zeta_{j,r,i}^{(t)} - 8n\sqrt{\log(4n^{2}/\delta)/d}\cdot\alpha$, which completes the proof.
\end{proof}
\begin{lemma}\label{lm: F-yi} Under Condition~\ref{condition:d_sigma0_eta}, suppose \eqref{eq:gu0001} and \eqref{eq:gu0002} hold at iteration $t$. Then 
\begin{align*}
\la \wb_{j,r}^{(t)}, y_{i}\bmu \ra &\leq \la \wb_{j,r}^{(0)}, y_{i}\bmu\ra, \\
\la \wb_{j,r}^{(t)}, \bxi_{i} \ra &\leq \la \wb_{j,r}^{(0)}, \bxi_{i}\ra + 8n\sqrt{\frac{\log(4n^{2}/\delta)}{d}}\alpha,\\
F_{j}(\Wb_{j}^{(t)}, \xb_{i})&\leq 1
\end{align*}
for all $r\in [m]$ and $j \not= y_{i}$.
\end{lemma}
\begin{proof}[Proof of Lemma~\ref{lm: F-yi}]
For $j \not= y_{i}$, we have that 
\begin{align}
\la \wb_{j,r}^{(t)}, y_{i}\bmu \ra = \la \wb_{j,r}^{(0)}, y_{i}\bmu\ra + y_{i}\cdot j \cdot \gamma_{j,r}^{(t)} \leq \la \wb_{j,r}^{(0)}, y_{i}\bmu\ra, \label{eq:F-yi1}
\end{align}
where the inequality is by $\gamma_{j,r}^{(t)} \geq 0$. 
In addition, we have
\begin{align}
\la \wb_{j,r}^{(t)}, \bxi_{i} \ra \leq \la \wb_{j,r}^{(0)}, \bxi_{i}\ra + \omega_{j,r,i}^{(t)} + 8n\sqrt{\frac{\log(4n^{2}/\delta)}{d}}\alpha \leq \la \wb_{j,r}^{(0)}, \bxi_{i}\ra + 8n\sqrt{\frac{\log(4n^{2}/\delta)}{d}}\alpha, \label{eq:F-yi2}
\end{align}
where the first inequality is by Lemma~\ref{lm:oppositebound} and the second inequality is due to $\omega_{j,r,i}^{(t)} \leq 0$.
Then we can get that 
\begin{align*}
F_{j}(\Wb_{j}^{(t)}, \xb_{i}) &= \frac{1}{m}\sum_{r=1}^{m}[\sigma(\la \wb_{j,r}^{(t)}, -j\cdot \bmu\ra) + \sigma(\la \wb_{j,r}^{(t)} ,\bxi_{i} \ra)]\\
&\leq 2^{q+1} \max_{j,r,i} \bigg\{|\la \wb_{j,r}^{(0)}, \bmu \ra|, |\la \wb_{j,r}^{(0)}, \bxi_{i}\ra|, 8n\sqrt{\frac{\log(4n^{2}/\delta)}{d}}\alpha\bigg\}^{q}\\
&\leq 1,
\end{align*}
where the first inequality is by \eqref{eq:F-yi1}, \eqref{eq:F-yi2} and the second inequality is by \eqref{eq:verify0}.
\end{proof}
\begin{lemma}\label{lm: Fyi}
Under Condition~\ref{condition:d_sigma0_eta}, suppose \eqref{eq:gu0001} and \eqref{eq:gu0002} hold at iteration $t$. Then
\begin{align*}
\la \wb_{j,r}^{(t)}, y_{i}\bmu \ra &= \la \wb_{j,r}^{(0)}, y_{i}\bmu\ra + \gamma_{j,r}^{(t)}, \\
\la \wb_{j,r}^{(t)}, \bxi_{i} \ra &\leq \la \wb_{j,r}^{(0)}, \bxi_{i}\ra + \zeta_{j,r,i}^{(t)} +  8n\sqrt{\frac{\log(4n^{2}/\delta)}{d}}\alpha
\end{align*}
for all $r\in [m]$, $j\in \{\pm 1\}$ and $i\in [n]$.
If $\max\{\gamma_{j,r}^{(t)}, \zeta_{j,r,i}^{(t)}\} = O(1)$, we further have that $F_{j}(\Wb_{j}^{(t)}, \xb_{i}) = O(1)$.
\end{lemma}
\begin{proof}[Proof of Lemma~\ref{lm: Fyi}]
For $j = y_{i}$, we have that 
\begin{align}
\la \wb_{j,r}^{(t)}, y_{i}\bmu \ra = \la \wb_{j,r}^{(0)}, y_{i}\bmu\ra +  \gamma_{j,r}^{(t)}, \label{eq:Fyi1}
\end{align}
where the equation is by Lemma~\ref{lm: mubound}. We also have that
\begin{align}
\la \wb_{j,r}^{(t)}, \bxi_{i} \ra \leq \la \wb_{j,r}^{(0)}, \bxi_{i}\ra + \zeta_{j,r,i}^{(t)} + 8n\sqrt{\frac{\log(4n^{2}/\delta)}{d}}\alpha, \label{eq:Fyi2}
\end{align}
where the inequality is by Lemma~\ref{lm:oppositebound}. 
If $\max\{\gamma_{j,r}^{(t)}, \zeta_{j,r,i}^{(t)}\} = O(1)$, we have following bound 
\begin{align*}
F_{j}(\Wb_{j}^{(t)}, \xb_{i}) &= \frac{1}{m}\sum_{r=1}^{m}[\sigma(\la \wb_{j,r}^{(t)}, -j\cdot \bmu) + \sigma(\la \wb_{j,r}^{(t)} ,\bxi_{i} \ra)]\\
&\leq 2\cdot 3^{q} \max_{j,r,i} \bigg\{\gamma_{j,r}^{(t)}, \zeta_{j,r,i}^{(t)}, |\la \wb_{j,r}^{(0)}, \bmu \ra|, |\la \wb_{j,r}^{(0)}, \bxi_{i}\ra|, 8n\sqrt{\frac{\log(4n^{2}/\delta)}{d}}\alpha\bigg\}^{q}\\
&= O(1),
\end{align*}
where the first inequality is by \eqref{eq:Fyi1}, \eqref{eq:Fyi2} and the second inequality is by \eqref{eq:verify0} where $\beta =  2\max_{i,j,r}\{|\la \wb_{j,r}^{(0)}, \bmu\ra|,|\la \wb_{j,r}^{(0)}, \bxi_{i}\ra|\}$.
\end{proof}

Now we are ready to prove Proposition~\ref{Prop:noise}.
\begin{proof}[Proof of Proposition~\ref{Prop:noise}] Our proof is based on induction. The results are obvious at $t = 0$ as all the coefficients are zero. 
Suppose that there exists $\tilde{T} \leq T^{*}$ such that the results in Proposition~\ref{Prop:noise} hold for all time $0 \leq t \leq \tilde{T}-1$. We aim to prove that they also hold for $t = \tilde{T}$.

We first prove that \eqref{eq:gu0002} holds for $t = \tilde{T}$, i.e., $\omega^{(t)}_{j,r,i} \geq -\beta - 16n\sqrt{\frac{\log(4n^{2}/\delta)}{d}}\alpha$ for $t = \tilde{T}$, $r\in [m]$,  $j\in \{\pm 1\}$ and $i\in [n]$.
Notice that $\omega_{j,r,i}^{(t)} = 0, \forall j = y_{i}$. Therefore, we only need to consider the case that $j \not= y_{i}$.
When $\omega_{j,r,i}^{(\tilde{T}-1)} \leq -0.5\beta - 8n\sqrt{\frac{\log(4n^{2}/\delta)}{d}}\alpha$, by Lemma~\ref{lm:oppositebound} we have that 
\begin{align*}
\la \wb_{j,r}^{(\tilde{T}-1)} ,\bxi_{i} \ra \leq  \omega_{j,r,i}^{(\tilde{T}-1)}  + \la \wb_{j,r}^{(0)}, \bxi_{i}\ra + 8n\sqrt{\frac{\log(4n^{2}/\delta)}{d}}\alpha \leq 0,
\end{align*}
and thus
\begin{align*}
\omega_{j,r,i}^{(\tilde{T})} &= \omega_{j,r,i}^{(\tilde{T}-1)} + \frac{\eta}{nm} \cdot \ell_i'^{(\tilde{T}-1)}\cdot \sigma'(\la\wb_{j,r}^{(\tilde{T}-1)}, \bxi_{i}\ra) \cdot \ind(y_{i} = -j)\|\bxi_{i}\|_{2}^{2}\\
&=  \omega_{j,r,i}^{(\tilde{T}-1)}\\
&\geq -\beta - 16n\sqrt{\frac{\log(4n^{2}/\delta)}{d}}\alpha,
\end{align*}
where the last inequality is by induction hypothesis. When $\omega_{j,r,i}^{(\tilde{T}-1)} \geq  -0.5\beta - 8n\sqrt{\frac{\log(4n^{2}/\delta)}{d}}\alpha$, we have that 
\begin{align*}
\omega_{j,r,i}^{(\tilde{T})} &= \omega_{j,r,i}^{(\tilde{T}-1)} + \frac{\eta}{nm} \cdot \ell_i'^{(\tilde{T}-1)}\cdot \sigma'(\la\wb_{j,r}^{(T-1)}, \bxi_{i}\ra) \cdot \ind(y_{i} = -j)\|\bxi_{i}\|_{2}^{2}\\
&\geq -0.5\beta - 8n\sqrt{\frac{\log(4n^{2}/\delta)}{d}}\alpha - O\bigg(\frac{\eta\sigma_{p}^{2}d}{nm}\bigg)\sigma'\bigg(0.5\beta + 8n\sqrt{\frac{\log(4n^{2}/\delta)}{d}}\alpha\bigg)\\
&\geq -0.5\beta - 8n\sqrt{\frac{\log(4n^{2}/\delta)}{d}}\alpha - O\bigg(\frac{\eta q\sigma_{p}^{2}d}{nm}\bigg)\bigg(0.5\beta + 8n\sqrt{\frac{\log(4n^{2}/\delta)}{d}}\alpha\bigg)\\
&\geq -\beta - 16n\sqrt{\frac{\log(4n^{2}/\delta)}{d}}\alpha, 
\end{align*}
where we use $\ell_i'^{(\tilde{T}-1)}\leq 1$ and $\|\bxi_{i}\|_{2} = O(\sigma_{p}^{2}d)$ in the first inequality, the second inequality is by $0.5\beta + 8n\sqrt{\frac{\log(4n^{2}/\delta)}{d}}\alpha \leq 1$, and the last inequality is by  $\eta = O\big(nm/(q\sigma_{p}^{2}d)\big)$ in  \eqref{eq: verify}.

Next we prove \eqref{eq:gu0001} holds for $t = \tilde{T}$. We have
\begin{align}
    |\ell_i'^{(t)}| &= \frac{1}{1 + \exp\{ y_i \cdot [F_{+1}(\Wb_{+1}^{(t)},\xb_i) - F_{-1}(\Wb_{-1}^{(t)},\xb_i)] \} }\notag\\
    & \leq \exp\{ -y_{i} \cdot [F_{+1}(\Wb_{+1}^{(t)},\xb_i) - F_{-1}(\Wb_{-1}^{(t)},\xb_i)]\}\notag\\
    & \leq \exp\{ -  F_{y_{i}}(\Wb_{y_{i}}^{(t)},\xb_i) + 1 \}. \label{eq:logit}
\end{align}
where the last inequality is due to Lemma~\ref{lm: F-yi}. Moreover, recall the update rule of $\gamma_{j,r}^{(t)}$ and $\zeta_{j,r,i}^{(t)}$,
\begin{align*}
   \gamma_{j,r}^{(t+1)} &= \gamma_{j,r}^{(t)} - \frac{\eta}{nm} \cdot \sum_{i=1}^n \ell_i'^{(t)} \cdot \sigma'(\la\wb_{j,r}^{(t)}, y_{i} \cdot \bmu\ra)\|\bmu\|_{2}^{2},\\
    \zeta_{j,r,i}^{(t+1)} &= \zeta_{j,r,i}^{(t)} - \frac{\eta}{nm} \cdot \ell_i'^{(t)}\cdot \sigma'(\la\wb_{j,r}^{(t)}, \bxi_{i}\ra) \cdot \ind(y_{i} = j)\|\bxi_{i}\|_{2}^{2}.
\end{align*}
Let $t_{j,r,i}$ to be the last time $t < T^{*}$ that $\zeta_{j,r,i}^{(t)} \leq 0.5 \alpha$.
Then we have that 
\begin{align}
\zeta_{j,r,i}^{(\tilde{T})} &= \zeta_{j,r,i}^{(t_{j,r,i})} - \underbrace{\frac{\eta}{nm} \cdot \ell_i'^{(t_{j,r,i})}\cdot \sigma'(\la\wb_{j,r}^{(t_{j,r,i})}, \bxi_{i}\ra) \cdot \ind(y_{i} = j)\|\bxi_{i}\|_{2}^{2}}_{I_{1}}\notag\\
&\qquad - \underbrace{\sum_{t_{j,r,i}<t<T}\frac{\eta}{nm} \cdot \ell_i'^{(t)}\cdot \sigma'(\la\wb_{j,r}^{(t)}, \bxi_{i}\ra) \cdot \ind(y_{i} = j)\|\bxi_{i}\|_{2}^{2}}_{I_{2}}.\label{eq:zeta}
\end{align}
We first bound $I_{1}$ as follows,
\begin{align*}
|I_{1}| \leq 2qn^{-1}m^{-1}\eta\bigg(\zeta_{j,r,i}^{(t_{j,r,i})} + 0.5\beta + 8n\sqrt{\frac{\log(4n^{2}/\delta)}{d}}\alpha\bigg)^{q-1}\sigma_{p}^{2}d \leq  q2^{q}n^{-1}m^{-1}\eta \alpha^{q-1} \sigma_{p}^{2}d \leq 0.25\alpha,
\end{align*}
where the first inequality is by Lemmas~\ref{lm:oppositebound} and~\ref{lemma:data_innerproducts}, the second inequality is by $\beta \leq 0.1\alpha$ and $8n\sqrt{\frac{\log(4n^{2}/\delta)}{d}}\alpha \leq 0.1\alpha$, the last inequality is by $\eta \leq nm/(q2^{q+2}\alpha^{q-2} \sigma_{p}^{2}d)$. 

Second, we bound $I_{2}$. For $t_{j,r,i}<t<\tilde{T}$ and $y_{i} = j$, we can lower bound $\la\wb_{j,r}^{(t)}, \bxi_{i}\ra$ as follows, 
 \begin{align*}
\la\wb_{j,r}^{(t)}, \bxi_{i}\ra &\geq \la\wb_{j,r}^{(0)},  \bxi_{i}\ra + \zeta_{j,r,i}^{(t)} - 8n\sqrt{\frac{\log(4n^{2}/\delta)}{d}}\alpha \\
&\geq - 0.5\beta + 0.5\alpha - 8n\sqrt{\frac{\log(4n^{2}/\delta)}{d}}\alpha\\
&\geq 0.25\alpha, 
\end{align*}
where the first inequality is by Lemma~\ref{lm:oppositebound}, the second inequality is by $\zeta_{j,r,i}^{(t)} > 0.5 \alpha$ and $\la\wb_{j,r}^{(0)},  \bxi_{i}\ra \geq - 0.5\beta$ due to the definition of $t_{j,r,i}$ and $\beta$, the last inequality is by $\beta \leq 0.1\alpha$ and $8n\sqrt{\frac{\log(4n^{2}/\delta)}{d}}\alpha \leq 0.1\alpha$. Similarly, for $t_{j,r,i}<t<\tilde{T}$ and $y_{i} = j$, we can also upper bound $\la\wb_{j,r}^{(t)}, \bxi_{i}\ra$ as follows, 
 \begin{align*}
\la\wb_{j,r}^{(t)}, \bxi_{i}\ra &\leq \la\wb_{j,r}^{(0)},  \bxi_{i}\ra + \zeta_{j,r,i}^{(t)} + 8n\sqrt{\frac{\log(4n^{2}/\delta)}{d}}\alpha \\
&\leq 0.5\beta + \alpha + 8n\sqrt{\frac{\log(4n^{2}/\delta)}{d}}\alpha\\
&\leq 2\alpha, 
\end{align*}
where the first inequality is by Lemma~\ref{lm:oppositebound}, the second inequality is by induction hypothesis $\zeta_{j,r,i}^{(t)} \leq \alpha$, the last inequality is by $\beta \leq 0.1\alpha$ and $8n\sqrt{\frac{\log(4n^{2}/\delta)}{d}}\alpha \leq 0.1\alpha$. Thus, plugging the upper and lower bounds of $\la\wb_{j,r}^{(t)}, \bxi_{i}\ra$ into $I_{2}$ gives
\begin{align*}
|I_{2}| &\leq \sum_{t_{j,r,i}<t<\tilde{T}}\frac{\eta}{nm} \cdot \exp(- \sigma(\la\wb_{j,r}^{(t)}, \bxi_{i}\ra) + 1)\cdot \sigma'(\la\wb_{j,r}^{(t)}, \bxi_{i}\ra) \cdot \ind(y_{i} = j)\|\bxi_{i}\|_{2}^{2}\\
&\leq \frac{eq2^{q}\eta T^{*}}{n}\exp(-\alpha^{q}/4^{q})\alpha^{q-1}\sigma_{p}^{2}d\\
&\leq 0.25 T^{*}\exp(-\alpha^{q}/4^{q})\alpha \\
&\leq 0.25 T^{*}\exp(-\log(T^{*})^{q})\alpha \\
&\leq 0.25\alpha,
\end{align*}
where the first inequality is by 
\eqref{eq:logit}, the second inequality is by Lemma~\ref{lemma:data_innerproducts}, the third inequality is by  $\eta = O\big( nm/(q2^{q+2}\alpha^{q-2} \sigma_{p}^{2}d)\big)$ in \eqref{eq: verify}, the fourth inequality is by our choice of $\alpha = 4\log(T^{*})$ and the last inequality is due to the fact that $\log(T^{*})^{q} \geq \log(T^{*})$. Plugging the bound of $I_{1}, I_{2}$ into \eqref{eq:zeta} completes the proof for $\zeta$. Similarly, we can prove that $\gamma_{j,r}^{(\tilde{T})} \leq \alpha$ using $\eta = O\big( nm/(q2^{q+2}\alpha^{q-2}\|\bmu\|_{2}^{2})\big)$ in \eqref{eq: verify}. 
Therefore Proposition~\ref{Prop:noise} holds for $t= \tilde{T}$, which completes the induction. 
\end{proof}
Based on Proposition~\ref{Prop:noise}, we introduce some important properties of the training loss function for $0 \leq t\leq T^{*}$.
\begin{lemma}[Restatement of Lemma~\ref{lm: gradient upbound sketch}]\label{lm: gradient upbound}
Under Condition~\ref{condition:d_sigma0_eta}, for $0 \leq t\leq T^{*}$, the following result holds.
\begin{align*}
\|\nabla L_{S}(\Wb^{(t)})\|_{F}^{2} \leq  O(\max\{\|\bmu\|_{2}^{2}, \sigma_{p}^{2}d\}) L_{S}(\Wb^{(t)}).   
\end{align*}
\end{lemma}
\begin{proof}[Proof of Lemma~\ref{lm: gradient upbound}]
We first prove that 
\begin{align}
-\ell'\big(y_{i}f(\Wb^{(t)}, \xb_{i})) \cdot \|\nabla f(\Wb^{(t)}, \xb_{i}\big)\|_{F}^2 = O(\max\{\|\bmu\|_{2}^{2}, \sigma_{p}^{2}d\}).  \label{eq: exptail2}  
\end{align}
Without loss of generality,  
we suppose that $y_{i} = 1$ and $\xb_{i} = [\bmu^{\top},\bxi_{i}]$. Then we have that 
\begin{align*}
\|\nabla f(\Wb^{(t)}, \xb_{i}) \|_{F}&\leq \frac{1}{m}\sum_{j,r}\bigg\|  \big[\sigma'(\la\wb_{j,r}^{(t)},\bmu\ra)\bmu + \sigma'(\la\wb_{j,r}^{(t)}, \bxi_i\ra)\bxi_i\big] \bigg\|_{2}\\
&\leq \frac{1}{m}\sum_{j,r}\sigma'(\la\wb_{j,r}^{(t)}, \bmu\ra)\|\bmu\|_{2} +   \frac{1}{m}\sum_{j,r}\sigma'(\la\wb_{j,r}^{(t)}, \bxi_{i}\ra)\|\bxi_{i}\|_{2}\\
&\leq 2q\bigg[F_{+1}(\Wb^{(t)}_{+1}, \xb_{i})\bigg]^{(q-1)/q}\max\{\|\bmu\|_{2}, 2\sigma_{p}\sqrt{d}\} \\
&\qquad + 2q\bigg[F_{-1}(\Wb^{(t)}_{-1}, \xb_{i})\bigg]^{(q-1)/q}\max\{\|\bmu\|_{2}, 2\sigma_{p}\sqrt{d}\}\\
&\leq 2q\bigg[F_{+1}(\Wb^{(t)}_{+1}, \xb_{i})\bigg]^{(q-1)/q}\max\{\|\bmu\|_{2}, 2\sigma_{p}\sqrt{d}\} + 2q\max\{\|\bmu\|_{2}, 2\sigma_{p}\sqrt{d}\} ,
\end{align*}
where the first and second inequalities are by triangle inequality, the third inequality is by Jensen's inequality and Lemma~\ref{lemma:data_innerproducts}, and the last inequality is due to Lemma~\ref{lm: F-yi}.
Denote $A = F_{+1}(\Wb_{+1}^{(t)}, \xb_{i})$. Then we have that $A \geq 0$, and besides, $F_{-1}(\Wb^{(t)}_{-1}, \xb_{i}) \leq 1$ by Lemma~\ref{lm: F-yi}. Then we have that
\begin{align*}
&-\ell'\big(y_{i}f(\Wb^{(t)},\xb_{i})\big) \cdot \|\nabla f(\Wb^{(t)}, \xb_{i})\|_{F}^2\\
&\qquad\leq -\ell'(A - 1)\bigg(2q\cdot A^{(q-1)/q}\max\{\|\bmu\|_{2}, 2\sigma_{p}\sqrt{d}\} + 2q\cdot \max\{\|\bmu\|_{2}, 2\sigma_{p}\sqrt{d}\}\bigg)^{2}\\
&\qquad = -4q^{2}\ell'(A-1)(A^{(q-1)/q} + 1)^{2}\cdot \max\{\|\bmu\|_{2}^{2},4 \sigma_{p}^{2}d\}\\
&\qquad\leq \big(\max_{z>0}-4q^{2}\ell'(z-1)(z^{(q-1)/q} + 1)^{2}\big)\max\{\|\bmu\|_{2}^{2},4 \sigma_{p}^{2}d\}\\
&\qquad \overset{(i)}{=} O(\max\{\|\bmu\|_{2}^{2}, \sigma_{p}^{2}d\}),
\end{align*}
where (i) is by $\max_{z\geq 0}-4q^{2}\ell'(z-1)(z^{(q-1)/q} + 1)^{2}< \infty$ because $\ell'$ has an exponentially decaying tail. 
Now we can upper bound the gradient norm $\|\nabla L_{S}(\Wb^{(t)})\|_{F}$ as follows,
\begin{align*}
\|\nabla L_{S}(\Wb^{(t)})\|_{F}^{2} &\leq \bigg[\frac{1}{n}\sum_{i=1}^{n}\ell'\big(y_{i}f(\Wb^{(t)},\xb_{i})\big)\|\nabla f(\Wb^{(t)}, \xb_{i})\|_{F}\bigg]^{2}\\
&\leq \bigg[\frac{1}{n}\sum_{i=1}^{n}\sqrt{-O(\max\{\|\bmu\|_{2}^{2}, \sigma_{p}^{2}d\})\ell'\big(y_{i}f(\Wb^{(t)},\xb_{i})\big)}\bigg]^{2}\\
&\leq O(\max\{\|\bmu\|_{2}^{2}, \sigma_{p}^{2}d\}) \cdot \frac{1}{n}\sum_{i=1}^{n}-\ell'\big(y_{i}f(\Wb^{(t)},\xb_{i})\big)\\
&\leq O(\max\{\|\bmu\|_{2}^{2}, \sigma_{p}^{2}d\}) L_{S}(\Wb^{(t)}),
\end{align*}
where the first inequality is by triangle inequality, the second inequality is by \eqref{eq: exptail2}, the third inequality is by Cauchy-Schwartz inequality and the last inequality is due to the property of the cross entropy loss $-\ell' \leq \ell$.
\end{proof}

\section{Signal Learning}
In this section, we consider the signal learning case under the condition that $n\|\bmu\|_{2}^{q} \geq \tilde{\Omega}(\sigma_{p}^{q}(\sqrt{d})^{q})$. We remind the readers that the proofs in this section are based on the results in Section~\ref{sec: initial}, which hold with high probability.

\subsection{First stage}




\begin{lemma}[Restatement of Lemma~\ref{lemma:phase1_main_sketch}]\label{lemma:phase1_main}
Under the same conditions as Theorem~\ref{thm:signal_learning_main}, in particular if we choose
\begin{align}
n \cdot \mathrm{SNR}^{q} \geq C\log(6/\sigma_{0}\|\bmu\|_{2})2^{2q+6}[4\log(8mn/\delta)]^{(q-1)/2},\label{eq:explicit condition}
\end{align}
where $C = O(1)$ is a positive constant, there exists time 
\begin{align*}
T_1 = \frac{C\log(6/\sigma_{0}\|\bmu\|_{2})2^{q+1}m}{\eta\sigma_{0}^{q-2}\|\bmu\|_{2}^{q}}
\end{align*}
such that 
\begin{itemize}
\item $\max_{ r}\gamma_{j, r}^{(T_{1})} \geq 2$ for $j\in \{\pm 1\}$.
\item $|\rho_{j,r,i}^{(t)}| \leq \sigma_0 \sigma_p \sqrt{d} / 2$ for all $j\in \{\pm 1\}, r\in[m]$, $i \in [n]$ and $0 \leq t \leq T_{1}$. 
\end{itemize}
\end{lemma}
\begin{proof}[Proof of Lemma~\ref{lemma:phase1_main}]
Let 
\begin{align}
    T_1^{+} = \frac{nm\eta^{-1}\sigma_{0}^{2-q}\sigma_{p}^{-q}d^{-q/2}}{2^{q+4}q[4\log(8mn/\delta)]^{(q-1)/2}}. \label{eq:T1upper}
\end{align}
We first prove the second bullet. 
Define $\Psi^{(t)} = \max_{j,r,i} |\rho_{j,r,i}^{(t)}|=  \max_{j,r,i}\{ \zeta_{j,r,i}^{(t)},  -\omega_{j,r,i}^{(t)}\} $. We use induction to show that
\begin{align}
    \Psi^{(t)} \leq \sigma_0 \sigma_p \sqrt{d} / 2 \label{eq:Psi_induction}
\end{align}
for all $0 \leq t \leq T_{1}^{+}$. By definition, clearly we have $\Psi^{(0)} = 0$. Now suppose that there exists some $\tilde{T} \leq T_1^+$ such that \eqref{eq:Psi_induction} holds for $0 < t \leq \tilde{T}-1$. Then by \eqref{eq:update_zeta2} and \eqref{eq:update_omega2} we have
\begin{align*}
    \Psi^{(t+1)} &\leq \Psi^{(t)} + \max_{j,r,i}\bigg\{\frac{\eta}{nm} \cdot |\ell_i'^{(t)}|\cdot  \sigma'\Bigg(\la\wb_{j,r}^{(0)}, \bxi_{i}\ra + \sum_{ i'= 1 }^n \Psi^{(t)} \cdot \frac{ |\la \bxi_{i'}, \bxi_i \ra|}{ \| \bxi_{i'} \|_2^2} + \sum_{ i' = 1}^n \Psi^{(t)} \cdot \frac{|\la \bxi_{i'}, \bxi_i \ra|}{ \| \bxi_{i'} \|_2^2} \Bigg)\cdot \| \bxi_{i} \|_2^2 \bigg\} \\
    &\leq \Psi^{(t)} + \max_{j,r,i}\bigg\{\frac{\eta}{nm} \cdot  \sigma'\Bigg(\la\wb_{j,r}^{(0)}, \bxi_{i}\ra + 2\cdot \sum_{ i'= 1 }^n \Psi^{(t)} \cdot \frac{ |\la \bxi_{i'}, \bxi_i \ra| }{ \| \bxi_{i'} \|_2^2}  \Bigg)\cdot \| \bxi_{i} \|_2^2\bigg\} \\
    &= \Psi^{(t)} + \max_{j,r,i}\bigg\{\frac{\eta}{nm} \cdot  \sigma'\Bigg(\la\wb_{j,r}^{(0)}, \bxi_{i}\ra + 2\Psi^{(t)}  + 2\cdot \sum_{ i' \neq i }^n \Psi^{(t)} \cdot \frac{ |\la \bxi_{i'}, \bxi_i \ra| }{ \| \bxi_{i'} \|_2^2} \Bigg)\cdot \| \bxi_{i} \|_2^2\bigg\} \\
    &\leq \Psi^{(t)} + \frac{\eta q}{nm} \cdot  \Bigg[2\cdot \sqrt{ \log(8mn/\delta)} \cdot \sigma_0 \sigma_p \sqrt{d} + \Bigg( 2 + \frac{4n \sigma_p^2 \cdot \sqrt{d \log(4n^2 / \delta)} }{ \sigma_p^2 d /2} \Bigg) \cdot \Psi^{(t)}  \Bigg]^{q-1}\cdot 2 \sigma_p^2 d\\
    &\leq \Psi^{(t)} + \frac{\eta q}{nm} \cdot  \big(2\cdot \sqrt{ \log(8mn/\delta)} \cdot \sigma_0 \sigma_p \sqrt{d} + 4 \Psi^{(t)}  \big)^{q-1}\cdot 2 \sigma_p^2 d\\
    &\leq \Psi^{(t)} + \frac{\eta q}{nm} \cdot \big(4\cdot \sqrt{ \log(8mn/\delta)} \cdot \sigma_0 \sigma_p \sqrt{d} \big)^{q-1} \cdot 2 \sigma_p^2 d,
 \end{align*}
where the second inequality is by $|\ell_i'^{(t)}| \leq 1$, the third inequality is due to Lemmas~\ref{lemma:data_innerproducts} and \ref{lemma:initialization_norms}, the fourth inequality follows by the condition that $d \geq 16 n^2 \log (4n^2/\delta)$, and the last inequality follows by the induction hypothesis \eqref{eq:Psi_induction}. Taking a telescoping sum over $t=0,1,\ldots, \tilde{T}-1$ then gives
\begin{align*}
    \Psi^{(\tilde{T})} 
    &\leq \tilde{T} \cdot\frac{\eta q}{nm} \cdot \big(4\cdot \sqrt{ \log(8mn/\delta)} \cdot \sigma_0 \sigma_p \sqrt{d} \big)^{q-1} \cdot 2 \sigma_p^2 d\\
    &\leq T_{1}^+\cdot\frac{\eta q}{nm} \cdot \big(4\cdot \sqrt{ \log(8mn/\delta)} \cdot \sigma_0 \sigma_p \sqrt{d} \big)^{q-1} \cdot 2 \sigma_p^2 d\\
     &\leq \frac{\sigma_0 \sigma_p \sqrt{d} }{2},
\end{align*}
where the second inequality follows by $\tilde{T} \leq T_1^+$ in our induction hypothesis. Therefore, by induction, we have $ \Psi^{(t)} \leq \sigma_0 \sigma_p \sqrt{d} / 2$ for all $t \leq T_{1}^{+}$. 

Now, without loss of generality, let us consider $j = 1$ first. Denote by $T_{1,1}$ the last time for $t$ in the period $[0, T_1^{+}]$ satisfying that $\max_{r}\gamma_{1,r}^{(t)}\leq 2$. Then for $t \leq T_{1,1}$, $\max_{j,r,i}\{ |\rho_{j,r,i}^{(t)}|\} = O(\sigma_{0}\sigma_{p}\sqrt{d})= O(1)$ and $\max_{r}\gamma_{1,r}^{(t)} \leq 2$. Therefore, by Lemmas~\ref{lm: F-yi} and~\ref{lm: Fyi}, we know that $F_{-1}(\Wb_{-1}^{(t)},\xb_{i}), F_{+1}(\Wb_{+1}^{(t)},\xb_{i}) = O(1)$ for all $i$ with $y_{i} = 1$. Thus there exists a positive constant $C_{1}$ such that $-\ell'^{(t)}_{i} \geq C_{1}$ for all $i$ with $y_{i} = 1$.

By \eqref{eq:update_gamma2}, for $t\leq T_{1,1}$ we have
\begin{align*}
    \gamma_{1,r}^{(t+1)} &= \gamma_{1,r}^{(t)} - \frac{\eta}{nm} \cdot \sum_{i=1}^n \ell_i'^{(t)} \cdot \sigma'( y_{i} \cdot  \la\wb_{1,r}^{(0)}, \bmu \ra +  y_{i}  \cdot \gamma_{1,r}^{(t)} )\cdot \|\bmu\|_{2}^{2}\\
    &\geq \gamma_{1,r}^{(t)} + \frac{C_{1}\eta}{nm} \cdot \sum_{y_i=1}  \sigma'(  \la\wb_{1,r}^{(0)}, \bmu \ra +   \gamma_{1,r}^{(t)} )\cdot \|\bmu\|_{2}^{2}.
\end{align*}
Denote $\hat{\gamma}_{1,r}^{(t)} = \gamma_{1,r}^{(t)} + \la \wb_{1,r}^{(0)}, \bmu \ra $ and let $A^{(t)} = \max_{r}\hat{\gamma}_{1,r}^{(t)}$. Then we have 
\begin{align*}
   A^{(t+1)}&\geq A^{(t)} + \frac{C_{1}\eta}{nm} \cdot \sum_{y_i=1} \sigma'(A^{(t)})\cdot \|\bmu\|_{2}^{2}\\
    &\geq A^{(t)} +\frac{C_{1} \eta q\|\bmu\|_{2}^{2}}{4m}  \big[A^{(t)}\big]^{q-1}\\
    &\geq \bigg(1 + \frac{C_{1} \eta q\|\bmu\|_{2}^{2}}{4m}\big[A^{(0)}\big]^{q-2}\bigg) A^{(t)}\\
    &\geq \bigg(1 + \frac{C_{1}\eta q\sigma_{0}^{q-2}\|\bmu\|_{2}^{q}}{2^{q}m} \bigg) A^{(t)},
\end{align*}
where the second inequality is by the lower bound on the number of positive data  in Lemma~\ref{lemma:numberofdata}, the third inequality is due to the fact that $A^{(t)}$ is an increasing sequence, and the last inequality follows by $A^{(0)} = \max_{r} \la \wb_{1,r}^{(0)}, \bmu \ra \geq \sigma_0 \|\bmu\|_2/2$ proved in Lemma~\ref{lemma:initialization_norms}. Therefore, the sequence $A^{(t)}$ will exponentially grow and we have that 
\begin{align*}
   A^{(t)}&\geq\bigg(1 + \frac{C_{1}\eta q\sigma_{0}^{q-2}\|\bmu\|_{2}^{q}}{2^{q}m} \bigg)^{t} A^{(0)}\geq \exp\bigg(\frac{C_{1}\eta q\sigma_{0}^{q-2}\|\bmu\|_{2}^{q}}{2^{q+1}m}t\bigg)A^{(0)}\geq \exp\bigg(\frac{C_{1}\eta q\sigma_{0}^{q-2}\|\bmu\|_{2}^{q}}{2^{q+1}m}t\bigg)\frac{\sigma_{0}\|\bmu\|_{2}}{2},
\end{align*}
where the second inequality is due to the fact that $1+z \geq \exp(z/2)$ for $z \leq 2$ and our condition of $\eta$ and $\sigma_{0}$ listed in Condition~\ref{condition:d_sigma0_eta}, and the last inequality follows by Lemma~\ref{lemma:initialization_norms} and $A^{(0)} = \max_{r} \la \wb_{1,r}^{(0)}, \bmu \ra$. Therefore, $A^{(t)}$ will reach $3$ within $T_{1} = \frac{\log(6/\sigma_{0}\|\bmu\|_{2})2^{q+1}m}{C_{1}\eta q\sigma_{0}^{q-2}\|\bmu\|_{2}^{q}}$ iterations. Since $\max_{r}\gamma_{1,r}^{(t)} \geq A^{(t)} - \max_{r}|\la \wb_{1,r}^{(0)}, \bmu \ra| \geq A^{(t)} - 1$,  $\max_{r}\gamma_{1,r}^{(t)}$ will reach $2$ within $T_{1}$ iterations. We can next verify that 
\begin{align*}
T_{1} = \frac{\log(6/\sigma_{0}\|\bmu\|_{2})2^{q+1}m}{C_{1}\eta q\sigma_{0}^{q-2}\|\bmu\|_{2}^{q}} \leq \frac{nm\eta^{-1}\sigma_{0}^{2-q}\sigma_{p}^{-q}d^{-q/2}}{2^{q+5}q[4\log(8mn/\delta)]^{(q-1)/2}} =  T_{1}^{+}/2 ,
\end{align*}
where the inequality holds due to our SNR condition in \eqref{eq:explicit condition}. Therefore, by the definition of $T_{1,1}$, we have $T_{1,1} \leq T_{1} \leq T_{1}^{+}/2$, where we use the non-decreasing property of $\gamma$. The proof for $j=-1$ is similar, and we can prove that $\max_{r}\gamma_{-1,r}^{(T_{1,-1})} \geq 2$ while $T_{1,-1} \leq T_{1} \leq T_{1}^{+}/2$, which completes the proof.
\end{proof}

\subsection{Second Stage}

By the results we get in the first stage we know that 
\begin{align*}
\wb_{j,r}^{(T_{1})} &= \wb_{j,r}^{(0)} + j \cdot \gamma_{j,r}^{(T_{1})} \cdot \frac{\bmu}{\|\bmu\|_{2}^{2}} + \sum_{ i = 1}^n \zeta_{j,r,i}^{(T_{1})} \cdot \frac{\bxi_{i}}{\|\bxi_{i}\|_{2}^{2}} + \sum_{ i = 1}^n \omega_{j,r,i}^{(T_{1})} \cdot \frac{\bxi_{i}}{\|\bxi_{i}\|_{2}^{2}}.
\end{align*}
And at the beginning of the second stage, we have following property holds:
\begin{itemize}
\item $\max_{r}\gamma_{j, r}^{(T_{1})} \geq 2, \forall j \in \{\pm 1\}$. 
\item $\max_{j,r,i}|\rho_{j,r,i}^{(T_{1})}| \leq \hat{\beta}$ where $\hat{\beta} = \sigma_0 \sigma_p \sqrt{d} / 2$. 
\end{itemize}
Lemma~\ref{lemma:coefficient_iterative} implies that the learned feature $\gamma_{j,r}^{(t)}$ will not get worse, i.e., for $t \geq T_{1}$, we have that $\gamma_{j,r}^{(t+1)} \geq \gamma_{j,r}^{(t)} $, and therefore $\max_{ r}\gamma_{j, r}^{(t)} \geq 2$. Now we choose $\Wb^{*}$ as follows:
\begin{align*}
\wb^{*}_{j,r} = \wb_{j,r}^{(0)} + 2qm\log(2q/\epsilon) \cdot j \cdot  \frac{\bmu}{\|\bmu\|_{2}^{2}}. 
\end{align*}
Based on the above definition of $\Wb^{*}$, we have the following lemma.
\begin{lemma}\label{lm:distance1}
Under the same conditions as Theorem~\ref{thm:noise_memorization_main}, we have that $\|\Wb^{(T_{1})} - \Wb^{*}\|_{F} \leq \tilde{O}(m^{3/2}\|\bmu\|_{2}^{-1})$.
\end{lemma}
\begin{proof}[Proof of Lemma~\ref{lm:distance1}] We have
\begin{align*}
\|\Wb^{(T_{1})} - \Wb^{*}\|_{F} &\leq \|\Wb^{(T_{1})} - \Wb^{(0)}\|_{F} + \|\Wb^{(0)} - \Wb^{*}\|_{F} \\
&\leq\sum_{j,r} \frac{\gamma_{j,r}^{(T_{1})}}{\|\bmu\|_{2}} + \sum_{j,r,i}\frac{|\zeta_{j,r,i}^{(T_{1})}|}{\|\bxi_{i}\|_{2}} + \sum_{j,r,i}\frac{|\omega_{j,r,i}^{(T_{1})}|}{\|\bxi_{i}\|_{2}} + O(m^{3/2}\log(1/\epsilon))\|\bmu\|_{2}^{-1}\\
&\leq \tilde{O}(m\|\bmu\|^{-1}) + O(nm\sigma_{0}) +  O(m^{3/2}\log(1/\epsilon))\|\bmu\|_{2}^{-1}\\
&\leq \tilde{O}(m^{3/2}\|\bmu\|_{2}^{-1}),
\end{align*}
where the first inequality is by triangle inequality, the second inequality is by our decomposition of $\Wb^{(T_{1})}$ and the definition of $\Wb^{*}$, the third inequality is by Proposition~\ref{Prop:noise} and Lemma~\ref{lemma:phase1_main}, and the last inequality is by our condition of $\sigma_{0}$ in Condition~\ref{condition:d_sigma0_eta}.
\end{proof}

\begin{lemma}\label{lm: Gradient Stable}
Under the same conditions as Theorem~\ref{thm:signal_learning_main},  we have that $y_{i}\la \nabla f(\Wb^{(t)}, \xb_{i}), \Wb^{*} \ra \geq q\log(2q/\epsilon)$ for all $i \in [n]$ and $ T_{1} \leq t\leq T^{*}$.
\end{lemma}
\begin{proof}[Proof of Lemma~\ref{lm: Gradient Stable}] Recall that $f(\Wb^{(t)}, \xb_{i}) = (1/m) {\sum_{j,r}}j\cdot \big[\sigma(\la\wb_{j,r},y_{i}\cdot\bmu\ra) + \sigma(\la\wb_{j,r}, \bxi_{i}\ra)\big] $, so we have 
\begin{align}
y_{i}\la \nabla f(\Wb^{(t)}, \xb_{i}), \Wb^{*} \ra 
&= \frac{1}{m}\sum_{j,r}\sigma'(\la \wb_{j,r}^{(t)}, y_{i}\bmu\ra)\la \bmu,  j\wb_{j,r}^{*}\ra + \frac{1}{m}\sum_{j,r}\sigma'(\la \wb_{j,r}^{(t)}, \bxi_{i}\ra)\la y_{i}\bxi_{i},  j\wb_{j,r}^{*}\ra \notag\\
&= \frac{1}{m}\sum_{j,r}\sigma'(\la \wb_{j,r}^{(t)}, y_{i}\bmu\ra)2qm\log(2q/\epsilon) + \frac{1}{m}\sum_{j,r}\sigma'(\la \wb_{j,r}^{(t)}, y_{i}\bmu\ra)\la \bmu,  j\wb_{j,r}^{(0)}\ra \notag\\
&\qquad + \frac{1}{m}\sum_{j,r}\sigma'(\la \wb_{j,r}^{(t)}, \bxi_{i}\ra)\la y_{i}\bxi_{i},  j\wb_{j,r}^{(0)}\ra \notag\\
&\geq \frac{1}{m}\sum_{j,r}\sigma'(\la \wb_{j,r}^{(t)}, y_{i}\bmu\ra)2qm\log(2q/\epsilon) -  \frac{1}{m}\sum_{j,r}\sigma'(\la \wb_{j,r}^{(t)}, y_{i}\bmu\ra)\tilde{O}(\sigma_{0}\|\bmu\|_{2}) \notag\\
&\qquad - \frac{1}{m}\sum_{j,r}\sigma'(\la \wb_{j,r}^{(t)}, \bxi_{i}\ra)\tilde{O}(\sigma_{0}\sigma_{p}\sqrt{d}),
\label{eq: inner}
\end{align}
where the inequality is by Lemma~\ref{lemma:initialization_norms}. Next we will bound the inner-product terms in \eqref{eq: inner} respectively. By Lemma~\ref{lm: Fyi}, we have that for $j = y_{i}$
\begin{align}
\max_{r}\{\la \wb_{j,r}^{(t)}, y_{i}\bmu \ra\} = \max_{r}\{\gamma_{j,r}^{(t)} + \la \wb_{j,r}^{(0)}, y_{i}\bmu \ra\} \geq 2 - \tilde{O}(\sigma_{0}\|\bmu\|_{2}) \geq 1. \label{eq:GS1}
\end{align}
We can also get the upper bound of the inner products between the parameter and the signal (noise) as follows,
\begin{align}
&|\la \wb_{j,r}^{(t)},  \bmu\ra| \overset{(i)}{\leq} |\la\wb_{j,r}^{(0)}, \bmu\ra| + |\gamma_{j,r}^{(t)}| \overset{(ii)}{\leq} \tilde{O}(1) \notag\\
&|\la \wb_{j,r}^{(t)},  \bxi_{i}\ra| \overset{(iii)}{\leq} |\la\wb_{j,r}^{(0)}, \bxi_{i}\ra| + |\omega_{j,r,i}^{(t)}| + |\zeta_{j,r,i}^{(t)}| + 8n\sqrt{\frac{\log(4n^{2}/\delta)}{d}}\alpha \overset{(iv)}{\leq} \tilde{O}(1), \label{eq:GS2}
\end{align}
where (i) is by Lemma~\ref{lm: mubound}, (iii) is by Lemma~\ref{lm:oppositebound}, (ii) and (iv) are due to Proposition~\ref{Prop:noise}. Plugging \eqref{eq:GS1} and \eqref{eq:GS2} into \eqref{eq: inner} gives,
\begin{align*}
y_{i}\la \nabla f(\Wb^{(t)}, \xb_{i}), \Wb^{*} \ra 
&\geq 2q\log(2q/\epsilon) - \tilde{O}(\sigma_{0}\|\bmu\|_{2}) - \tilde{O}(\sigma_{0}\sigma_{p}\sqrt{d}) \geq q\log(2q/\epsilon),
\end{align*}
where the last inequality is by $\sigma_{0} \leq \tilde{O}(m^{-2/(q-2)}n^{-1})\cdot\min\{(\sigma_{p}\sqrt{d})^{-1}, \|\bmu\|_{2}^{-1}\}$ in Condition~\ref{condition:d_sigma0_eta}. This completes the proof.
\end{proof}

\begin{lemma}\label{lemma:signal_stage2_homogeneity}
Under the same conditions as Theorem~\ref{thm:signal_learning_main}, we have that  
\begin{align*}
\|\Wb^{(t)} - \Wb^{*}\|_{F}^{2} - \|\Wb^{(t+1)} - \Wb^{*}\|_{F}^{2} \geq (2q-1)\eta L_{S}(\Wb^{(t)}) - \eta\epsilon
\end{align*}
for all $ T_{1} \leq t\leq T^{*}$.
\end{lemma}
\begin{proof}[Proof of Lemma~\ref{lemma:signal_stage2_homogeneity}]
We first apply a proof technique similar to Lemma 2.6 in \citet{ji2019polylogarithmic}. The difference between our analysis and \citet{ji2019polylogarithmic} is that here the neural network is $q$ homogeneous rather than 1 homogeneous.
\begin{align*}
&\|\Wb^{(t)} - \Wb^{*}\|_{F}^{2} - \|\Wb^{(t+1)} - \Wb^{*}\|_{F}^{2}\\
&=  2\eta \la \nabla L_{S}(\Wb^{(t)}), \Wb^{(t)} - \Wb^{*}\ra - \eta^{2}\|\nabla L_{S}(\Wb^{(t)})\|_{F}^{2}\\
&=  \frac{2\eta}{n}\sum_{i=1}^{n}\ell'^{(t)}_{i}[qy_{i}f(\Wb^{(t)},\xb_{i}) - \la \nabla f(\Wb^{(t)}, \xb_{i}), \Wb^{*} \ra] - \eta^{2}\|\nabla L_{S}(\Wb^{(t)})\|_{F}^{2}\\
&\geq   \frac{2\eta}{n}\sum_{i=1}^{n}\ell'^{(t)}_{i}[qy_{i}f(\Wb^{(t)},\xb_{i}) - q\log(2q/\epsilon)] - \eta^{2}\|\nabla L_{S}(\Wb^{(t)})\|_{F}^{2}\\
&\geq \frac{2q\eta}{n}\sum_{i=1}^{n}[\ell\big(y_{i}f(\Wb^{(t)}, \xb_{i})\big) - \epsilon/(2q)] - \eta^{2}\|\nabla L_{S}(\Wb^{(t)})\|_{F}^{2}\\
&\geq (2q-1)\eta L_{S}(\Wb^{(t)}) - \eta\epsilon,
\end{align*}
where the first inequality is by Lemma~\ref{lm: Gradient Stable}, the second inequality is due to the convexity of the cross entropy function, and the last inequality is due to Lemma~\ref{lm: gradient upbound}.
\end{proof}

\begin{lemma}[Restatement of Lemma~\ref{lemma:signal_proof_sketch}]\label{thm:signal_proof}
Under the same conditions as Theorem~\ref{thm:signal_learning_main}, let $T = T_{1} + \Big\lfloor \frac{\|\Wb^{(T_{1})} - \Wb^{*}\|_{F}^{2}}{2\eta \epsilon}
\Big\rfloor = T_{1} + \tilde{O}(m^{3}\eta^{-1}\epsilon^{-1}\|\bmu\|_{2}^{-2})$. Then we have $\max_{j,r,i}|\rho_{j,r,i}^{(t)}| \leq 2\hat{\beta} = \sigma_{0}\sigma_{p}\sqrt{d}$ for all $T_{1} \leq t\leq T$. Besides,
\begin{align*}
\frac{1}{t - T_{1} + 1}\sum_{s=T_{1}}^{t}L_{S}(\Wb^{(s)}) \leq  \frac{\|\Wb^{(T_{1})} - \Wb^{*}\|_{F}^{2}}{(2q-1) \eta(t - T_{1} + 1)} + \frac{\epsilon}{2q-1} 
\end{align*}
for all $T_{1} \leq t\leq T$, and we can find an iterate with training loss smaller than $\epsilon$ within $T $ iterations.
\end{lemma}
\begin{proof}[Proof of Lemma~\ref{thm:signal_proof}]
By Lemma~\ref{lemma:signal_stage2_homogeneity}, for any $t\in[T_1,T]$, we have that  
\begin{align*}
\|\Wb^{(s)} - \Wb^{*}\|_{F}^{2} - \|\Wb^{(s+1)} - \Wb^{*}\|_{F}^{2} \geq (2q-1)\eta L_{S}(\Wb^{(s)}) - \eta\epsilon 
\end{align*}
holds for $s \leq t$. Taking a summation, we obtain that 
\begin{align}
\sum_{s=T_{1}}^{t}L_{S}(\Wb^{(s)}) &\leq \frac{\|\Wb^{(T_{1})} - \Wb^{*}\|_{F}^{2} + \eta\epsilon (t - T_{1} + 1)}{(2q-1) \eta}\label{eq:vanillasum}
\end{align}
for all $T_{1} \leq t\leq T$. 
Dividing $(t-T_{1}+1)$ on both side of \eqref{eq:vanillasum} gives that 
\begin{align*}
\frac{1}{t - T_{1} + 1}\sum_{s=T_{1}}^{t}L_{S}(\Wb^{(s)}) \leq  \frac{\|\Wb^{(T_{1})} - \Wb^{*}\|_{F}^{2}}{(2q-1) \eta(t - T_{1} + 1)} + \frac{\epsilon}{2q-1}.  
\end{align*}
Then we can take $t= T$ and have that 
\begin{align*}
\frac{1}{T - T_{1} + 1}\sum_{s=T_{1}}^{T}L_{S}(\Wb^{(s)}) \leq  \frac{\|\Wb^{(T_{1})} - \Wb^{*}\|_{F}^{2}}{(2q-1) \eta(T - T_{1} + 1)} + \frac{\epsilon}{2q-1} \leq \frac{3\epsilon}{2q-1} < \epsilon, 
\end{align*}
where we use the fact that $q> 2$ and our choice that  $T  = T_{1} + \Big\lfloor \frac{\|\Wb^{(T_{1})} - \Wb^{*}\|_{F}^{2}}{2\eta \epsilon}
\Big\rfloor$. Because the mean is smaller than $\epsilon$, we can conclude that there exist $T_{1} \leq t \leq T$ such that $L_{S}(\Wb^{(t)}) < \epsilon$. 

Finally, we will prove that $\max_{j,r,i}|\rho_{j,r,i}^{(t)}| \leq 2\hat{\beta}$ for all $ t \in [T_1, T]$. Plugging $T  = T_{1} + \Big\lfloor \frac{\|\Wb^{(T_{1})} - \Wb^{*}\|_{F}^{2}}{2\eta \epsilon}
\Big\rfloor$ into \eqref{eq:vanillasum} gives that 
\begin{align}
\sum_{s=T_{1}}^{T}L_{S}(\Wb^{(s)}) &\leq \frac{2\|\Wb^{(T_{1})} - \Wb^{*}\|_{F}^{2}}{(2q-1) \eta}  = \tilde{O}(\eta^{-1}m^{3}\|\bmu\|_{2}^{2}), \label{eq: sum1}
\end{align}
where the inequality is due to $\|\Wb^{(T_{1})} - \Wb^{*}\|_{F} \leq \tilde{O}(m^{3/2}\|\bmu\|_{2}^{-1})$ in Lemma~\ref{lm:distance1}. Define $\Psi^{(t)} =  \max_{j,r,i}|\rho_{j,r,i}^{(t)}|$. We will  use induction to prove $\Psi^{(t)} \leq 2\hat{\beta}$ for all $ t \in [T_1, T]$. At $t = T_1$, by the definition of $\hat\beta$, clearly we have $\Psi^{(T_1)} \leq \hat{\beta} \leq 2\hat{\beta}$. Now suppose that there exists $\tilde{T} \in [T_1, T]$ such that $\Psi^{(t)} \leq 2\hat{\beta}$ for all $t \in [T_1, \tilde{T}-1]$. Then for $t \in [T_1, \tilde{T}-1]$, by \eqref{eq:update_zeta2} and \eqref{eq:update_omega2} we have
\begin{align*}
    \Psi^{(t+1)} &\leq \Psi^{(t)} + \max_{j,r,i}\bigg\{\frac{\eta }{nm} \cdot |\ell_i'^{(t)}| \cdot \sigma'\Bigg(\la\wb_{j,r}^{(0)}, \bxi_{i}\ra + 2\sum_{ i'= 1 }^n \Psi^{(t)} \cdot \frac{ |\la \bxi_{i'}, \bxi_i \ra|}{ \| \bxi_{i'} \|_2^2}  \Bigg)\cdot \| \bxi_{i'} \|_2^2 \bigg\} \\
    &= \Psi^{(t)} + \max_{j,r,i}\bigg\{\frac{\eta }{nm} \cdot |\ell_i'^{(t)}| \cdot \sigma'\Bigg(\la\wb_{j,r}^{(0)}, \bxi_{i}\ra + 2\Psi^{(t)} + 2\sum_{ i'\neq i }^n \Psi^{(t)} \cdot \frac{ |\la \bxi_{i'}, \bxi_i \ra|}{ \| \bxi_{i'} \|_2^2}  \Bigg)\cdot \| \bxi_{i'} \|_2^2 \bigg\} \\
    &\leq \Psi^{(t)} + \frac{\eta q}{nm} \cdot \max_{i}|\ell_i'^{(t)}|\cdot \Bigg[2\cdot\sqrt{ \log(8mn/\delta)} \cdot \sigma_0 \sigma_p \sqrt{d} \\
    &\qquad+ \Bigg( 2+ \frac{4n \sigma_p^2 \cdot \sqrt{d \log(4n^2 / \delta)} }{ \sigma_p^2 d /2} \Bigg) \cdot \Psi^{(t)}  \Bigg]^{q-1}\cdot 2 \sigma_p^2 d\\
    &\leq \Psi^{(t)} + \frac{\eta q}{nm} \cdot \max_{i}|\ell_i'^{(t)}|\cdot \big(2\cdot \sqrt{ \log(8mn/\delta) } \cdot \sigma_0 \sigma_p \sqrt{d} + 4 \cdot \Psi^{(t)}  \big)^{q-1}\cdot 2 \sigma_p^2 d,
\end{align*}
where the second inequality is due to  Lemmas~\ref{lemma:data_innerproducts} and \ref{lemma:initialization_norms}, and the last inequality follows by the assumption that $d \geq 16n^2 \log (4n^2/\delta)$. Taking a telescoping sum over $t=0,1,\ldots, \tilde{T}-1$, we have that 
\begin{align*}
    \Psi^{(T)} 
    &\overset{(i)}{\leq} \Psi^{(T_{1})} + \frac{\eta q}{nm} \sum_{s=T_{1}}^{\tilde{T}-1}\max_{i} |\ell_i'^{(s)}|\tilde{O}(\sigma_{p}^{2}d)\hat{\beta}^{q-1}\\
    &\overset{(ii)}{\leq} \Psi^{(T_{1})} + \frac{\eta q}{nm}\tilde{O}(\sigma_{p}^{2}d)\hat{\beta}^{q-1} \sum_{s=T_{1}}^{\tilde{T}-1}\max_{i} \ell_{i}^{(s)}\\
    &\overset{(iii)}{\leq} \Psi^{(T_{1})} + \tilde{O}(\eta m^{-1}\sigma_{p}^{2}d)\hat{\beta}^{q-1} \sum_{s=T_{1}}^{\tilde{T}-1}L_{S}(\Wb^{(s)})\\
    &\overset{(iv)}{\leq} \Psi^{(T_{1})} + \tilde{O}(m^{2}\mathrm{SNR}^{-2})\hat{\beta}^{q-1}\\
    &\overset{(v)}{\leq} \hat{\beta} + \tilde{O}(m^{2}n^{2/q}\hat{\beta}^{q-2})\hat{\beta}\\
    &\overset{(vi)}{\leq} 2\hat{\beta}, 
\end{align*}
where (i) is by out induction hypothesis that $\Psi^{(t)} \leq 2\hat{\beta}$, (ii) is by $|\ell'| \leq \ell$, (iii) is by $\max_{i}\ell_{i}^{(s)} \leq \sum_{i}\ell_{i}^{(s)} = n L_{S}(\Wb^{(s)})$, (iv) is due to $\sum_{s=T_{1}}^{\tilde{T} - 1}L_{S}(\Wb^{(s)}) \leq \sum_{s=T_{1}}^{T}L_{S}(\Wb^{(s)})  = \tilde{O}(\eta^{-1}m^{3}\|\bmu\|_{2}^{2})$ in \eqref{eq: sum1}, (v) is by $n \mathrm{SNR}^{q} \geq \tilde{\Omega}(1)$, and (vi) is by the definition that $\hat{\beta} = \sigma_{0}\sigma_{p}\sqrt{d} / 2$ and $\tilde{O}(m^{2}n^{2/q}\hat{\beta}^{q-2}) = \tilde{O}(m^{2}n^{2/q}(\sigma_{0}\sigma_{p}\sqrt{d})^{q-2}) \leq 1$ by Condition~\ref{condition:d_sigma0_eta}. Therefore, $\Psi^{(\tilde{T})} \leq 2\hat{\beta}$, which completes the induction. 
\end{proof}


\subsection{Population Loss}
Consider a new data point $(\xb,y)$ drawn from the distribution defined in Definition \ref{def:data}. Without loss of generality, we suppose that the first patch is the signal patch and the second patch is the noise patch, i.e., $\xb = [y\bmu, \bxi]$. Moreover, by the signal-noise decomposition, the learned neural network has parameter
\begin{align*}
\wb_{j,r}^{(t)} &= \wb_{j,r}^{(0)} + j \cdot \gamma_{j,r}^{(t)} \cdot \frac{\bmu}{\|\bmu\|_{2}^{2}} + \sum_{ i = 1}^n \zeta_{j,r,i}^{(t)} \cdot \frac{\bxi_{i}}{\|\bxi_{i}\|_{2}^{2}} + \sum_{ i = 1}^n \omega_{j,r,i}^{(t)} \cdot \frac{\bxi_{i}}{\|\bxi_{i}\|_{2}^{2}}
\end{align*}
for $j\in\{\pm 1\}$ and $r\in [m]$.

\begin{lemma}\label{lm:signalg2}
Under the same conditions as Theorem~\ref{thm:signal_learning_main}, we have that $\max_{j,r}|\la \wb_{j,r}^{(t)}, \bxi_{i} \ra| \leq 1/2$ for all $0 \leq t \leq T$.
\end{lemma}
\begin{proof}
We can get the upper bound of the inner products between the parameter and the noise as follows:
\begin{align*}
|\la \wb_{j,r}^{(t)},  \bxi_{i}\ra| &\overset{(i)}{\leq} |\la\wb_{j,r}^{(0)}, \bxi_{i}\ra| + |\omega_{j,r,i}^{(t)}| + |\zeta_{j,r,i}^{(t)}| + 8n\sqrt{\frac{\log(4n^{2}/\delta)}{d}}\alpha \\
&\overset{(ii)}{\leq} 2\sqrt{\log(8mn/\delta)}\cdot \sigma_{0}\sigma_{p}\sqrt{d} + \sigma_{0}\sigma_{p}\sqrt{d} + 8n\sqrt{\frac{\log(4n^{2}/\delta)}{d}}\alpha\\
&\overset{(iii)}{\leq} 1/2
\end{align*}
for all $j\in \{\pm 1\}$, $r\in [m]$ and $i\in[n]$,
where (i) is by Lemma~\ref{lm: mubound}, (ii) is due to  $|\la\wb_{j,r}^{(0)}, \bxi_{i}\ra| \leq 2\sqrt{\log(8mn/\delta)}\cdot \sigma_{0}\sigma_{p}\sqrt{d}$ in  Lemma~\ref{lemma:initialization_norms} and $\max_{j,r,i}|\rho_{j,r,i}^{(t)}| \leq  \sigma_{0}\sigma_{p}\sqrt{d}$ in Lemma~\ref{thm:signal_proof}, and (iii) is due to our condition of $\sigma_{0}\leq \tilde{O}(m^{-2/(q-2)}n^{-1})\cdot (\sigma_{p}\sqrt{d})^{-1}$ and $d\geq \tilde{\Omega}(m^{2}n^{4})$ in Condition~\ref{condition:d_sigma0_eta}.
\end{proof}

\begin{lemma}\label{lm:signalg}
Under the same conditions as Theorem~\ref{thm:signal_learning_main}, with probability at least $1 - 4mT \cdot \exp(-C_{2}^{-1}\sigma_{0}^{-2}\sigma_{p}^{-2}d^{-1})$, we have that $\max_{j,r}|\la \wb_{j,r}^{(t)}, \bxi \ra| \leq 1/2$ for all $0 \leq t \leq T$, where $C_{2} = \tilde{O}(1)$.
\end{lemma}
\begin{proof}[Proof of Lemma~\ref{lm:signalg}]
Let $\tilde{\wb}_{j,r}^{(t)} = \wb_{j,r}^{(t)} - j \cdot \gamma_{j,r}^{(t)} \cdot \frac{\bmu}{\|\bmu\|_{2}^{2}}$, then we have that $\la\tilde{\wb}_{j,r}^{(t)}, \bxi\ra = \la\wb_{j,r}^{(t)}, \bxi\ra$ and
\begin{align}
\|\tilde{\wb}_{j,r}^{(t)}\|_{2} \leq \tilde{O}(\sigma_{0}\sqrt{d} + n\sigma_{0}) =  \tilde{O}(\sigma_{0}\sqrt{d}), \label{eq:tildew}
\end{align} 
where the equality is due to $d \geq \tilde{\Omega}(m^{2}n^{4})$ by Condition~\ref{condition:d_sigma0_eta}.

By \eqref{eq:tildew}, $\max_{j,r}\|\tilde{\wb}_{j,r}^{(t)}\|_{2} \leq C_{1}\sigma_{0}\sqrt{d}$, where $C_{1} = \tilde{O}(1)$. Clearly $\la \tilde{\wb}_{j,r}^{(t)}, \bxi \ra $ is a Gaussian distribution with mean zero and standard deviation smaller than $C_{1}\sigma_{0}\sigma_{p}\sqrt{d}$. Therefore, the probability is bounded by 
\begin{align*}
\mathbb{P}\big(|\la \tilde{\wb}_{j,r}^{(t)}, \bxi \ra|\geq 1/2\big) \leq 2\exp\bigg(- \frac{1}{8C_{1}^{2}\sigma_{0}^{2}\sigma_{p}^{2}d}\bigg).
\end{align*}
Applying a union bound over $j,r, t$ completes the proof.
\end{proof}


\begin{lemma}[Restatement of Lemma~\ref{lemma:signal_polulation_loss_main}]\label{lemma:signal_polulation_loss}
Let $T$ be defined in Lemma~\ref{lemma:phase1_main_sketch} respectively. Under the same conditions as Theorem~\ref{thm:signal_learning_main}, for any $0 \leq t \leq T$ with $L_{S}(\Wb^{(t)}) \leq 1$, it holds that $L_{\cD}(\Wb^{(t)}) \leq 6 \cdot L_{S}(\Wb^{(t)}) + \exp(- n^{2})$.
\end{lemma}
\begin{proof}[Proof of Lemma~\ref{lemma:signal_polulation_loss}]
Let event $\cE$ to be the event that Lemma~\ref{lm:signalg} holds. Then we can divide $L_{\cD}(\Wb^{(t)})$ into two parts:

\begin{align}
\EE\big[\ell\big(yf(\Wb^{(t)}, \xb)\big)\big] &= \underbrace{\EE[\ind(\cE)\ell\big(yf(\Wb^{(t)}, \xb)\big)]}_{I_{1}} + \underbrace{\EE[\ind(\cE^{c})\ell\big(yf(\Wb^{(t)}, \xb)\big)]}_{I_{2}}.\label{generalize all}
\end{align}
In the following, we bound $I_1$ and $I_2$ respectively.

\noindent\textbf{Bounding $I_1$:} Since $L_{S}(\Wb^{(t)}) \leq 1$, there must exist one $(\xb_{i},y_{i})$ such that $\ell\big(y_{i}f(\Wb^{(t)}, \xb_{i})\big)\leq L_{S}(\Wb^{(t)}) \leq 1$, which implies that $y_{i}f(\Wb^{(t)},\xb_{i}) \geq 0$. Therefore, we have that
\begin{align}
\exp(-y_{i}f(\Wb^{(t)},\xb_{i}))\overset{(i)}{\leq} 2\log\big(1+\exp(-y_{i}f(\Wb^{(t)},\xb_{i}))\big) = 2\ell\big(y_{i}f(\Wb^{(t)}, \xb_{i})\big) \leq 2L_{S}(\Wb^{(t)}),\label{eq:I1bbd}    
\end{align}
where (i) is by $z \leq 2\log(1+z), \forall z \leq 1$. If event $\cE$ holds, we have that 
\begin{align}
|yf(\Wb^{(t)}, \xb) - y_{i}f(\Wb^{(t)}, \xb_{i})| \nonumber
&\leq \frac{1}{m}\sum_{j,r}\sigma(\la\wb_{j,r}, \bxi_{i}\ra) + \frac{1}{m}\sum_{j,r}\sigma(\la\wb_{j,r}, \bxi\ra ) \nonumber\\
&\leq \frac{1}{m}\sum_{j,r}\sigma(1/2) + \frac{1}{m}\sum_{j,r}\sigma(1/2)\nonumber\\
&\leq 1,\label{eq:I1bbd2} 
\end{align}
where the second inequality is by $\max_{j,r}|\la \wb_{j,r}^{(t)}, \bxi \ra| \leq 1/2$ in Lemma~\ref{lm:signalg} and $\max_{j,r}|\la \wb_{j,r}^{(t)}, \bxi_{i} \ra| \leq 1/2$ in Lemma~\ref{lm:signalg2}.
Thus we have that 
\begin{align*}
I_{1}&\leq \EE[\ind(\cE) \exp(-yf(\Wb^{(t)}, \xb))] \\   
&\leq e\cdot\EE[\ind(\cE) \exp(-y_{i}f(\Wb^{(t)}, \xb_{i}))] \\
&\leq 2e\cdot\EE[\ind(\cE)  L_{S}(\Wb^{(t)})],
\end{align*}
where the first inequality is by the property of cross-entropy loss that $\ell(z) \leq \exp(-z)$ for all $z$, the second inequality is by \eqref{eq:I1bbd2}, and the third inequality is by \eqref{eq:I1bbd}. Dropping the event in the expectation gives $I_{1} \leq 6L_{S}(\Wb^{(t)})$. 

\noindent\textbf{Bounding $I_2$:} Next we bound the second term $I_{2}$. We choose an arbitrary training data $(\xb_{i'},y_{i'})$ such that $y_{i'} = y$. Then we have 
\begin{align}
\ell\big(yf(\Wb^{(t)}, \xb)\big) &\leq \log(1 + \exp(F_{-y}(\Wb^{(t)}, \xb))) \notag\\    
&\leq 1 + F_{-y}(\Wb^{(t)}, \xb)\notag\\
&= 1 + \frac{1}{m}\sum_{j=-y,r\in [m]}\sigma(\la\wb_{j,r}^{(t)}, y\bmu\ra) + \frac{1}{m}\sum_{j=-y,r\in [m]}\sigma(\la\wb_{j,r}^{(t)}, \bxi\ra)\notag\\
&\leq 1 + F_{-y_{i]}}(\Wb_{-y_{i'}}, \xb_{i'}) + \frac{1}{m}\sum_{j=-y,r\in [m]}\sigma(\la\wb_{j,r}^{(t)}, \bxi\ra) \notag\\
&\leq 2  + \frac{1}{m}\sum_{j=-y,r\in [m]}\sigma(\la\wb_{j,r}^{(t)}, \bxi\ra)\notag\\
&\leq 2 + \tilde{O}((\sigma_{0}\sqrt{d})^{q})\|\bxi\|^{q}\label{eq:1.1},
\end{align}
where the first inequality is due to $F_{y}(\Wb^{(t)}, \xb) \geq 0$, the second inequality is by the property of cross-entropy loss, i.e., $\log(1+\exp(z)) \leq 1+z$ for all $z\geq 0$, the third inequality is by $ \frac{1}{m}\sum_{j=-y,r\in [m]}\sigma(\la\wb_{j,r}^{(t)}, y\bmu\ra) \leq F_{-y}(\Wb_{-y},\xb_{i'}) = F_{-y_{i'}}(\Wb_{-y_{i'}},\xb_{i'})$, the fourth inequality is by $F_{-y_{i'}}(\Wb_{-y_{i'}},\xb_{i'}) \leq 1$ in Lemma~\ref{lm: F-yi}, and the last inequality is due to $\la\tilde{\wb}_{j,r}^{(t)}, \bxi\ra = \la\wb_{j,r}^{(t)}, \bxi\ra \leq \|\tilde{\wb}_{j,r}^{(t)}\|_{2}\|\bxi\|_{2} \leq \tilde{O}(\sigma_{0}\sqrt{d})\|\bxi\|_{2}$ in \eqref{eq:tildew}. Then we further have that
\begin{align*}
I_{2} &\leq   \sqrt{\EE[\ind(\cE^{c})]} \cdot \sqrt{\EE\Big[\ell\big(yf(\Wb^{(t)}, \xb)\big)^{2}\Big]}\\
&\leq \sqrt{\mathbb{P}(\cE^{c})} \cdot \sqrt{4 + \tilde{O}((\sigma_{0}\sqrt{d})^{2q})\EE[\|\bxi\|_{2}^{2q}]}\\
&\leq \exp [ -\tilde\Omega(\sigma_{0}^{-2}\sigma_{p}^{-2}d^{-1}) +\text{polylog}(d) ]\\
&\leq \exp(- n^{2}),
\end{align*}
where the first inequality is by Cauchy-Schwartz inequality, the second inequality is by \eqref{eq:1.1}, the third inequality is by Lemma~\ref{lm:signalg} and the fact that $\sqrt{4 + \tilde{O}((\sigma_{0}\sqrt{d})^{2q})\EE[\|\bxi\|_{2}^{2q}]} = O(\poly(d))$, and the last inequality is by our condition $\sigma_{0}\leq \tilde{O}(m^{-2/(q-2)}n^{-1})\cdot (\sigma_{p}\sqrt{d})^{-1}$ in Condition~\ref{condition:d_sigma0_eta}. Plugging the bounds of $I_{1}$, $I_{2}$ into \eqref{generalize all} completes the proof.
\end{proof}

\section{Noise Memorization}
In this section, we will consider the noise memorization case under the condition that $\sigma_{p}^{q}(\sqrt{d})^{q}  \geq \tilde{\Omega}(n\|\bmu\|_{2}^{q})$. We remind the readers that the proofs in this section are based on the results in Section~\ref{sec: initial}, which hold with high probability. 

We also remind readers that $\alpha = 4\log(T^{*})$ is defined in Appendix~\ref{section:decompositionproof}. 
Denote $\bar{\beta} = \min_{i}\max_{r}\la \wb_{y_{i},r}^{(0)},\bxi_{i}\ra$. The following lemma provides a lower bound of $\bar{\beta}$.

\begin{lemma}\label{lm: betabarlower}
Under the same conditions as Theorem~\ref{thm:noise_memorization_main}, if in particular
\begin{align}
\sigma_{0} \geq  80n\sqrt{\frac{\log(4n^{2}/\delta)}{d}}\alpha \cdot \min\{(\sigma_{p}\sqrt{d})^{-1}, \|\bmu\|_{2}^{-1}\}, \label{eq:exactsigma}
\end{align}
then we have that  $\bar{\beta} \geq \sigma_{0}\sigma_{p}\sqrt{d}/4 \geq  20n\sqrt{\frac{\log(4n^{2}/\delta)}{d}}\alpha$.
\end{lemma}
\begin{proof}[Proof of Lemma~\ref{lm: betabarlower}]
Because $\sigma_{p}^{q}(\sqrt{d})^{q}  \geq \tilde{\Omega}(n\|\bmu\|_{2}^{q})$, we have that $\sigma_p\sqrt{d} \geq \|\bmu\|_{2}$. Therefore we have that 
\begin{align*}
\bar{\beta}&\geq \sigma_{0}\sigma_{p}\sqrt{d}/4\\
&=\sigma_{0}/4\cdot\max\{\sigma_{p}\sqrt{d}, \|\bmu\|_{2}\}\\
&\geq 20n\sqrt{\frac{\log(4n^{2}/\delta)}{d}}\alpha,
\end{align*}
where the first inequality is by Lemma~\ref{lemma:initialization_norms} and the last inequality is by our lower bound condition of $\sigma_{0}$ in \eqref{eq:exactsigma}.
\end{proof}

\subsection{First Stage}
\begin{lemma}\label{lemma:phase1_main_noise}
Under the same conditions as Theorem~\ref{thm:noise_memorization_main}, in particular if we choose 
\begin{align}
n^{-1}\mathrm{SNR}^{-q} \geq \frac{C2^{q+2}\log\big(20/(\sigma_{0}\sigma_{p}\sqrt{d})\big) (\sqrt{2 \log(8m/\delta)})^{q-2}}{ 0.15^{q-2}  }, \label{eq:explicit condition2} 
\end{align}
where $C = O(1)$ is a positive constant, then 
there exist 
\begin{align*}
T_{1} = \frac{C\log\big(20/(\sigma_{0}\sigma_{p}\sqrt{d})\big)4mn}{0.15^{q-2}\eta q\sigma_{0}^{q-2}(\sigma_{p}^{2}\sqrt{d})^{q}}  
\end{align*}
such that
\begin{itemize}
\item $\max_{j,r}\zeta_{j,r,i}^{(T_{1})} \geq 2$ for all $i \in [n]$.
\item $\max_{j,r}\gamma_{j,r}^{(t)} = \tilde{O}(\sigma_{0}\|\bmu\|_{2})$ for all $0 \leq t \leq T_{1}$.
\item $\max_{j,r,i}|\omega_{j,r,i}^{(t)}| = \tilde{O}(\sigma_{0}\sigma_{p}\sqrt{d})$ for all $0 \leq t \leq T_{1}$.
\end{itemize}
\end{lemma}

\begin{proof}[Proof of Lemma~\ref{lemma:phase1_main_noise}]
Let 
\begin{align}
T_{1}^{+} = \frac{m}{\eta q 2^{q-1} (\sqrt{2 \log(8m/\delta)})^{q-2} \sigma_0^{q-2}\|\bmu\|_{2}^q}. \label{eq: T1+}
\end{align}
By Proposition~\ref{Prop:noise}, we have that $\omega_{j,r,i}^{(t)} \geq -\beta - 16n\sqrt{\frac{\log(4n^{2}/\delta)}{d}}\alpha \geq -\beta - \bar{\beta}$ for all $j\in \{\pm 1\}$, $r\in [m]$, $i\in [n]$ and $0 \leq t\leq T^*$. Since $\omega_{j,r,i}^{(t)} \leq 0$ and $\bar{\beta} \leq \beta = \tilde{O}(\sigma_{0}\sigma_{p}\sqrt{d})$, we have that $\max_{j,r,i}|\omega_{j,r,i}^{(t)}| = \tilde{O}(\sigma_{0}\sigma_{p}\sqrt{d})$.
Next, we will carefully compute the growth of the $\gamma_{j,r}^{(t)}$.
\begin{align*}
    \gamma_{j,r}^{(t+1)} &= \gamma_{j,r}^{(t)} - \frac{\eta}{nm} \cdot \sum_{i=1}^n \ell_i'^{(t)} \cdot \sigma'(\la\wb_{j,r}^{(t)}, y_{i} \cdot \bmu\ra)\|\bmu\|_{2}^{2}\\
    &\leq \gamma_{j,r}^{(t)} + \frac{\eta}{nm}\cdot \sum_{i=1}^{n} \sigma'(|\la\wb_{j,r}^{(0)}, \bmu\ra| + \gamma_{j,r}^{(t)})\|\bmu\|_{2}^{2},
\end{align*}
where the inequality is by $|\ell'|\leq 1$. Let $A^{(t)} = \max_{j,r}\{\gamma_{j,r}^{(t)} + |\la\wb_{j,r}^{(0)}, \bmu\ra|\}$, then we have that
\begin{align}
A^{(t+1)} \leq A^{(t)} + \frac{\eta q \|\bmu\|_{2}^2}{m} [A^{(t)}]^{q-1}.\label{telescope initial}
\end{align}
We will use induction to prove that $A^{(t)} \leq 2A^{(0)}$ for $t\leq T_{1}^{+}$. By definition, clearly we have that $A^{(0)} \leq 2A^{(0)}$. Now suppose that there exists some $\tilde{T} \leq T_{1}^{+}$ such that $A^{(t)} \leq 2A^{(0)}$ holds for $0 \leq t\leq \tilde{T}-1$. Taking a telescoping sum of \eqref{telescope initial} gives that 
\begin{align*}
A^{(\tilde{T})} &\leq A^{(0)} + \sum_{s=0}^{\tilde{T}} \frac{\eta q \|\bmu\|_{2}^{2}}{m} [A^{(s)}]^{q-1}\\
&\leq A^{(0)} +  \frac{\eta q \|\bmu\|_{2}^{2}T_{1}^{+}2^{q-1}}{m}[A^{(0)}]^{q-1}\\
&\leq  A^{(0)} +  \frac{\eta q \|\bmu\|_{2}^{2}T_{1}^{+}2^{q-1}}{m}[\sqrt{2 \log(8m/\delta)} \cdot \sigma_0 \| \bmu \|_2]^{q-2}A^{(0)}\\
&\leq 2A^{(0)},
\end{align*}
where the second inequality is by our induction hypothesis, the third inequality is by $A_{0} \leq \sqrt{2 \log(8m/\delta)} \cdot \sigma_0 \| \bmu \|_2$ in Lemma~\ref{lemma:initialization_norms}, and the last inequality is by \eqref{eq: T1+}.
Thus we have that $A^{(t)}\leq 2A^{(0)}$ for all $t\leq T_1^{+}$. Therefore, $\max_{j,r}\gamma_{j,r}^{(t)}\leq A^{(t)} + \max_{j,r}\{|\la \wb_{j,r}^{(0)}, \bmu\ra|\} \leq 3A^{(0)}$ for all $0 \leq  t\leq T_1^{+}$.
Recall that 
\begin{align*}
\zeta_{j,r,i}^{(t+1)} &= \zeta_{j,r,i}^{(t)} - \frac{\eta}{nm} \cdot \ell_i'^{(t)}\cdot \sigma'(\la\wb^{(t)}_{j,r}, \bxi_{i}\ra) \cdot \ind(y_{i} = j)\|\bxi_{i}\|_{2}^{2}.   
\end{align*}
For $y_{i} = j$, Lemma~\ref{lm:oppositebound} implies that
\begin{align*}
\la \wb_{j,r}^{(t)} ,\bxi_{i} \ra 
&\geq \la\wb_{j,r}^{(0)}, \bxi_{i}\ra + \zeta_{j,r,i}^{(t)} - 8n\sqrt{\frac{\log(4n^{2}/\delta)}{d}}\alpha\\
&\geq \zeta^{(t)}_{j,r,i} + \la\wb_{j,r}^{(0)}, \bxi_{i}\ra - 0.4\bar{\beta},
\end{align*}
where the last inequality is by $\bar{\beta} \geq 20n\sqrt{\frac{\log(4n^{2}/\delta)}{d}}\alpha$. 
Now let $B_{i}^{(t)} = \max_{j = y_{i},r}\{\zeta_{j,r,i}^{(t)} + \la \wb_{j,r}^{(0)}, \bxi_{i}\ra - 0.4\bar{\beta}\}$.
For each $i$, denote by $T_{1}^{(i)}$ the last time in the period $[0, T_{1}^{+}]$ satisfying that $\zeta_{j,r,i}^{(t)} \leq 2$. Then for $t \leq T_{1}^{(i)}$, $\max_{j,r}\{ |\zeta_{j,r,i}^{(t)}|, |\omega_{j,r,i}^{(t)}|\} = O(1)$ and $\max_{j,r}\gamma_{j,r}^{(t)} \leq 3A^{(0)} = O(1)$. Therefore, by Lemmas~\ref{lm: F-yi} and~\ref{lm: Fyi}, we know that $F_{-1}(\Wb^{(t)},\xb_{i}), F_{+1}(\Wb^{(t)},\xb_{i}) = O(1)$. Thus there exists a positive constant $C_{1}$ such that $-\ell'^{(t)}_{i}  \geq C_{1} $ for all $0 \leq  t\leq T_{1}^{(i)}$. It is also easy to check that $B_{i}^{(0)} \geq 0.6 \bar{\beta} \geq 0.15\sigma_{0}\sigma_{p}\sqrt{d}$. Then we can carefully compute the growth of $B_{i}^{(t)}$,
\begin{align*}
B_{i}^{(t+1)} &\geq B_{i}^{(t)} + \frac{C_{1}\eta q\sigma_{p}^{2}d}{2nm}[B_{i}^{(t)}]^{q-1}\\
&\geq B_{i}^{(t)} +  \frac{C_{1}\eta q\sigma_{p}^{2}d}{2nm}[B_{i}^{(0)}]^{q-2}B_{i}^{(t)}\\
&\geq \bigg(1 +  \frac{C_{1}0.15^{q-2}\eta q\sigma_{0}^{q-2}(\sigma_{p}^{2}\sqrt{d})^{q}}{2nm}\bigg)B_{i}^{(t)},
\end{align*}
where the second inequality is by the non-decreasing property of $B_{i}^{(t)}$. Therefore, $B_{i}^{(t)}$ is an exponentially increasing sequence and we have that 
\begin{align*}
B_{i}^{(t)} &\geq \bigg(1 +  \frac{C_{1}0.15^{q-2}\eta q\sigma_{0}^{q-2}(\sigma_{p}^{2}\sqrt{d})^{q}}{2nm}\bigg)^{t}B_{i}^{(0)} \\
&\geq  \exp\bigg(\frac{C_{1}0.15^{q-2}\eta q\sigma_{0}^{q-2}(\sigma_{p}^{2}\sqrt{d})^{q}}{4nm}t\bigg)B_{i}^{(0)}\\
&\geq \exp\bigg(\frac{C_{1}0.15^{q-2}\eta q\sigma_{0}^{q-2}(\sigma_{p}^{2}\sqrt{d})^{q}}{4nm}t\bigg)\cdot 0.15\sigma_{0}\sigma_{p}\sqrt{d},
\end{align*}
where the second inequality is due to the fact that $1+z \geq \exp(z/2)$ for $z \leq 2$ and our conditions of $\eta$ and $\sigma_{0}$ listed in Condition~\ref{condition:d_sigma0_eta}, and the last inequality is due to $B_{i}^{(0)}\geq 0.15\sigma_{0}\sigma_{p}\sqrt{d}$. Therefore, $B_{i}^{(t)}$ will reach $3$ within $T_{1} = \frac{\log\big(20/(\sigma_{0}\sigma_{p}\sqrt{d})\big)4mn}{C_{1}0.15^{q-2}\eta q\sigma_{0}^{q-2}(\sigma_{p}^{2}\sqrt{d})^{q}}$ iterations. Since $\max_{j=y_{i}, r}\zeta_{j,r, i}^{(t)} \geq B_{i}^{(t)} - \max_{j=y_{i}, r}|\la \wb_{j,r}^{(0)}, \xi_{i} \ra| +0.4\bar{\beta} \geq B_{i}^{(t)}  - 1$,  $\max_{j=y_{i}, r}\zeta_{j,r,i}^{(t)}$ will reach $2$ within $T_{1}$ iterations. We can next verify that 
\begin{align*}
T_{1} = \frac{\log\big(20/(\sigma_{0}\sigma_{p}\sqrt{d})\big)4mn}{C_{1}0.15^{q-2}\eta q\sigma_{0}^{q-2}(\sigma_{p}^{2}\sqrt{d})^{q}} \leq \frac{m}{\eta q 2^{q} (\sqrt{2 \log(8m/\delta)})^{q-2} \sigma_0^{q-2}\|\bmu\|_{2}^q} = T_{1}^{+}/2,   
\end{align*}
where the inequality holds due to our SNR condition in \eqref{eq:explicit condition2}. Therefore, by the definition of $T_{1}^{(i)}$, we have $T_{1}^{(i)} \leq T_{1} \leq T_{1}^{+}/2$, where we use the non-decreasing property of $\zeta_{j,r,i}$. This completes the proof.
\end{proof}

\subsection{Second Stage}
By the signal-noise decompositon, at the end of the first stage, we have
\begin{align*}
\wb_{j,r}^{(T_{1})} &= \wb_{j,r}^{(0)} + j \cdot \gamma_{j,r}^{(T_{1})} \cdot \frac{\bmu}{\|\bmu\|_{2}^{2}} + \sum_{ i = 1}^n \zeta_{j,r,i}^{(T_{1})} \cdot \frac{\bxi_{i}}{\|\bxi_{i}\|_{2}^{2}} + \sum_{ i = 1}^n \omega_{j,r,i}^{(T_{1})} \cdot \frac{\bxi_{i}}{\|\bxi_{i}\|_{2}^{2}}
\end{align*}
for $j\in\{\pm 1\}$ and $r\in[m]$. 
By the results we get in the first stage, we know that at the beginning of this stage, we have following property holds:
\begin{itemize}
\item $\max_{ r}\zeta_{y_{i}, r, i}^{(T_{1})} \geq 2$ for all $i \in [n]$.
\item $\max_{j,r,i}|\omega_{j,r,i}^{(T_{1})}| = \tilde{O}(\sigma_{0}\sigma_{p}\sqrt{d})$.
\item $\max_{j,r}\gamma_{j,r}^{(T_{1})} \leq \hat{\beta}'$, where $\hat{\beta}' = \tilde{O}(\sigma_{0}\|\bmu\|_{2})$.
\end{itemize}
Note that Lemma~\ref{lemma:coefficient_iterative} implies that the learned noise $\zeta_{j,r,i}^{(t)}$ will not decrease, i.e., $\zeta_{j,r,i}^{(t+1)}\geq \zeta_{j,r,i}^{(t)}$. Therefore, for all data index $i$, we have $\max_{ r}\zeta_{y_{i}, r, i}^{(t)} \geq 2$ for all $t\geq T_{1}$. Now we choose $\Wb^{*}$ as follows
\begin{align*}
\wb^{*}_{j,r} = \wb_{j,r}^{(0)} + 2qm\log(2q/\epsilon))\bigg[\sum_{ i = 1}^n \ind(j = y_{i}) \cdot \frac{\bxi_{i}}{\|\bxi_{i}\|_{2}}\bigg].
\end{align*}
Based on the definition of $\Wb^{*}$, we have the following lemma.
\begin{lemma}\label{lm:distance2}
Under the same conditions as Theorem~\ref{thm:noise_memorization_main}, we have that $\|\Wb^{(T_{1})} - \Wb^{*}\|_{F} \leq  \tilde{O}(m^{2}n^{1/2}\sigma_{p}^{-1}d^{-1/2})$.
\end{lemma}
\begin{proof}[Proof of Lemma~\ref{lm:distance2}] We have
\begin{align*}
\|\Wb^{(T_{1})} - \Wb^{*}\|_{F} &\leq \|\Wb^{(T_{1})} - \Wb^{(0)}\|_{F} + \|\Wb^{(0)} - \Wb^{*}\|_{F} \\
&\leq \sum_{j,r}\gamma_{j,r}^{(T_{1})}\|\bmu\|_{2}^{-1} + O(\sqrt{m})\max_{j,r}\bigg\|\sum_{ i = 1}^n \zeta_{j,r,i}^{(T_{1})} \cdot \frac{\bxi_{i}}{\|\bxi_{i}\|_{2}^{2}} + \sum_{ i = 1}^n \omega_{j,r,i}^{(T_{1})} \cdot \frac{\bxi_{i}}{\|\bxi_{i}\|_{2}^{2}}\bigg\|_{2}\\
&\qquad + O(m^{3/2}n^{1/2}\log(1/\epsilon)\sigma_{p}^{-1}d^{-1/2})\\
&\leq \tilde{O}(m^{3/2}n^{1/2}\sigma_{p}^{-1}d^{-1/2}),
\end{align*}
where the first inequality is by triangle inequality, the second inequality is by our decomposition of $\Wb^{(T_{1})}, \Wb^{*}$ and  Lemma~\ref{lemma:data_innerproducts} (notice that different noises are almost orthogonal), and the last inequality is by Proposition~\ref{Prop:noise} and Lemma~\ref{lemma:phase1_main_noise}. This completes the proof.
\end{proof}

\begin{lemma}\label{lm: Gradient Stable2}
Under the same conditions as Theorem~\ref{thm:noise_memorization_main}, we have that
\begin{align*}
y_{i}\la \nabla f(\Wb^{(t)}, \xb_{i}), \Wb^{*} \ra \geq q\log(2q/\epsilon)   
\end{align*}
for all $ T_{1} \leq t\leq T^{*}$.
\end{lemma}
\begin{proof}[Proof of Lemma~\ref{lm: Gradient Stable2}]
Recall that $f(\Wb^{(t)}, \xb_{i}) = (1/m) {\sum_{j,r}}j\cdot \big[\sigma(\la\wb_{j,r},y_{i}\cdot\bmu\ra) + \sigma(\la\wb_{j,r}, \bxi_{i}\ra)\big] $, so we have 
\begin{align}
&y_{i}\la \nabla f(\Wb^{(t)}, \xb_{i}), \Wb^{*} \ra \notag\\
&= \frac{1}{m}\sum_{j,r}\sigma'(\la \wb_{j,r}^{(t)}, y_{i}\bmu\ra)\la \bmu,  j\wb_{j,r}^{*}\ra + \frac{1}{m}\sum_{j,r}\sigma'(\la \wb_{j,r}^{(t)}, \bxi_{i}\ra)\la y_{i}\bxi_{i},  j\wb_{j,r}^{*}\ra \notag\\
&= \frac{1}{m}\sum_{j,r}\sum_{i'=1}^{n}\sigma'(\la \wb_{j,r}^{(t)}, \bxi_{i}\ra)2qm\log(2q/\epsilon) \ind(j = y_{i'}) \cdot \frac{\la\bxi_{i'}, \bxi_{i} \ra}{\|\bxi_{i'}\|_{2}}\notag\\
&\qquad+ \frac{1}{m}\sum_{j,r}\sigma'(\la \wb_{j,r}^{(t)}, y_{i}\bmu\ra)\la \bmu,  j\wb_{j,r}^{(0)}\ra  + \frac{1}{m}\sum_{j,r}\sigma'(\la \wb_{j,r}^{(t)}, \bxi_{i}\ra)\la y_{i}\bxi_{i},  j\wb_{j,r}^{(0)}\ra \notag\\
&\geq \frac{1}{m}\sum_{j=y_{i},r}\sigma'(\la \wb_{j,r}^{(t)}, \bxi_{i}\ra)2qm\log(2q/\epsilon) - \frac{1}{m}\sum_{j,r}\sum_{i'\not= i}\sigma'(\la \wb_{j,r}^{(t)}, \bxi_{i}\ra)2qm\log(2q/\epsilon)\frac{|\la\bxi_{i'}, \bxi_{i} \ra|}{\|\bxi_{i'}\|_{2}} \notag \\
&\qquad -  \frac{1}{m}\sum_{j,r}\sigma'(\la \wb_{j,r}^{(t)}, y_{i}\bmu\ra)\tilde{O}(\sigma_{0}\|\bmu\|_{2})  - \frac{1}{m}\sum_{j,r}\sigma'(\la \wb_{j,r}^{(t)}, \bxi_{i}\ra)\tilde{O}(\sigma_{0}\sigma_{p}\sqrt{d}) \notag\\
&\geq \frac{1}{m}\sum_{j=y_{i},r}\sigma'(\la \wb_{j,r}^{(t)}, \bxi_{i}\ra)2qm\log(2q/\epsilon) - \frac{1}{m}\sum_{j,r}\sigma'(\la \wb_{j,r}^{(t)}, \bxi_{i}\ra)\tilde{O}(mnd^{-1/2}) \notag \\
&\qquad -  \frac{1}{m}\sum_{j,r}\sigma'(\la \wb_{j,r}^{(t)}, y_{i}\bmu\ra)\tilde{O}(\sigma_{0}\|\bmu\|_{2})  - \frac{1}{m}\sum_{j,r}\sigma'(\la \wb_{j,r}^{(t)}, \bxi_{i}\ra)\tilde{O}(\sigma_{0}\sigma_{p}\sqrt{d}),
\label{eq: inner2}
\end{align}
where the first inequality is by Lemma~\ref{lemma:initialization_norms} and the last inequality is by Lemma~\ref{lemma:data_innerproducts}.  Next we will bound the inner-product terms in \eqref{eq: inner} respectively. 
By Lemma~\ref{lm: Fyi}, we have that 
\begin{align}
|\la \wb_{j,r}^{(t)}, y_{i}\bmu\ra| \leq |\la \wb_{j,r}^{(0)}, y_{i}\bmu \ra| + \gamma_{j,r}^{(t)} \leq  \tilde{O}(1), \label{eq:GS21}
\end{align}
where the last inequality is by Proposition~\ref{Prop:noise}.

For $j = y_{i}$, we can bound the inner product between the parameter and the noise as follows
\begin{align}
\max_{j,r}\la \wb_{j,r}^{(t)},  \bxi_{i}\ra &\geq \max_{j,r}\big[\la\wb_{j,r}^{(0)}, \bxi_{i}\ra + \zeta_{j,r,i}^{(t)}\big] -  8n\sqrt{\frac{\log(4n^{2}/\delta)}{d}}\alpha \geq 1, \label{eq:GS22}
\end{align}
where the first inequality is by Lemma~\ref{lm:oppositebound}, the second inequality is by Lemma~\ref{lemma:phase1_main_noise}.

For $j = -y_{i}$, we can bound the inner product between the parameter and the noise as follows
\begin{align}
\la \wb_{j,r}^{(t)},  \bxi_{i}\ra &\leq \la\wb_{j,r}^{(0)}, \bxi_{i}\ra  + 8n\sqrt{\frac{\log(4n^{2}/\delta)}{d}}\alpha \leq 1\label{eq:GS23},
\end{align}
where the first inequality is by Lemma~\ref{lm: F-yi} and the last inequality is by Lemma~\ref{lemma:initialization_norms} and the conditions of $\sigma_{0}$ and $d$ in Condition~\ref{condition:d_sigma0_eta}.
Therefore, plugging \eqref{eq:GS21}, \eqref{eq:GS22}, \eqref{eq:GS23} into \eqref{eq: inner2} gives 
\begin{align*}
y_{i}\la \nabla f(\Wb^{(t)}, \xb_{i}), \Wb^{*} \ra 
&\geq 2q\log(2q/\epsilon))  - \tilde{O}(mnd^{-1/2})  -\tilde{O}(\sigma_{0}\|\bmu\|_{2}) - \tilde{O}(\sigma_{0}\sigma_{p}\sqrt{d}) \\
&\geq q\log(2q/\epsilon),
\end{align*}
where the last inequality is by $d\geq \tilde{\Omega}(m^{2}n^{4})$ and $\sigma_{0} \leq \tilde{O}(m^{-2/(q-2)}n^{-1})\cdot\min\{(\sigma_{p}\sqrt{d})^{-1}, \|\bmu\|_{2}^{-1}\}$ in Condition~\ref{condition:d_sigma0_eta}.
\end{proof}
\begin{lemma}\label{lemma:noise_stage2_homogeneity}
Under the same conditions as Theorem~\ref{thm:noise_memorization_main}, we have that  
\begin{align*}
\|\Wb^{(t)} - \Wb^{*}\|_{F}^{2} - \|\Wb^{(t+1)} - \Wb^{*}\|_{F}^{2} \geq (2q-1)\eta L_{S}(\Wb^{(t)}) - \eta\epsilon
\end{align*}
for all $ T_{1} \leq t \leq T^{*}$.
\end{lemma}
\begin{proof}[Proof of Lemma~\ref{lemma:noise_stage2_homogeneity}]
The proof is exactly same as the proof of Lemma~\ref{lemma:signal_stage2_homogeneity}.
\begin{align*}
&\|\Wb^{(t)} - \Wb^{*}\|_{F}^{2} - \|\Wb^{(t+1)} - \Wb^{*}\|_{F}^{2}\\
&\qquad =  2\eta \la \nabla L_{S}(\Wb^{(t)}), \Wb^{(t)} - \Wb^{*}\ra - \eta^{2}\|\nabla L_{S}(\Wb^{(t)})\|_{F}^{2}\\
&\qquad =  \frac{2\eta}{n}\sum_{i=1}^{n}\ell'^{(t)}_{i}[qy_{i}f(\Wb^{(t)}, \xb_{i}) - \la \nabla f(\Wb^{(t)}, \xb_{i}), \Wb^{*} \ra] - \eta^{2}\|\nabla L_{S}(\Wb^{(t)})\|_{F}^{2}\\
&\qquad \geq   \frac{2\eta}{n}\sum_{i=1}^{n}\ell'^{(t)}_{i}[qy_{i}f(\Wb^{(t)},\xb_{i}) - q\log(6/\epsilon)] - \eta^{2}\|\nabla L_{S}(\Wb^{(t)})\|_{F}^{2}\\
&\qquad \geq \frac{2q\eta}{n}\sum_{i=1}^{n}[\ell\big(y_{i}f(\Wb^{(t)},\xb_{i})\big) - \epsilon/(2q)] - \eta^{2}\|\nabla L_{S}(\Wb^{(t)})\|_{F}^{2}\\
&\qquad \geq (2q-1)\eta L_{S}(\Wb^{(t)}) - \eta\epsilon,
\end{align*}
where the first inequality is by Lemma~\ref{lm: Gradient Stable2}, the second inequality is due to the convexity of the cross entropy function and the last inequality is due to Lemma~\ref{lm: gradient upbound}.
\end{proof}

\begin{lemma}\label{thm:noise_proof}
Under the same conditions as Theorem~\ref{thm:noise_memorization_main}, let $T = T_{1} + \Big\lfloor \frac{\|\Wb^{(T_{1})} - \Wb^{*}\|_{F}^{2}}{2\eta \epsilon}
\Big\rfloor  = T_{1} + \tilde{O}(\eta^{-1}\epsilon^{-1}m^{3}nd^{-1}\sigma_{p}^{-2})$. Then we have $\max_{j,r}\gamma_{j,r}^{(t)} \leq 2\hat{\beta}'$, $\max_{j,r,i}|\omega_{j,r,i}^{(t)}| = \tilde{O}(\sigma_{0}\sigma_{p}\sqrt{d})$ for all $T_{1} \leq t \leq T$. Besides,
\begin{align*}
\frac{1}{t - T_{1} + 1}\sum_{s=T_{1}}^{t}L_{S}(\Wb^{(s)}) \leq  \frac{\|\Wb^{(T_{1})} - \Wb^{*}\|_{F}^{2}}{(2q-1) \eta(t - T_{1} + 1)} + \frac{\epsilon}{(2q-1)}
\end{align*}
for all $T_{1} \leq t \leq T$, and we can find an iterate with training loss smaller than $\epsilon$ within $T$ iterations.
\end{lemma}
\begin{proof}[Proof of Lemma~\ref{thm:noise_proof}]
By Lemma~\ref{lemma:noise_stage2_homogeneity}, for any $T_{1} \leq t \leq T$, we obtain that  
\begin{align}
\|\Wb^{(s)} - \Wb^{*}\|_{F}^{2} - \|\Wb^{(s+1)} - \Wb^{*}\|_{F}^{2} \geq (2q-1)\eta L_{S}(\Wb^{(s)}) - \eta\epsilon \label{eq: bounddistance}
\end{align}
holds for $T_1 \leq s \leq t$. Taking a summation, we have that 
\begin{align}
\sum_{s=T_{1}}^{t}L_{S}(\Wb^{(s)}) &\leq  \frac{\|\Wb^{(T_{1})} - \Wb^{*}\|_{F}^{2} + \eta\epsilon (t - T_{1} + 1)}{(2q-1) \eta} \notag\\
&\overset{(i)}{\leq} \frac{2\|\Wb^{(T_{1})} - \Wb^{*}\|_{F}^{2} }{(2q-1) \eta} \notag\\
&\overset{(ii)}{=} \tilde{O}(\eta^{-1}m^{3}nd^{-1}\sigma_{p}^{-2})\label{eq: sum2},
\end{align}
where (i) is by $t \leq T_{2}$ and (ii) is by Lemma~\ref{lm:distance2} 
Then we can use induction to prove that $\max_{j,r}\gamma_{j,r}^{(t)} \leq 2\hat{\beta}'$ for all $t \in [T_1,T]$. Clearly, by the definition of $\hat{\beta}'$, we have $\max_{j,r}\gamma_{j,r}^{(T_1)} \leq \hat{\beta}' \leq 2\hat{\beta}'$. Now suppose that there exists $\tilde{T} \in [T_1, T]$ such that $\max_{j,r}\gamma_{j,r}^{(t)} \leq 2\hat{\beta}'$ for all $t \in [T_1,\tilde{T}-1]$. Then by \eqref{eq:update_gamma2}, we have
\begin{align*}
    \gamma_{j,r}^{(\tilde{T})} &= \gamma_{j,r}^{(T_{1})} - \frac{\eta}{nm} \sum_{s=T_{1}}^{\tilde{T} - 1}\sum_{i=1}^n \ell_i'^{(t)} \cdot \sigma'(\la\wb_{j,r}^{(t)}, y_{i} \cdot \bmu\ra)\|\bmu\|_{2}^{2},\\
    &\overset{(i)}{\leq} \gamma_{j,r}^{(T_{1})}  + \frac{q3^{q-1}\eta}{nm}\|\bmu\|_{2}^{2}\hat{\beta}'^{q-1} \sum_{s=T_{1}}^{\tilde{T} - 1}\sum_{i=1}^n |\ell_i'^{(t)}|\\
    &\overset{(ii)}{\leq} \gamma_{j,r}^{(T_{1})}  + q3^{q-1}\eta m^{-1}\|\bmu\|_{2}^{2}\hat{\beta}'^{q-1} \sum_{s=T_{1}}^{\tilde{T} - 1}L_{S}(\Wb^{(s)})\\
    &\overset{(iii)}{\leq} \gamma_{j,r}^{(T_{1})}  + \hat{\beta}'^{q-1} \tilde{O}(m^{2}n\mathrm{SNR}^{2})\\
    &\overset{(iv)}{\leq} \gamma_{j,r}^{(T_{1})}  + \hat{\beta}'^{q-1} \tilde{O}(m^{2}n^{1-2/q})\\
    &\overset{(v)}{\leq} 2\hat{\beta}'
\end{align*}
for all $j\in \{\pm1\}$ and $r\in [m]$, 
where (i) is by induction hypothesis $\max_{j,r}\gamma_{j,r}^{(t)} \leq 2\hat{\beta}'$, (ii) is by $|\ell'| \leq \ell$, (iii) is by \eqref{eq: sum2}, (iv) is by $n^{-1}\mathrm{SNR}^{-q}\geq \tilde{\Omega}(1)$, and $(v)$ is by $\hat{\beta}' = \tilde{O}(\sigma_{0}\|\bmu\|_{2})$ and 
$\hat{\beta}'^{q-2}\tilde{O}(m^{2}n^{1-2/q}) = \tilde{O}(m^{2}n^{1-2/q}(\sigma_{0}\|\bmu\|_{2})^{q-2}) \leq 1$ by Condition~\ref{condition:d_sigma0_eta}. Therefore, we have $\max_{j,r}\gamma_{j,r}^{(\tilde{T})} \leq 2\hat{\beta}'$, which completes the induction.
\end{proof}

\subsection{Population Loss}
\begin{lemma}[4th statement of Theorem~\ref{thm:noise_memorization_main}]\label{lemma: noise generalization}
Under the same conditions as Theorem~\ref{thm:noise_memorization_main}, within $\tilde{O}(\eta^{-1}n\sigma_{0}^{2-q}\sigma_{p}^{-q}d^{-q/2} + \eta^{-1}\epsilon^{-1}m^{3}n\sigma_{p}^{-2}d^{-1})$ iterations, we can find $\Wb^{(T)}$ such that $L_{S}(\Wb^{(T)}) \leq \epsilon$. Besides, for any $0 \leq t\leq T$ we have that $L_{\cD}(\Wb^{(t)}) \geq 0.1$.
\end{lemma}

\begin{proof}[Proof of Lemma~\ref{lemma: noise generalization}]

Given a new example $(x,y)$, we have that 
\begin{align*}
\|\wb_{j,r}^{(t)}\|_{2} &= \bigg\|\wb_{j,r}^{(0)} + j \cdot \gamma_{j,r}^{(t)} \cdot \frac{\bmu}{\|\bmu\|_{2}^{2}} + \sum_{ i = 1}^n \zeta_{j,r,i}^{(t)} \cdot \frac{\bxi_{i}}{\|\bxi_{i}\|_{2}^{2}} + \sum_{ i = 1}^n \omega_{j,r,i}^{(t)} \cdot \frac{\bxi_{i}}{\|\bxi_{i}\|_{2}^{2}}\bigg\|_{2}\\
&\overset{(i)}{\leq} \|\wb_{j,r}^{(0)}\|_{2} +  \frac{\gamma_{j,r}^{(t)}}{\|\bmu\|_{2}} + \sum_{ i = 1}^n  \frac{\zeta_{j,r,i}^{(t)}}{\|\bxi_{i}\|_{2}} + \sum_{ i = 1}^n  \frac{|\omega_{j,r,i}^{(t)}|}{\|\bxi_{i}\|_{2}}\\
&\overset{(ii)}{=} O(\sigma_{0}\sqrt{d})  + \tilde{O}(n\sigma_{p}^{-1}d^{-1/2}),
\end{align*}
where (i) is by triangle inequality and (ii) is by $\max_{j,r}\gamma_{j,r}^{(t)} = \tilde{O}(\sigma_{0}\|\bmu\|_{2})$ in Lemma~\ref{thm:noise_proof} and $\max_{i,j,r}|\rho_{j,r,i}| \leq 4\log(T^{*})$ in Proposition~\ref{Prop:main1}.

Therefore, we have that $\la\wb_{j,r}^{(t)}, \bxi\ra \sim \cN(0, \sigma_{p}^{2}\|\wb_{j,r}^{(t)}\|_{2}^{2})$. So with probability $1 - 1/(4m)$,
\begin{align*}
|\la\wb_{j,r}^{(t)}, \bxi\ra| \leq \tilde{O}(\sigma_{0}\sigma_{p}\sqrt{d} + nd^{-1/2}).
\end{align*}
Since the signal vector $\bmu$ is orthogonal to noises, by $\max_{j,r}\gamma_{j,r}^{(t)} \leq 2\hat\beta' = \tilde{O}(\sigma_{0}\|\bmu\|_{2})$ in Lemma~\ref{thm:noise_proof}, we also have that $|\la\wb_{j,r}^{(t)}, \bmu\ra| \leq |\la \wb_{j,r}^{(0)}, y_{i}\bmu\ra| + \gamma_{j,r}^{(t)} = \tilde{O}(\sigma_{0}\|\bmu\|_{2})$. 
Now by union bound, with probability at least $1-1/2$, we have that  \begin{align*}
F_{j}(\Wb_{j}^{(t)},\xb) &= \frac{1}{m}\sum_{r=1}^{m}\sigma(\la \wb_{j,r}^{(t)}, y\bmu \ra)
+ \frac{1}{m}\sum_{r=1}^{m}\sigma(\la \wb_{j,r}^{(t)}, \bxi \ra)\\
&\leq \max_{r} |\la \wb_{j,r}^{(t)}, \bmu\ra|^{q} + \max_{r} |\la \wb_{j,r}^{(t)}, \bxi\ra|^{q}\\
&\leq \tilde{O}(\sigma_{0}^{q}\sigma_{p}^{q}d^{q/2} + n^{q}d^{-q/2} + \sigma_{0}^{q}\|\bmu\|_{2}^{2})\\
&\leq 1,
\end{align*}
where the last inequality is by $\sigma_{0} \leq \tilde{O}(n^{-1}m^{-2/(q-2)}\cdot\min\{(\sigma_{p}\sqrt{d})^{-1}, \|\bmu\|_{2}^{-1}\})$ and $d \geq \tilde{\Omega}(m^{2}n^{4})$ in Condition~\ref{condition:d_sigma0_eta}. 
Therefore, with probability at least $1 - 1/2$, we have that 
\begin{align*}
\ell\big(y \cdot f(\Wb^{(t)},\xb)\big) \geq \log(1 + e^{-1}).   
\end{align*}
Thus $L_{\cD}(\Wb^{(t)}) \geq \log(1 + e^{-1})\cdot 0.5\geq 0.1$. This completes the proof.
\end{proof}

\bibliography{deeplearningreference}

\begin{thebibliography}{57}
\expandafter\ifx\csname natexlab\endcsname\relax\def\natexlab#1{#1}\fi
\expandafter\ifx\csname url\endcsname\relax
  \def\url#1{\texttt{#1}}\fi
\expandafter\ifx\csname urlprefix\endcsname\relax\def\urlprefix{URL }\fi

\bibitem[{Adlam and Pennington(2020)}]{adlam2020neural}
\textsc{Adlam, B.} and \textsc{Pennington, J.} (2020).
\newblock The neural tangent kernel in high dimensions: Triple descent and a
  multi-scale theory of generalization.
\newblock In \textit{International Conference on Machine Learning}. PMLR.

\bibitem[{Allen-Zhu and Li(2020{\natexlab{a}})}]{allen2020feature}
\textsc{Allen-Zhu, Z.} and \textsc{Li, Y.} (2020{\natexlab{a}}).
\newblock Feature purification: How adversarial training performs robust deep
  learning.
\newblock \textit{arXiv preprint arXiv:2005.10190} .

\bibitem[{Allen-Zhu and Li(2020{\natexlab{b}})}]{allen2020towards}
\textsc{Allen-Zhu, Z.} and \textsc{Li, Y.} (2020{\natexlab{b}}).
\newblock Towards understanding ensemble, knowledge distillation and
  self-distillation in deep learning.
\newblock \textit{arXiv preprint arXiv:2012.09816} .

\bibitem[{Allen-Zhu et~al.(2019{\natexlab{a}})Allen-Zhu, Li and
  Liang}]{allen2018learning}
\textsc{Allen-Zhu, Z.}, \textsc{Li, Y.} and \textsc{Liang, Y.}
  (2019{\natexlab{a}}).
\newblock Learning and generalization in overparameterized neural networks,
  going beyond two layers.
\newblock In \textit{Advances in Neural Information Processing Systems}.

\bibitem[{Allen-Zhu et~al.(2019{\natexlab{b}})Allen-Zhu, Li and
  Song}]{allen2018convergence}
\textsc{Allen-Zhu, Z.}, \textsc{Li, Y.} and \textsc{Song, Z.}
  (2019{\natexlab{b}}).
\newblock A convergence theory for deep learning via over-parameterization.
\newblock In \textit{International Conference on Machine Learning}.

\bibitem[{Arora et~al.(2019{\natexlab{a}})Arora, Du, Hu, Li and
  Wang}]{arora2019fine}
\textsc{Arora, S.}, \textsc{Du, S.}, \textsc{Hu, W.}, \textsc{Li, Z.} and
  \textsc{Wang, R.} (2019{\natexlab{a}}).
\newblock Fine-grained analysis of optimization and generalization for
  overparameterized two-layer neural networks.
\newblock In \textit{International Conference on Machine Learning}.

\bibitem[{Arora et~al.(2019{\natexlab{b}})Arora, Du, Hu, Li, Salakhutdinov and
  Wang}]{arora2019exact}
\textsc{Arora, S.}, \textsc{Du, S.~S.}, \textsc{Hu, W.}, \textsc{Li, Z.},
  \textsc{Salakhutdinov, R.} and \textsc{Wang, R.} (2019{\natexlab{b}}).
\newblock On exact computation with an infinitely wide neural net.
\newblock In \textit{Advances in Neural Information Processing Systems}.

\bibitem[{Arora et~al.(2018)Arora, Ge, Neyshabur and Zhang}]{arora2018stronger}
\textsc{Arora, S.}, \textsc{Ge, R.}, \textsc{Neyshabur, B.} and \textsc{Zhang,
  Y.} (2018).
\newblock Stronger generalization bounds for deep nets via a compression
  approach.
\newblock In \textit{International Conference on Machine Learning}.

\bibitem[{Bartlett et~al.(2017)Bartlett, Foster and
  Telgarsky}]{bartlett2017spectrally}
\textsc{Bartlett, P.~L.}, \textsc{Foster, D.~J.} and \textsc{Telgarsky, M.~J.}
  (2017).
\newblock Spectrally-normalized margin bounds for neural networks.
\newblock In \textit{Advances in Neural Information Processing Systems}.

\bibitem[{Bartlett et~al.(2020)Bartlett, Long, Lugosi and
  Tsigler}]{bartlett2020benign}
\textsc{Bartlett, P.~L.}, \textsc{Long, P.~M.}, \textsc{Lugosi, G.} and
  \textsc{Tsigler, A.} (2020).
\newblock Benign overfitting in linear regression.
\newblock \textit{Proceedings of the National Academy of Sciences} .

\bibitem[{Belkin et~al.(2019{\natexlab{a}})Belkin, Hsu, Ma and
  Mandal}]{belkin2019reconciling}
\textsc{Belkin, M.}, \textsc{Hsu, D.}, \textsc{Ma, S.} and \textsc{Mandal, S.}
  (2019{\natexlab{a}}).
\newblock Reconciling modern machine-learning practice and the classical
  bias--variance trade-off.
\newblock \textit{Proceedings of the National Academy of Sciences} \textbf{116}
  15849--15854.

\bibitem[{Belkin et~al.(2019{\natexlab{b}})Belkin, Hsu and Xu}]{belkin2019two}
\textsc{Belkin, M.}, \textsc{Hsu, D.} and \textsc{Xu, J.} (2019{\natexlab{b}}).
\newblock Two models of double descent for weak features.
\newblock \textit{arXiv preprint arXiv:1903.07571} .

\bibitem[{Belkin et~al.(2018)Belkin, Ma and Mandal}]{belkin2018understand}
\textsc{Belkin, M.}, \textsc{Ma, S.} and \textsc{Mandal, S.} (2018).
\newblock To understand deep learning we need to understand kernel learning.
\newblock In \textit{International Conference on Machine Learning}.

\bibitem[{Bousquet and Elisseeff(2002)}]{bousquet2002stability}
\textsc{Bousquet, O.} and \textsc{Elisseeff, A.} (2002).
\newblock Stability and generalization.
\newblock \textit{Journal of machine learning research} \textbf{2} 499--526.

\bibitem[{Brutzkus et~al.(2018)Brutzkus, Globerson, Malach and
  Shalev-Shwartz}]{brutzkus2017sgd}
\textsc{Brutzkus, A.}, \textsc{Globerson, A.}, \textsc{Malach, E.} and
  \textsc{Shalev-Shwartz, S.} (2018).
\newblock Sgd learns over-parameterized networks that provably generalize on
  linearly separable data.
\newblock In \textit{International Conference on Learning Representations}.

\bibitem[{Cao and Gu(2019{\natexlab{a}})}]{cao2019generalizationsgd}
\textsc{Cao, Y.} and \textsc{Gu, Q.} (2019{\natexlab{a}}).
\newblock Generalization bounds of stochastic gradient descent for wide and
  deep neural networks.
\newblock In \textit{Advances in Neural Information Processing Systems}.

\bibitem[{Cao and Gu(2019{\natexlab{b}})}]{cao2019tight}
\textsc{Cao, Y.} and \textsc{Gu, Q.} (2019{\natexlab{b}}).
\newblock Tight sample complexity of learning one-hidden-layer convolutional
  neural networks.
\newblock In \textit{Advances in Neural Information Processing Systems}.

\bibitem[{Cao et~al.(2021)Cao, Gu and Belkin}]{cao2021risk}
\textsc{Cao, Y.}, \textsc{Gu, Q.} and \textsc{Belkin, M.} (2021).
\newblock Risk bounds for over-parameterized maximum margin classification on
  sub-gaussian mixtures.
\newblock \textit{Advances in Neural Information Processing Systems}
  \textbf{34}.

\bibitem[{Chatterji and Long(2020)}]{chatterji2020finite}
\textsc{Chatterji, N.~S.} and \textsc{Long, P.~M.} (2020).
\newblock Finite-sample analysis of interpolating linear classifiers in the
  overparameterized regime.
\newblock \textit{arXiv preprint arXiv:2004.12019} .

\bibitem[{Chen et~al.(2018)Chen, Jin and Yu}]{chen2018stability}
\textsc{Chen, Y.}, \textsc{Jin, C.} and \textsc{Yu, B.} (2018).
\newblock Stability and convergence trade-off of iterative optimization
  algorithms.
\newblock \textit{arXiv preprint arXiv:1804.01619} .

\bibitem[{Chen et~al.(2019)Chen, Cao, Zou and Gu}]{chen2019much}
\textsc{Chen, Z.}, \textsc{Cao, Y.}, \textsc{Zou, D.} and \textsc{Gu, Q.}
  (2019).
\newblock How much over-parameterization is sufficient to learn deep relu
  networks?
\newblock \textit{arXiv preprint arXiv:1911.12360} .

\bibitem[{Du et~al.(2019{\natexlab{a}})Du, Lee, Li, Wang and
  Zhai}]{du2018gradientdeep}
\textsc{Du, S.}, \textsc{Lee, J.}, \textsc{Li, H.}, \textsc{Wang, L.} and
  \textsc{Zhai, X.} (2019{\natexlab{a}}).
\newblock Gradient descent finds global minima of deep neural networks.
\newblock In \textit{International Conference on Machine Learning}.

\bibitem[{Du et~al.(2018{\natexlab{a}})Du, Lee and Tian}]{du2017convolutional}
\textsc{Du, S.~S.}, \textsc{Lee, J.~D.} and \textsc{Tian, Y.}
  (2018{\natexlab{a}}).
\newblock When is a convolutional filter easy to learn?
\newblock In \textit{International Conference on Learning Representations}.

\bibitem[{Du et~al.(2018{\natexlab{b}})Du, Lee, Tian, Singh and
  Poczos}]{du2017gradient}
\textsc{Du, S.~S.}, \textsc{Lee, J.~D.}, \textsc{Tian, Y.}, \textsc{Singh, A.}
  and \textsc{Poczos, B.} (2018{\natexlab{b}}).
\newblock Gradient descent learns one-hidden-layer {CNN}: Don’t be afraid of
  spurious local minima.
\newblock In \textit{International Conference on Machine Learning}.

\bibitem[{Du et~al.(2019{\natexlab{b}})Du, Zhai, Poczos and
  Singh}]{du2018gradient}
\textsc{Du, S.~S.}, \textsc{Zhai, X.}, \textsc{Poczos, B.} and \textsc{Singh,
  A.} (2019{\natexlab{b}}).
\newblock Gradient descent provably optimizes over-parameterized neural
  networks.
\newblock In \textit{International Conference on Learning Representations}.

\bibitem[{Frei et~al.(2021)Frei, Cao and Gu}]{frei2021provable}
\textsc{Frei, S.}, \textsc{Cao, Y.} and \textsc{Gu, Q.} (2021).
\newblock Provable generalization of sgd-trained neural networks of any width
  in the presence of adversarial label noise.
\newblock In \textit{International Conference on Machine Learning}. PMLR.

\bibitem[{Frei et~al.(2022)Frei, Chatterji and Bartlett}]{frei2022benign}
\textsc{Frei, S.}, \textsc{Chatterji, N.~S.} and \textsc{Bartlett, P.~L.}
  (2022).
\newblock Benign overfitting without linearity: Neural network classifiers
  trained by gradient descent for noisy linear data.
\newblock \textit{arXiv preprint arXiv:2202.05928} .

\bibitem[{Golowich et~al.(2018)Golowich, Rakhlin and Shamir}]{golowich2017size}
\textsc{Golowich, N.}, \textsc{Rakhlin, A.} and \textsc{Shamir, O.} (2018).
\newblock Size-independent sample complexity of neural networks.
\newblock In \textit{Conference On Learning Theory}.

\bibitem[{Hardt et~al.(2016)Hardt, Recht and Singer}]{hardt2016train}
\textsc{Hardt, M.}, \textsc{Recht, B.} and \textsc{Singer, Y.} (2016).
\newblock Train faster, generalize better: stability of stochastic gradient
  descent.
\newblock In \textit{Proceedings of the 33rd International Conference on
  International Conference on Machine Learning-Volume 48}. JMLR. org.

\bibitem[{Hastie et~al.(2019)Hastie, Montanari, Rosset and
  Tibshirani}]{hastie2019surprises}
\textsc{Hastie, T.}, \textsc{Montanari, A.}, \textsc{Rosset, S.} and
  \textsc{Tibshirani, R.~J.} (2019).
\newblock Surprises in high-dimensional ridgeless least squares interpolation.
\newblock \textit{arXiv preprint arXiv:1903.08560} .

\bibitem[{Jacot et~al.(2018)Jacot, Gabriel and Hongler}]{jacot2018neural}
\textsc{Jacot, A.}, \textsc{Gabriel, F.} and \textsc{Hongler, C.} (2018).
\newblock Neural tangent kernel: Convergence and generalization in neural
  networks.
\newblock In \textit{Advances in neural information processing systems}.

\bibitem[{Ji and Telgarsky(2020)}]{ji2019polylogarithmic}
\textsc{Ji, Z.} and \textsc{Telgarsky, M.} (2020).
\newblock Polylogarithmic width suffices for gradient descent to achieve
  arbitrarily small test error with shallow relu networks.
\newblock In \textit{International Conference on Learning Representations}.

\bibitem[{Lee et~al.(2019)Lee, Xiao, Schoenholz, Bahri, Sohl-Dickstein and
  Pennington}]{lee2019wide}
\textsc{Lee, J.}, \textsc{Xiao, L.}, \textsc{Schoenholz, S.~S.}, \textsc{Bahri,
  Y.}, \textsc{Sohl-Dickstein, J.} and \textsc{Pennington, J.} (2019).
\newblock Wide neural networks of any depth evolve as linear models under
  gradient descent.
\newblock In \textit{Advances in Neural Information Processing Systems}.

\bibitem[{Li and Liang(2018)}]{li2018learning}
\textsc{Li, Y.} and \textsc{Liang, Y.} (2018).
\newblock Learning overparameterized neural networks via stochastic gradient
  descent on structured data.
\newblock In \textit{Advances in Neural Information Processing Systems}.

\bibitem[{Li et~al.(2019)Li, Wei and Ma}]{li2019towards}
\textsc{Li, Y.}, \textsc{Wei, C.} and \textsc{Ma, T.} (2019).
\newblock Towards explaining the regularization effect of initial large
  learning rate in training neural networks.
\newblock In \textit{Advances in Neural Information Processing Systems}.

\bibitem[{Li and Yuan(2017)}]{li2017convergence}
\textsc{Li, Y.} and \textsc{Yuan, Y.} (2017).
\newblock Convergence analysis of two-layer neural networks with relu
  activation.
\newblock In \textit{Advances in Neural Information Processing Systems}.

\bibitem[{Li et~al.(2021)Li, Zhou and Gretton}]{li2021towards}
\textsc{Li, Z.}, \textsc{Zhou, Z.-H.} and \textsc{Gretton, A.} (2021).
\newblock Towards an understanding of benign overfitting in neural networks.
\newblock \textit{arXiv preprint arXiv:2106.03212} .

\bibitem[{Liang and Rakhlin(2020)}]{liang2020just}
\textsc{Liang, T.} and \textsc{Rakhlin, A.} (2020).
\newblock Just interpolate: Kernel “ridgeless” regression can generalize.
\newblock \textit{The Annals of Statistics} \textbf{48} 1329--1347.

\bibitem[{Liao et~al.(2020)Liao, Couillet and Mahoney}]{liao2020random}
\textsc{Liao, Z.}, \textsc{Couillet, R.} and \textsc{Mahoney, M.} (2020).
\newblock A random matrix analysis of random fourier features: beyond the
  gaussian kernel, a precise phase transition, and the corresponding double
  descent.
\newblock In \textit{34th Conference on Neural Information Processing Systems
  (NeurIPS 2020)}.

\bibitem[{Lin et~al.(2013)Lin, Chen and Yan}]{lin2013network}
\textsc{Lin, M.}, \textsc{Chen, Q.} and \textsc{Yan, S.} (2013).
\newblock Network in network.
\newblock \textit{arXiv preprint arXiv:1312.4400} .

\bibitem[{Lyu and Li(2019)}]{lyu2019gradient}
\textsc{Lyu, K.} and \textsc{Li, J.} (2019).
\newblock Gradient descent maximizes the margin of homogeneous neural networks.
\newblock \textit{arXiv preprint arXiv:1906.05890} .

\bibitem[{Mei and Montanari(2019)}]{mei2019generalization}
\textsc{Mei, S.} and \textsc{Montanari, A.} (2019).
\newblock The generalization error of random features regression: Precise
  asymptotics and double descent curve.
\newblock \textit{arXiv preprint arXiv:1908.05355} .

\bibitem[{Montanari and Zhong(2020)}]{montanari2020interpolation}
\textsc{Montanari, A.} and \textsc{Zhong, Y.} (2020).
\newblock The interpolation phase transition in neural networks: Memorization
  and generalization under lazy training.
\newblock \textit{arXiv preprint arXiv:2007.12826} .

\bibitem[{Mou et~al.(2017)Mou, Wang, Zhai and Zheng}]{mou2017generalization}
\textsc{Mou, W.}, \textsc{Wang, L.}, \textsc{Zhai, X.} and \textsc{Zheng, K.}
  (2017).
\newblock Generalization bounds of sgld for non-convex learning: Two
  theoretical viewpoints.
\newblock \textit{arXiv preprint arXiv:1707.05947} .

\bibitem[{Neyshabur et~al.(2018)Neyshabur, Bhojanapalli, McAllester and
  Srebro}]{neyshabur2017pac}
\textsc{Neyshabur, B.}, \textsc{Bhojanapalli, S.}, \textsc{McAllester, D.} and
  \textsc{Srebro, N.} (2018).
\newblock A pac-bayesian approach to spectrally-normalized margin bounds for
  neural networks.
\newblock In \textit{International Conference on Learning Representation}.

\bibitem[{Neyshabur et~al.(2019)Neyshabur, Li, Bhojanapalli, LeCun and
  Srebro}]{neyshabur2018role}
\textsc{Neyshabur, B.}, \textsc{Li, Z.}, \textsc{Bhojanapalli, S.},
  \textsc{LeCun, Y.} and \textsc{Srebro, N.} (2019).
\newblock Towards understanding the role of over-parametrization in
  generalization of neural networks.
\newblock In \textit{International Conference on Learning Representations}.

\bibitem[{Neyshabur et~al.(2015)Neyshabur, Tomioka and
  Srebro}]{neyshabur2015norm}
\textsc{Neyshabur, B.}, \textsc{Tomioka, R.} and \textsc{Srebro, N.} (2015).
\newblock Norm-based capacity control in neural networks.
\newblock In \textit{Conference on Learning Theory}.

\bibitem[{Shamir(2022)}]{shamir2022implicit}
\textsc{Shamir, O.} (2022).
\newblock The implicit bias of benign overfitting.
\newblock \textit{arXiv preprint arXiv:2201.11489} .

\bibitem[{Soltanolkotabi(2017)}]{soltanolkotabi2017learning}
\textsc{Soltanolkotabi, M.} (2017).
\newblock Learning {ReLUs} via gradient descent.
\newblock In \textit{Advances in Neural Information Processing Systems}.

\bibitem[{Soudry et~al.(2018)Soudry, Hoffer, Nacson, Gunasekar and
  Srebro}]{soudry2017implicit}
\textsc{Soudry, D.}, \textsc{Hoffer, E.}, \textsc{Nacson, M.~S.},
  \textsc{Gunasekar, S.} and \textsc{Srebro, N.} (2018).
\newblock The implicit bias of gradient descent on separable data.
\newblock \textit{The Journal of Machine Learning Research} \textbf{19}
  2822--2878.

\bibitem[{Wu and Xu(2020)}]{wu2020optimal}
\textsc{Wu, D.} and \textsc{Xu, J.} (2020).
\newblock On the optimal weighted $\ell_2$ regularization in overparameterized
  linear regression.
\newblock \textit{Advances in Neural Information Processing Systems}
  \textbf{33}.

\bibitem[{Zhang et~al.(2017)Zhang, Bengio, Hardt, Recht and
  Vinyals}]{zhang2016understanding}
\textsc{Zhang, C.}, \textsc{Bengio, S.}, \textsc{Hardt, M.}, \textsc{Recht, B.}
  and \textsc{Vinyals, O.} (2017).
\newblock Understanding deep learning requires rethinking generalization.
\newblock In \textit{International Conference on Learning Representations}.

\bibitem[{Zhang et~al.(2019)Zhang, Yu, Wang and Gu}]{zhang2018learning}
\textsc{Zhang, X.}, \textsc{Yu, Y.}, \textsc{Wang, L.} and \textsc{Gu, Q.}
  (2019).
\newblock Learning one-hidden-layer {ReLU} networks via gradient descent.
\newblock In \textit{The 22nd International Conference on Artificial
  Intelligence and Statistics}.

\bibitem[{Zhong et~al.(2017)Zhong, Song, Jain, Bartlett and
  Dhillon}]{zhong2017recovery}
\textsc{Zhong, K.}, \textsc{Song, Z.}, \textsc{Jain, P.}, \textsc{Bartlett,
  P.~L.} and \textsc{Dhillon, I.~S.} (2017).
\newblock Recovery guarantees for one-hidden-layer neural networks.
\newblock In \textit{International Conference on Machine Learning}.

\bibitem[{Zou et~al.(2021{\natexlab{a}})Zou, Cao, Li and
  Gu}]{zou2021understanding}
\textsc{Zou, D.}, \textsc{Cao, Y.}, \textsc{Li, Y.} and \textsc{Gu, Q.}
  (2021{\natexlab{a}}).
\newblock Understanding the generalization of adam in learning neural networks
  with proper regularization.
\newblock \textit{arXiv preprint arXiv:2108.11371} .

\bibitem[{Zou et~al.(2019)Zou, Cao, Zhou and Gu}]{zou2019gradient}
\textsc{Zou, D.}, \textsc{Cao, Y.}, \textsc{Zhou, D.} and \textsc{Gu, Q.}
  (2019).
\newblock Gradient descent optimizes over-parameterized deep {ReLU} networks.
\newblock \textit{Machine Learning} .

\bibitem[{Zou et~al.(2021{\natexlab{b}})Zou, Wu, Braverman, Gu and
  Kakade}]{zou2021benign}
\textsc{Zou, D.}, \textsc{Wu, J.}, \textsc{Braverman, V.}, \textsc{Gu, Q.} and
  \textsc{Kakade, S.} (2021{\natexlab{b}}).
\newblock Benign overfitting of constant-stepsize sgd for linear regression.
\newblock In \textit{Conference on Learning Theory}. PMLR.

\end{thebibliography}
\bibliographystyle{ims}

\end{document}